\documentclass[11pt]{article}

\usepackage{fullpage,times,url,bm}
\usepackage{amsthm}
\usepackage{hyperref}
\hypersetup{
	colorlinks   = true, %Colours links instead of ugly boxes
	urlcolor     = blue, %Colour for external hyperlinks
	linkcolor    = blue, %Colour of internal links
	citecolor   =  %Colour of citations
}

\usepackage{natbib}
\usepackage{bbm}
\usepackage{epstopdf}
\usepackage{authblk}
\usepackage[ruled,vlined]{algorithm2e}
\usepackage{amsfonts,amssymb,mathtools,color,float,graphicx,verbatim}

\newtheorem{theorem}{Theorem}
\newtheorem{proposition}{Proposition}
\newtheorem{lemma}{Lemma}

\newtheorem{conjecture}{Conjecture}
\newtheorem{definition}{Definition}
\newtheorem{remark}{Remark}

\newcommand{\reals}{\mathbb{R}}
\newcommand{\E}{\mathbb{E}}

\newcommand{\relu}[1]{\left[ #1 \right]_+}

\usepackage{enumitem}

\newcommand{\be}{\mathbf{e}}
\newcommand{\bx}{\mathbf{x}}
\newcommand{\bw}{\mathbf{w}}
\newcommand{\bg}{\mathbf{g}}

\newcommand{\bu}{\mathbf{u}}
\newcommand{\bv}{\mathbf{v}}
\newcommand{\bz}{\mathbf{z}}

\newcommand{\by}{\mathbf{y}}
\newcommand{\bn}{\mathbf{n}}

\newcommand{\Dcal}{\mathcal{D}}

\newcommand{\norm}[1]{\|#1\|}
\newcommand{\inner}[1]{\langle#1\rangle}
\newcommand{\p}[1]{\left(#1\right)}
\newcommand{\abs}[1]{\left|#1\right|}

\newtheorem{example}{Example}

\newcommand{\secref}[1]{Sec.~\ref{#1}}
\newcommand{\subsecref}[1]{Subsection~\ref{#1}}
\newcommand{\figref}[1]{Fig.~\ref{#1}}
\renewcommand{\eqref}[1]{Eq.~(\ref{#1})}
\newcommand{\lemref}[1]{Lemma~\ref{#1}}
\newcommand{\thmref}[1]{Thm.~\ref{#1}}
\newcommand{\propref}[1]{Proposition~\ref{#1}}
\newcommand{\appref}[1]{Appendix~\ref{#1}}
\newcommand{\algref}[1]{Algorithm~\ref{#1}}

\newcommand\revision[1]{\textcolor{black}{#1}}
\makeatletter
\def\moverlay{\mathpalette\mov@rlay}
\def\mov@rlay#1#2{\leavevmode\vtop{%
   \baselineskip\z@skip \lineskiplimit-\maxdimen
   \ialign{\hfil$\m@th#1##$\hfil\cr#2\crcr}}}
\newcommand{\charfusion}[3][\mathord]{
    #1{\ifx#1\mathop\vphantom{#2}\fi
        \mathpalette\mov@rlay{#2\cr#3}
      }
    \ifx#1\mathop\expandafter\displaylimits\fi}
\makeatother

\newcommand{\bigcupdot}{\charfusion[\mathop]{\bigcup}{\cdot}}

\makeatletter
\newcommand{\printfnsymbol}[1]{%
  \textsuperscript{\@fnsymbol{#1}}%
}
\makeatother

\title{The Effects of Mild Over-parameterization on the Optimization Landscape of Shallow ReLU Neural Networks}

\usepackage{times}
\author{Itay Safran\thanks{equal contribution}}
\author{Gilad Yehudai\printfnsymbol{1}}
\author{Ohad Shamir}
\affil{Weizmann Institute of Science}
\date{}

\begin{document}

\maketitle

\begin{abstract}
    We study the effects of mild over-parameterization on the optimization landscape of a simple ReLU neural network of the form $\mathbf{x}\mapsto\sum_{i=1}^k\max\{0,\mathbf{w}_i^{\top}\mathbf{x}\}$, in a well-studied teacher-student setting where the target values are generated by the same architecture, and when directly optimizing over the population squared loss with respect to Gaussian inputs. We prove that while the objective is strongly convex around the global minima when the teacher and student networks possess the same number of neurons, it is not even \emph{locally convex} after any amount of over-parameterization. Moreover, related desirable properties (e.g., one-point strong convexity and the Polyak-{\L}ojasiewicz condition) also do not hold even locally. On the other hand, we establish that the objective remains one-point strongly convex in \emph{most} directions (suitably defined), and show an optimization guarantee under this property. For the non-global minima, we prove that adding even just a single neuron will turn a non-global minimum into a saddle point. This holds under some technical conditions which we validate empirically. 
    These results provide a possible explanation for why recovering a global minimum becomes significantly easier when we over-parameterize, even if the amount of over-parameterization is very moderate.
\end{abstract}

\section{Introduction}
In recent years, a spur of theoretical papers studied how the training of neural networks benefits from over-parameterization, namely the use of more neurons than needed to express a good predictor (e.g., \citep{safran2016quality,du2018gradient2,safran2017spurious,allen2018learning,daniely2017sgd,li2018learning,cao2019generalization,andoni2014learning,jacot2018neural}). The vast majority of these papers focus on settings where a large amount of over-parameterization is needed (e.g., polynomial in some natural problem parameters). However, empirical studies such as in \citep{livni2014computational,safran2017spurious}  indicate that in many cases, very mild over-parameterization is required to successfully reach a global optimum, and sometimes adding even one or two neurons is enough. The aim of this paper is to theoretically study the effect of such mild over-parameterization. 

Specifically, we focus on a simple and well-studied student-teacher setting, where the labels are generated by a teacher network composed of a sum of $k$ neurons, and learned by a student network of the same architecture with $n$ neurons, using the squared loss with respect to some input distribution $\Dcal$:
\begin{equation}\label{eq:objective normal dist intro}
    F(\bw) ~:=~ F(\bw_1,\ldots,\bw_n)~=~ \mathbb{E}_{x\sim\mathcal{D}}\left[\left(\sum_{i=1}^n \sigma(\inner{\bw_i, \bx}) - \sum_{i=1}^k \sigma(\inner{\bv_i, \bx})\right)^2\right] ~.
\end{equation}
In the above, $\sigma:\mathbb{R}\rightarrow\mathbb{R}$ is some univariate activation function. This objective has been studied in quite a few recent works (e.g., \citep{zhong2017recovery,tian2017analytical,soltanolkotabi2019theoretical,yehudai2020learning,li2017convergence,arjevani2019spurious,arjevani2020symmetry}), perhaps most commonly when $\Dcal$ is a standard Gaussian, and $\sigma$ is the ReLU function. This will also be the setting we focus on in this paper.

Our paper is motivated by the empirical findings in \citep{safran2017spurious}. In that paper, the authors prove that the objective above possesses local minima which are not global, and empirically show that gradient descent with standard initialization does tend to get stuck in them when $n=k$. However, significantly fewer local minima are encountered already when $n=k+1$, and when $n=k+2$, no local minima were encountered at all for values of $n\le 20$ (see Table 2 in \citep{safran2017spurious}). Despite the progress made in understanding the loss surface and the dynamics of optimization techniques on the objective in \eqref{eq:objective normal dist intro}, to the best of our knowledge, current deep learning theory is unable to explain why such mild over-parameterization helps gradient methods to recover the global minimum in this setting. This leads us to the following question:
\begin{quote}
    \emph{What are the geometrical effects of mild over-parameterization on the objective function, which facilitate the use of common optimization techniques for recovering the global minimum?}
\end{quote}

In this paper, we take a few steps in understanding the above question in the context of \eqref{eq:objective normal dist intro}, under the standard setting where $\Dcal$ is a standard Gaussian distribution, the $\bv_i$'s are orthogonal and of unit norm, and $\sigma$ is the popular ReLU activation function. Our contributions are as follows:
\begin{itemize}[leftmargin=*]
    \item
    First, we provide a full characterization of all twice differentiable points and all global minima of the objective (\thmref{thm:global_minima_form} and \lemref{lem:hessian is differentiable}). We then formally prove that \emph{without} over-parameterization ($n=k$), the objective is strongly convex in a neighborhood of every global minimum (\thmref{thm:n=k strongly convex}). This property ensures that initializing close enough to the global minima (e.g.\ using a tensor initialization \citep{zhong2017recovery}), gradient descent with small enough step sizes will converge to it. We note that this in itself is not too surprising, and that a similar result was shown in \citep{zhong2017recovery,li2017convergence} for a slightly different setting.
    
    \item Next, we prove that perhaps surprisingly, in the over-parameterization regime ($n>k$) the local geometry around global minima changes significantly: The objective is not even locally \emph{convex} around global minima (\thmref{thm:over-parameterization not convex}). Moreover, we study other commonly used geometrical properties such as one-point strong convexity (also known as strong star-convexity) and Polyak-{\L}ojasiewicz (PL) condition (see \citep{karimi2016linear}) and show that these also do not hold, even locally, around global minima (\thmref{thm:over-parameterization not OPSC} and \thmref{thm:no PL}).
    
    \item On the flip side, we show that our objective is one-point strongly convex \emph{in most directions} -- that is, there is a significant set of points around the global minima that satisfy one-point strong convexity (\thmref{thm:almost OPSC}). This allows us to prove an optimization guarantee using gradient descent with small perturbations, for functions that satisfy this property in a simplified setting \revision{and under a certain technical assumption} (\thmref{thm: OPSC in most directions optimization}).

    \item
    Turning to the non-global minima, we prove that for \emph{any} such point, a slight over-parameterization consisting of `splitting' a  neuron into two neurons (having the same angle and summing to the original neuron) results in turning the non-global minimum into a saddle point with a direction of descent (\thmref{thm:norm_k}). This holds under a technical condition on the norm of the neurons at the local minima, which we justify empirically. This result demonstrates how even a tiny amount of over-parameterization helps eliminate non-global minima. 
\end{itemize}

The remainder of the paper is structured as follows: After discussing related work, we formalize our setting in \secref{sec:prelim}, and introduce relevant definitions and notations. Next, \secref{sec:global minima} investigates properties relating to the global minima of our objective, which includes their general form and the geometry of the objective function around them. Lastly, \secref{sec:non global minima} studies the non-global minima of our objective, showing when can we guarantee that splitting local minima will become a saddle points. 

\subsection{Related Work}\label{subsec:related work}

\textbf{Over-Parameterization.}
It was shown empirically that over-parameterized networks are easier to train, e.g.\ in \citep{livni2014computational,safran2017spurious}. Over-parameterization was extensively studied theoretically in several contexts and architectures, such as \citep{du2018gradient2,allen2018learning,daniely2017sgd,li2018learning,cao2019generalization,andoni2014learning,yehudai2019power,ghorbani2019linearized,ghorbani2019limitations,kamath2020approximate,allen2019can}. In particular, one very popular line of works argue that sufficiently over-parameterized networks behave similarly to kernel methods (in particular, the neural tangent kernel) or random feature methods. However, these approaches only apply for a very large amount of over-parameterization, as shown in several recent papers \cite{yehudai2019power, allen2019can, kamath2020approximate}. Thus, they cannot be used to explain why adding just a few neurons can significantly increase the probability of converging to a global minimum. In contrast, our results hold for \emph{any} amount of over-parameterization.
Notably, in \cite{yehudai2019power} it was shown that kernel methods (such as the NTK) cannot explain learnability of even a single ReLU neuron. This means that NTK cannot explain learnability of gradient descent on \eqref{eq:objective normal dist intro}, even for the simple case of $k=1$, unless $n$ is exponential in the input dimension.

%\note{edited above a bit}

% Several works \citep{yehudai2019power,ghorbani2019linearized,ghorbani2019limitations,kamath2020approximate,allen2019can} showed that the methods used to prove the results on over-parameterized neural networks actually make a reduction from neural networks to the well-known model of random features \citep{rahimi2008random,rahimi2009weighted}. In \cite{yehudai2019power} the authors show that using the approach used in the over-parameterization works, it is not possible to explain training a network where the data is generated from a single ReLU neuron ($\bx\mapsto\max\{0,\bw^{\top}\bx\}$), even when the data have standard Gaussian distribution, unless the amount of neurons is exponential in the input dimension (see also \cite{kamath2020approximate}). 
% In \cite{ghorbani2019limitations,ghorbani2019linearized} the authors also show the limitations of these approaches in the so called "lazy training" regime, focusing on the random features and NTK models.

% A few papers \citep{fukumizu2000local,mizutani2010analysis,petzka2018non} investigate conditions under which adding a neuron to an existing local minimum will either keep it a local minimum or turn it into a saddle point in the new resulting architecture. While similar in nature to the results we provide in \secref{sec:non global minima}, their setting is somewhat different, dealing with finite data and sigmoid (rather than ReLU) activations.

\textbf{ Over-Parameterization beyond NTK regimes.} Several papers considered theoretical analysis of over-parameterized models beyond the NTK regime. \citet{li2020learning} provide recovery and generalization guarantees for an objective similar to \eqref{eq:objective normal dist intro}, however their result only guarantees convergences to a solution with loss of about $1/d$ ($d$ being the input dimension) and not to arbitrarily small loss, and their analysis strongly relies on the symmetry of the teacher network, and therefore cannot be generalized to cases where this symmetry breaks. \cite{allen2019bcan} show an analysis that goes beyond NTK, where the target network is a one layer ResNet. \cite{daniely2020learning} provide an optimization guarantee on the problem of learning parity functions under some specific distribution using a 2-layer neural network. With that said, providing optimization guarantees for \eqref{eq:objective normal dist intro} for general $n$ and $k$ largely remains an open question.

%\note{In the above: (1) What does ``convergence guarantee of magnitude $1/d$'' mean? guarantees don't have magnitudes. (2) ''Resnet with'' -- rest is missing.}

\textbf{Previous works on \eqref{eq:objective normal dist intro}} Several works studied \eqref{eq:objective normal dist intro} under different assumptions such as \cite{tian2017analytical,soltanolkotabi2017learning,zhong2017recovery,yehudai2020learning,li2020learning}. In \cite{yehudai2020learning} the authors study the case of $n=k=1$, and show that even in this simple regime there exists distributions and activations in which gradient methods are unable to learn. On the other hand, they show that under mild assumptions on the activation and distribution it is possible to guarantee convergence to the global optimum, although in this simple case there are no non-global minima (there is a non-differentiable saddle point at the origin). This analysis does not generalize even to the case of $n=k=2$. 
In \cite{zhong2017recovery} the authors give optimization guarantees for the case of $n=k$ for general $k$, where $\mathcal{D}$ is standard Gaussian and some assumptions on $\sigma$ (which includes ReLU). Their method is to show that locally around global minima the objective is strongly convex, %similar to \thmref{thm:n=k strongly convex}
and use tensor initialization to initialize close enough to the global minimum. We prove a similar theorem (\thmref{thm:n=k strongly convex}), although there are a couple of small differences: The objective is a bit different, because in \cite{zhong2017recovery} the authors consider an empirical loss over a finite set of examples drawn i.i.d from $\mathcal{N}(0,I)$, whereas we consider the population loss. Moreover, we state an explicit numerical lower bound on the minimal eigenvalue of the Hessian at the minimum. On the other hand, \cite{zhong2017recovery} show the result for a general class of activation functions (including ReLU) and we show it specifically for the ReLU activation. They also specify how large the open neighborhood for which the objective is strongly convex, while we only state that there exists an open neighborhood without guarantees on its size. In any case, we note that this is not a main result of our paper, as we focus more on the over-parameterized case and this theorem is given mainly as a comparison to how over-parameterization significantly changes the optimization landscape. 

A similar analysis for the case of $n=k$ is done in \cite{li2017convergence} where the authors consider an architecture where the target neurons are close to unit vectors, and they show that the objective is one-point strongly convex (as opposed to strongly convex) around the global minimum. In \cite{arjevani2019spurious,arjevani2020symmetry,arjevani2020analytic} the authors study the properties of local minima of \eqref{eq:objective normal dist intro} in the case of $n=k$, standard Gaussian distribution and ReLU activation. They identify certain symmetries of the local minima and utilize them to characterize a certain family of local minima. 

\revision{In \cite{jin2017escape} the authors show how perturbed gradient descent can help in escaping saddle points. In our paper we also analyze perturbed gradient descent, and show that it can help to ensure convergence to a global minima, even when standard convexity-like properties (e.g. one-point strong convexity and PL) do not apply to the optimization landscape.}
% Other notable works include \cite{soudry2016no} and \cite{laurent2017multilinear} where they show that under certain assumptions there are no differentiable local minima of the objective they study. A major difference compared to us is that they study the empirical loss over a finite dataset, while here we study the population loss over a continuous distribution.

\section{Preliminaries}\label{sec:prelim}

\textbf{Terminology and Notation.}
We use $[n]$ as shorthand for $\{1,\ldots,n\}$. We denote the ReLU function ($z\mapsto\max\{0,z\}$) by $[\cdot]_+$. We denote vectors using bold-faced letters (e.g.\ $\bw$). We let barred bold-faced letters denote vectors normalized to unit length (i.e.\ $\bar{\bw}=\frac{\bw}{\norm{\bw}}$). Given two non-zero vectors $\bw,\bv\in\mathbb{R}^d$, we denote the angle between them using $\theta_{\bw,\bv}=\arccos\p{\bar{\bw}^{\top}\bar{\bv}}$. Unless stated otherwise, we denote by $\norm{\cdot}$ the standard Euclidean norm. We denote the matrix with all zero entries of size $m\times n$ by $\pmb{0}_{m\times n}$. For $\bw_1,\dots,\bw_n\in\mathbb{R}^d$ denote by $\bw_1^n = (\bw_1,\dots,\bw_n)\in\mathbb{R}^{n\cdot d}$ their concatenation. For symmetric matrices $A,B$ we say that $A\succeq B$ if $A-B$ is positive semi-definite (PSD). Recall that a function $f:\mathbb{R}^d\rightarrow\mathbb{R}$ that is twice continuously differentiable is said to be strongly convex in $A\subseteq \mathbb{R}^d$ iff there is a constant $\lambda > 0$ such that $\nabla^2f(\bx) \succeq \lambda I$ for any $\bx\in A$. It is convex if the above holds for $\lambda=0$.

\textbf{Setting.} In this paper we study a simple network in a student-teacher setting, assuming our data have a standard Gaussian distribution. In more detail, we fix the vectors in the teacher network $\bv_1,\dots,\bv_k\in\mathbb{R}^d$, and the population objective is:
%\begin{equation}\label{eq:objective no overparam}
%    F(\bw_1^k) = \mathbb{E}_{x\sim\mathcal{N}(0,I)}\left[\left(\sum_{i=1}^k [\inner{\bw_i, \bx}]_+ - \sum_{i=1}^k [\inner{\bv_i, \bx}]_+\right)^2\right].
%\end{equation} 
%To model possible over-parameterization, we assume that $n\ge k$ \note{change this to $n\ge1$} for some natural $n$ and we consider the objective given by
\begin{equation}\label{eq:objective normal dist}
    F(\bw_1^n) = \mathbb{E}_{x\sim\mathcal{N}(0,I)}\left[\left(\sum_{i=1}^n [\inner{\bw_i, \bx}]_+ - \sum_{i=1}^k [\inner{\bv_i, \bx}]_+\right)^2\right].
\end{equation} 
Throughout this paper we always assume that $d \geq k$ (to model a high-dimensional setting). We also assume for simplicity that the target vectors $\bv_1,\dots,\bv_k$ are orthogonal with $\norm{\bv_i}=1$ for $i\in[k]$. This assumption is also made in \cite{safran2017spurious}, and approximately holds if $\bv_1,\ldots,\bv_k$ are chosen uniformly at random from the unit sphere and the dimension is high enough. We conjecture that all the results in the paper can be extended to general target vectors, and leave it to future work.

\textbf{Basic Properties of the Objective Function.} 
% For a standard Gaussian distribution, the objective function in \eqref{eq:objective normal dist} can be written down in closed form. Moreover, it is continuously differentiable if $\bw_i\neq 0$ for every $i\in [n]$, with explicit expressions for the Gradient and Hessian at any point (see \citep{cho2009kernel,brutzkus2017globally,safran2017spurious}). In particular, we will need the following explicit expression for the Hessian \citep[Section 4.1.1]{safran2017spurious}:
For a standard Gaussian distribution, the objective function in \eqref{eq:objective normal dist} can be written down in closed form \revision{(without expectation terms)}. Moreover, it is continuously differentiable if $\bw_i\neq 0$ for every $i\in [n]$, with explicit expressions for the Gradient and Hessian at any point (see \citep{cho2009kernel,brutzkus2017globally,safran2017spurious}). In particular, we will need an explicit expression for the Hessian from \citep[Section 4.1.1]{safran2017spurious}. For completeness we include the formal statement in \thmref{thm:hessian of objective}, from which we immediately get that the objective is twice continuously differentiable for every $\bw_1^n=(\bw_1,\ldots,\bw_n)$ where $\bw_i\neq 0$ for every $i\in[n]$ and there are no two $\bw_i,\bw_j$ with $\theta_{\bw_i,\bw_j}\in\{0,\pi\}$. To complete the picture we show that even when $\theta_{\bw_i,\bw_j}\in\{0,\pi\}$ for some $i\neq j$ the Hessian is well defined and continuous. The formal proof can be found in \appref{appen:prof of hessian differentiable}.

% It is clear from \thmref{thm:hessian of objective} that the Hessian is continuous for every $\bw_1^n=(\bw_1,\ldots,\bw_n)$ where the objective is differentiable and there are no two $\bw_i,\bw_j$ with $\theta_{\bw_i,\bw_j}\in\{0,\pi\}$. If $\theta_{\bw_i,\bw_j}\in\{0,\pi\}$ for some distinct $i,j\in[n]$, then $\bar{\bn}_{\bw_i,\bw_j}$ is undefined. To complete the picture, we prove that even in this situation, even though the formula above does not formally apply, the Hessian is well-defined and continuous:

\begin{lemma}\label{lem:hessian is differentiable}
    $ F(\bw_1^n) $ is twice continuously differentiable at any $ \bw_1^n=(\bw_1,\ldots,\bw_n) $ such that $\bw_i\neq \mathbf{0} $ for all $i\in[n]$.
\end{lemma}

\section{Effects of Over-parameterization on the Global Minima}\label{sec:global minima}
In this section we study the local geometric properties of the global minima of the objective in \eqref{eq:objective normal dist}. We first characterize all the global minima of the objective $F(\bw_1^n)$ for any $n\geq k$.

\begin{theorem}\label{thm:global_minima_form}
Suppose $\bw_1^n=(\bw_1,\ldots,\bw_n)$ is a global minimum of the objective in \eqref{eq:objective normal dist}. Then there exists a partition $ \bigcupdot_{i=1}^k I_i = [n] $ and $ \alpha_1,\ldots,\alpha_n\ge0 $ satisfying $\sum_{j\in I_i}\alpha_j=1$ and $ \bw_j =\alpha_j \bv_i$ for all $i\in[k]$ and $j\in I_i$.
\end{theorem}

The full proof can be found in \appref{apen:global minima proofs}. \thmref{thm:global_minima_form} states that for a global minimum, each vector $\bw_i$ must be equal to some target vector $\bv_j$ times some positive constant $\alpha_i$. In addition, the sum of all the constants, for all the $\bw_i$ in the direction of some $\bv_j$ must be equal to $1$. In particular, for the case of $n=k$ we get that the only global minima are those that for each target vector $\bv_j$ there is exactly one $\bw_i$ for which $\bw_i=\bv_j$, hence there are exactly $n!$ isolated global minima. For the case of $n>k$ there is a manifold consisting of infinitely many global minima. For example, if $n=k+1$, then the following is a global minimum for every $\alpha\in [0,1]$: $\bw_1 = \bv_1,\dots,~\bw_{n-1} = \bv_{n-1},~ \bw_n = \alpha\bv_n,~ \bw_{n+1} = (1-\alpha)\bv_n~.$

Combining \thmref{thm:global_minima_form} and \lemref{lem:hessian is differentiable}, we have a full characterization of all (twice continuously) differentiable global minima of the objective $F(\bw)$ for general $n \geq k$. More specifically, all minima that admit the form of \thmref{thm:global_minima_form} and in addition satisfy that $\bw_i\neq 0$ for all $i\in[n]$ are differentiable. In this section we will study local geometric properties of the differentiable local minima, distinguishing between two cases: exact parameterization ($n=k$) and over-parameterization ($n>k)$.

\subsection{Exact Parameterization}\label{sec:n=k}
We first consider the case of exact parameterization, where the labels are created by a teacher network with $k$ neurons, and learned by a student network with $k$ neurons. 
%In this case, by \thmref{thm:global_minima_form} there are exactly $k!$ global minima, where each is given by a different permutation of the neurons $\bw_1,\dots,\bw_k$. 
Even though the objective $F(\bw_1^k)$ in this case is not convex (at least for $k\ge2$, as there are $k!$ isolated global minima), we will show that locally around each global minimum it is actually strongly convex.

\begin{theorem}\label{thm:n=k strongly convex}
Suppose $n = k$. For every global minimum of the objective $F(\bw_1^k)$ in \eqref{eq:objective normal dist} we have that $\nabla^2 F(\bw_1^k) \succeq \left(\frac{1}{4} - \frac{1}{2\pi}\right)I$. Moreover, the objective is strongly convex around an open neighborhood of any global minimum.
\end{theorem}

Note that in the case of $n=k$, by \thmref{thm:global_minima_form} all the global minima are differentiable. The proof idea behind \thmref{thm:n=k strongly convex} is straightforward. The Hessian at the global minimum can be divided into a sum of two matrices, and we lower bound the smallest eigenvalue of these two matrices. Note that since the objective is twice continuously differentiable around any global minimum (in the case of $n=k$), and that the eigenvalue of a matrix is a continuous function we immediately get that in an open neighborhood of the global minimum all the eigenvalues of the Hessian are positive, hence the objective is locally strongly convex.

As discussed in the related work section, a similar result was shown in \citep{zhong2017recovery} for a slightly different setting. Although this result might give hope that such properties are also preserved when over-parameterizing, as we will show in the next subsection, the over-parameterized case has a completely different geometry. Thus, this kind of analysis is specific for exact parameterization.

\subsection{Over-Parameterization}\label{sec:over-parameterization}

% Here we consider the case of $n>k$. We want to stress that empirically learning a teacher network with $k$ neurons by using a student network with $n>k$ neurons yields much better results than learning with $n=k$ neurons (see \citet[Section 3]{safran2017spurious}). This might lead us to conjecture that the optimization landscape

% First we note that in this paper we study only differential global minima. By \thmref{thm:global_minima_form} and \lemref{lem:hessian is differentiable}, all the global minima for which all the neurons are non-zero are differentiable.

In the exact parameterization case, we showed that around the global minima the objective is strongly convex. Since empirically, over-parameterization tends to improve training performance, we might expect that it improves or at least maintains favorable geometric properties around the global minima. However, we now prove that perhaps surprisingly, under any amount of over-parameterization, the objective in \eqref{eq:objective normal dist} is not even \emph{locally convex} around any differentiable global minimum:

\begin{theorem}\label{thm:over-parameterization not convex}
Assume that $n>k$ and $d>1$ (recall that $d\geq k$, hence this assumption is trivially true for $k>1$). Then in every neighborhood of a differentiable global minimum of \eqref{eq:objective normal dist} there is a point at which the Hessian of the objective has a negative eigenvalue.
\end{theorem}
Since convexity of a differentiable function requires the Hessian to be positive semidefinite, we get that no local convexity property can hold. We note that the theorem's assumptions are mild, since by \thmref{thm:global_minima_form}, the objective function is typically differentiable at a global minimum and its neighborhood. To provide some intuition how a global minimum without a convex neighborhood might look like, see an example (using a different function) in \figref{fig:x2y2} in the \appref{apen:proofs from OPSC}.

\subsection{One-Point Strong Convexity and the PL condition}\label{subsec:OPSC}
Instead of having convexity with respect to all directions, it may be enough from an optimization point of view to have convexity in the direction of the global minimum. This motivates the following well-known definition (see e.g. \cite{lee2016optimizing,kleinberg2018alternative}):

\begin{definition}\label{def:OPSC}
Let $f:\mathbb{R}^d\rightarrow \mathbb{R}$ be a differentiable function.$f(\bx)$ is said to be \textbf{one-point strongly convex} (OPSC) in an open neighborhood $A\subseteq \mathbb{R}^d$ with respect to a local minimum $\by^*\in A$ if there exists $\lambda >0$ such that for every $\bx\in A$: $\frac{1}{\norm{\bx-\by^*}^2}\inner{\nabla f(\bx),\bx-\by^*} \geq \lambda ~.$ If we further assume that $f(\bx)$ is twice differentiable, then it is OPSC in $A\subseteq \mathbb{R}^d$ if there exists $\lambda  >0$ such that for every $\bx\in A$: $ \frac{1}{\norm{\bx-\by^*}^2}(\bx-\by^*)^\top \nabla^2 f(\bx)(\bx-\by^*) \geq \lambda~,$ where $\nabla^2 f(\bx)$ is the Hessian of $f$ at $\bx$. We call such $\lambda$ the OPSC coefficient.
\end{definition}

The Hessian definition of one-point strong convexity can be easily derived from the gradient definition, in the same manner that the Hessian definition of strong convexity is derived from the gradient definition of strong convexity for twice continuously differentiable functions. In previous works it was shown that although an objective is not strongly-convex, it may be OPSC which is enough to show convergence to a minimum for certain local search algorithms (see e.g.\ \cite{li2017convergence}). Intuitively, this is because if $\by^*$ is a local minimum, the definition above implies that the gradient at $\bx$ is correlated with the direction to the minimum, and increases with the distance from $\by^*$. We note that one point convexity (i.e., taking $\lambda=0$) is not enough, as it may imply that the gradient is arbitrarily close to being orthogonal to the direction of the minimum (see also \citep{lee2016optimizing} for a discussion). 

% One may wonder why not consider functions which are \emph{one point convex}, i.e.\ taking $\lambda=0$ in Definition \ref{def:OPSC}. The reason is that, in this case the gradient may be almost orthogonal to the direction of the minimum, hence in order to ensure convergence using gradient descent there must be convexity in all directions, otherwise it is not clear how to choose the step size. In \cite{lee2016optimizing} the authors show an example of a one point convex function which cannot be optimized using gradient descent.

%Following the example in \figref{fig:x2y2}, we cannot really hope for OPSC for the objective in \eqref{eq:objective normal dist} in the over-parameterized case. The reason is that \thmref{thm:global_minima_form} reveals that in this case there is a connected manifold of global minima (on which the function is flat), instead of isolated minima as in the exact parameterization case. 

Unfortunately, we cannot really hope for OPSC for the objective in \eqref{eq:objective normal dist} in the over-parameterized case. The reason is that \thmref{thm:global_minima_form} reveals that in this case there is a connected manifold of global minima (on which the function is flat), instead of isolated minima as in the exact parameterization case.

Recall that if $n>k$ then the global minima form along a line on which each point is a global minimum (recall the discussion after \thmref{thm:global_minima_form}). One alternative formulation is to define OPSC on any point which is not a global minimum, but the problem of defining OPSC with respect to which point still stands. One way to overcome this problem is by considering OPSC with respect to a global minimum, only in directions which lead away from nearby global minima. This is formalized in the following definition (see \figref{fig:OPSC-orthogonal} in the supplementary material for an intuition):
\begin{definition}\label{def:eps orthogonal}
Let $\tilde{\bw}_1^n=(\tilde{\bw}_1,\dots,\tilde{\bw}_n)$ and $\epsilon > 0$. An $\pmb{\epsilon}$\textbf{-orthogonal Neighborhood} of $\tilde{\bw}_1^n$ is:
\begin{equation*}
    U_\epsilon^\perp(\tilde{\bw}_1^n) = \left\{ \bw_1^n=(\bw_1,\dots,\bw_n):~\forall i\in[n],~ \bw_i-\tilde{\bw}_i\perp \tilde{\bw}_i,~ \norm{\bw_i-\tilde{\bw}_i}\leq \epsilon\right\}.
\end{equation*}
\end{definition}
\revision{We refer to an $\epsilon$-neighborhood (i.e.\ not orthogonal) of $\tilde{\bw}_1^n$ as 
\[
U_\epsilon(\tilde{\bw}_1^n)=\left\{\bw_1^n:\forall i\in[n],~\norm{\bw_i-\tilde{\bw}_i}\leq \epsilon\right\}~.
\]
Note that this is different from the ``Standard" definition of a neighborhood of $\tilde{\bw}_1^n$, since here we allow each vector $\bw_i$ to be at distance $\epsilon$ from its corresponding $\tilde{\bw}_i$. }We could hope that the objective in \eqref{eq:objective normal dist} is OPSC at least in an $\epsilon$-orthogonal neighborhood of a global minimum, however this is not the case as shown in the following theorem. 

\begin{theorem}\label{thm:over-parameterization not OPSC}
Assume $n > k$, let $\epsilon >0$ and let $\tilde{\bw}_1^n=(\tilde{\bw}_1,\dots,\tilde{\bw}_n)$ be a differentiable global minimum of \eqref{eq:objective normal dist}. Then the objective is not OPSC with respect to $\tilde{\bw}_1^n$, even in an $\epsilon$-orthogonal neighborhood of $\tilde{\bw}_1^n$.
\end{theorem}

The theorem shows that the geometrical properties of our objective, although similar in some senses to the example of $f(x,y)=x^2y^2$, are still much more complex.

The full proof of the theorem can be found in \appref{apen:proofs from OPSC}. The intuition for the proof of the above theorem is the following: Assume that at the global minimum $\tilde{\bw}_1$ and $\tilde{\bw}_2$ are both directed in the same target vector $\bv_1$, i.e. $\tilde{\bw}_1=\alpha_1\bv_1$ and $\tilde{\bw}_2=\alpha_2\bv_1$ for some $\alpha_1,\alpha_2 >0$. We define a new point close to $\tilde{\bw}$ by taking $\bw_1 = \tilde{\bw}_1 + \epsilon\bu$ and $\tilde{\bw}_2 = \bw_2 - \epsilon\bu$ where $\bu\perp\tilde{\bw}_1,\tilde{\bw}_2$, and leave all the other vectors the same, thus ${\bw}_1^n\in U_\epsilon^\perp(\bw_1^n)$.
Intuitively, in the objective there are terms that to minimize them it is needed to make the $\bw_i$ close to the $\bv_j$, and other terms that will be minimized if the $\bw_i$'s are far apart. Since we haven't changed any of the vectors that are directed at the target vectors $\bv_2,\dots,\bv_k$, then most cancel out. Actually, the only terms that remain are the ones that are minimized when $\bw_1,\bw_2$ are close to $\bv_1$, and the ones that minimized when $\bw_1$ and $\bw_2$ are far apart from one another. But because of the way we defined ${\bw}_1^n$, these terms also \emph{almost} cancel out - they are of magnitude $O(\epsilon)$. 

Another useful property which became popular in recent years is the \textbf{Polyak- \L ojasiewicz} (PL) condition (\cite{polyak1963gradient,lojasiewicz1963topological}):
\begin{definition}
Let $f:\mathbb{R}^d\rightarrow\mathbb{R}$ be a differentiable function, and let $f^*$ be its optimal value. We say that $f(\bx)$ satisfies the \textbf{Polyak- \L ojasiewicz} (PL) condition in $\mathbb{A}\subseteq \mathbb{R}^d$ if there exists $\lambda >0$ such that for all $\bx\in A$: $\frac{1}{2}\norm{\nabla f(\bx)}^2 \geq \lambda (f(\bx)-f^*).$
\end{definition}

In \cite{karimi2016linear} the authors show that under mild smoothness assumptions on $f(\bx)$, if it satisfies the PL condition then gradient descent with a small enough step size have linear convergence rate to a global minimum. The PL condition became popular in recent years to show convergence of gradient descent for non-convex functions. For our objective, we will show a stronger result, that the PL condition does not apply even locally around any differentiable global minimum, and even if we restrict to an $\epsilon$-orthogonal neighborhood:
\begin{theorem}\label{thm:no PL}
    Assume $n>k$, let $\epsilon>0$ and let $\tilde{\bw}_1^n=(\tilde{\bw}_1,\dots,\tilde{\bw}_n)$ be a differentiable global minimum of \eqref{eq:objective normal dist}. Then the objective does not satisfy the PL condition, even in an $\epsilon$-orthogonal neighborhood of $\tilde{\bw}_1^n$.
\end{theorem}
The full proof can be found in \appref{apen:proofs from OPSC}. The proof idea is the same as \thmref{thm:over-parameterization not OPSC}, by showing that the same point chosen in the proof of that theorem also violates the PL condition.

\subsection{One-Point Strong Convex in Most Directions}\label{sec:OPSC in most directions}

As we previously showed, the objective surface in \eqref{eq:objective normal dist} around any differentiable global minimum is not locally convex, and also not necessarily locally OPSC, even if we restrict to an $\epsilon$-orthogonal neighborhood. The reason for the latter is that in this neighborhood, there are ``bad'' points which do not satisfy the OPSC condition. Thus, it is natural to ask how common are these ``bad'' points.

Here, we show that these points are fortunately rare, in the following sense: If we move away from a global minimum in some direction (inside its $\epsilon$-orthogonal neighborhood), then in ``most'' directions, we will arrive at points which do satisfy some form of the OPSC condition, as formalized in the theorem below. For this theorem, we consider the case where $n=m\cdot k$ where $m\geq 1$, and for simplicity consider the global minimum that for each target vector $\bv_i$ there are exactly $m$ neurons, each equal to $\frac{1}{m}\bv_i$ (however it is not too difficult to extend it to all differentiable global minima - see Remark \ref{rmk:works for all global minima}). We use a slightly different notation here, namely the vectorized form $\bw_1^n\in\mathbb{R}^{n\cdot d}$ here contains vectors $\bw_{i,j}\in\mathbb{R}^d$ for $i\in[k],j\in[m]$, to represent the assumption that at the global minimum there are $m$ neurons in the direction of the target $\bv_i$ for $i\in[k]$. 

\begin{theorem}\label{thm:almost OPSC}
Let $n = m\cdot k$ and let $\tilde{\bw}_1^{n}$ be the global minimum of \eqref{eq:objective normal dist} where $\tilde{\bw}_{i,j} = \frac{1}{m} \bv_i$ for $j\in[m]$ and $i\in[k]$. For $\epsilon >0$ let $ U_\epsilon^\perp(\tilde{\bw}_1^n)$ be the $\epsilon$-orthogonal neighborhood of $\tilde{\bw}_1^n$. Also, denote $\bg_{i,j} = \bw_{i,j} - \tilde{\bw}_{i,j}$, $\bg_i = \sum_{j=1}^m \bg_{i,j}$, $\bg= \sum_{j=1}^m\sum_{i=1}^k \bg_{i,j}$, denote by $H(\bw_1^n)$ the Hessian of the objective at $\bw_1^n$. Then if $\bw\in U_\epsilon^\perp(\tilde{\bw}_1^n)$ we have that:
%\begin{equation*}
%    \frac{1}{\norm{\bw_1^n - \tilde{\bw}_1^n}^2}\cdot(\bw_1^n - \tilde{\bw}_1^n)^\top H(\bw_1^n) (\bw_1^n - \tilde{\bw}_1^n) \geq \frac{1}{4\sum_{i=1}^k\sum_{j=1}^m\norm{\bg_{i,j}}^2}\cdot\left(\norm{\bg}^2 + \left(1-\frac{2}{\pi}\right)\sum_{i=1}^k\norm{\bg_i}^2 \right) - O(\sqrt{\epsilon})~,
    % \frac{3\norm{\bg}^2 + \sum_{i=1}^k\norm{\bg_i}^2}{12\sum_{i=1}^k\sum_{j=1}^m\norm{\bg_{i,j}}^2} 
%\end{equation*}
\begin{equation}\label{eq:curv}
    (\bw_1^n - \tilde{\bw}_1^n)^\top H(\bw_1^n) (\bw_1^n - \tilde{\bw}_1^n) \geq \frac{1}{4}\left(\norm{\bg}^2 + \left(1-\frac{2}{\pi}\right)\sum_{i=1}^k\norm{\bg_i}^2 \right) - O(\epsilon^{2.5})~,
\end{equation}
where the $O(\cdot)$ notation hides factors polynomial in $m$ and $k$. \end{theorem}

The theorem implies that the OPSC coefficient is determined by the norms of sums of differences between each $\bw_{i,j}$ and $\tilde{\bw}_{i,j}$. Thus, unless these differences exactly cancel out, the right hand side will generally be positive. This means that if we move away from the global minimum $\tilde{\bw}$ in some arbitrary direction, then the OPSC condition will generally hold w.r.t.\ $\tilde{\bw}$ and the current point $\bw$. We note that for simplicity's sake, the direction vector $\bw_1^n - \tilde{\bw}_1^n$ in \eqref{eq:curv} 
% in which the curvature of the objective is lower bounded 
is not normalized to unit length. 
% If we would have divided both sides of \eqref{eq:curv} by $\norm{\bw_1^n - \tilde{\bw}_1^n}^2$ , then the r.h.s.\ will have a term that is independent of $\epsilon$ in some directions, as we would see in the following examples. %\note{I did not understand this sentence (what is ``left-hand side approximation''? What needs to be normalized, and why?} \iedit{!!!!!! I just noticed that we can't just normalize \eqref{eq:curv} without normalizing the big O notation too (my mistake). Should we change the big O to have $1/\sqrt{\epsilon}$ instead? !!!!!!}

We now give a few examples for different values of $m$ and different points around the global minimum in order to give an intuition on which directions the one-point strong convexity applies:
\begin{example}
In the following examples, for brevity, we divide both sides of \eqref{eq:curv} by $\norm{\bw_1^n - \tilde{\bw}_1^n}^2$, this way the r.h.s.\ will have a term that is independent of $\epsilon$ in some directions, as we would see in the following examples
\begin{itemize}[leftmargin=*]
    \item Consider the case where $m=1$, meaning that $n=k$. This is the exact parameterization case, in this case we get by the theorem that:
    \[
     \frac{1}{\norm{\bw_1^n - \tilde{\bw}_1^n}^2}\cdot(\bw_1^n - \tilde{\bw}_1^n)^\top H(\bw_1^n) (\bw_1^n - \tilde{\bw}_1^n) \geq \frac{1}{4}-\frac{1}{2\pi} + \frac{\norm{\bg}^2  }{4\sum_{i=1}^k\norm{\bg_{i}}^2} - O(\sqrt{\epsilon})~.
    \] 
    This result conforms with our finding in \thmref{thm:n=k strongly convex} that for exact parameterization, the objective is strongly convex.
    % (and not just one-point strongly convex). Note \thmref{thm:n=k strongly convex} is stronger, since there we proved in it strongly-convex - i.e. in any direction not just in the direction of the global minimum.
    \item Assume that for every target vector $\bv_i$ we have that $\bw_{i,j}$ are equal for every $j\in[m]$. In this case:
    \[
     \frac{1}{\norm{\bw_1^n - \tilde{\bw}_1^n}^2}\cdot(\bw_1^n - \tilde{\bw}_1^n)^\top H(\bw_1^n) (\bw_1^n - \tilde{\bw}_1^n) \geq m\cdot\left(\frac{1}{4}-\frac{1}{2\pi}\right) + \frac{\norm{\bg}^2 }{4\sum_{i=1}^k\sum_{j=1}^m\norm{\bg_{i,j}}^2} - O(\sqrt{\epsilon})~.
    \]
    In this case the function is OPSC towards the global minimum $\tilde{\bw}_1^n$, assuming $\epsilon$ is not too large. Note that the $m$ term is a scaling factor that appears due to the over-parameterization. 
    \item Assume that for every target vector $\bv_i$ we have that $\sum_{j=1}^m\bw_{i,j} = \pmb{0}$. In this case $\frac{1}{\norm{\bw_1^n - \tilde{\bw}_1^n}^2}\cdot(\bw_1^n - \tilde{\bw}_1^n)^\top H(\bw_1^n) (\bw_1^n - \tilde{\bw}_1^n)$ is of magnitude at most $O(\sqrt{\epsilon})$. This case is similar in nature to what was shown in \thmref{thm:over-parameterization not OPSC} where the function is not OPSC.
\end{itemize}

\end{example}

\begin{remark}\label{rmk:works for all global minima}
In the theorem, we chose a specific global minimum for simplicity. The theorem can be readily extended to any differentiable global minimum $\tilde{\bw}$, at the cost of having inside the big-$O$ notation factors polynomial in  $\min_{i,j}\norm{\bw_{i,j}}^{-1}$ (which for our global minimum reduce to factors polynomial in $m$). We leave an exact analysis to future work.  
\end{remark}

\subsection{Optimization Under OPSC in Most Directions}

Until now we have shown that although several standard properties which guarantee convergence with gradient descent (convexity, OPSC and PL condition) are not satisfied by our objective, it does satisfy another property - OPSC in most directions. In this subsection we show that, at least in certain cases, this property is enough to ensure convergence. 

First, we note that in \thmref{thm:almost OPSC} there is a negative $O(\epsilon^{2.5})$ term. In the proof the sign of this term is not clearly determined, and further analysis will be needed to do so, which we leave for future work. With that said, we conjecture that this term is actually non-negative, at least in a close enough neighborhood of the global minimum. We also conjecture that this is true in a standard neighborhood of the global minimum, instead of an $\epsilon$-orthogonal neighborhood as stated in the theorem. We state this formally in the following:

\begin{conjecture}\label{conj:epsilon}
In the setting of \thmref{thm:almost OPSC} and under the same assumptions, we have that:
%\[
%    \frac{1}{\norm{\bw_1^n - \tilde{\bw}_1^n}^2}\cdot(\bw_1^n - \tilde{\bw}_1^n)^\top H(\bw_1^n) (\bw_1^n - \tilde{\bw}_1^n) \geq \frac{1}{4\sum_{i=1}^k\sum_{j=1}^m\norm{\bg_{i,j}}^2}\cdot\left(\norm{\bg}^2 + \left(1-\frac{2}{\pi}\right)\sum_{i=1}^k\norm{\bg_i}^2 \right)
%\]
\[
    (\bw_1^n - \tilde{\bw}_1^n)^\top H(\bw_1^n) (\bw_1^n - \tilde{\bw}_1^n) \geq \frac{1}{4}\left(\norm{\bg}^2 + \left(1-\frac{2}{\pi}\right)\sum_{i=1}^k\norm{\bg_i}^2 \right)~,
\]
in a standard $\epsilon$-neighborhood of every global minimum $\tilde{\bw}_1^n$, where $\bw_{i,j}\neq 0$ for all $i,j$.
\end{conjecture}

We conduct thorough experiments to verify this conjecture empirically. They can be seen in Appendix~\ref{app:curv_experiment}.

We would like to show that under Conjecture \ref{conj:epsilon}, initializing close enough to the global minimum would ensure convergence using gradient methods. Using standard gradient descent will not be enough here, since there are points for which the OPSC parameter is zero (even under the above assumption). To ensure convergence we need to add random noise to the optimization process which can help to escape those "bad" points.

We use a simple form of perturbed gradient descent, for the exact algorithm, see Appendix \ref{appen: perturbed GD}. In simple words, the algorithm receives an initialized weights $\bw_1^n(0)$, a learning rate $\eta$ and noise level $\alpha$. At each iteration the algorithm updates the weights w.r.t the loss function $F$ similarly to gradient descent, and adds a perturbation in a random direction with magnitude $\alpha$. The perturbation is in the same direction for all the learned vectors $\bw_1,\dots,\bw_n$.

We show convergence for a general function that have the property from \thmref{thm:almost OPSC}, under an assumption similar to Conjecture \ref{conj:epsilon}. Even under this assumption, the OPSC parameter may be zero (or arbitrarily small) at some points. Nevertheless, using perturbed gradient descent we can show the following:

\begin{theorem}\label{thm: OPSC in most directions optimization}
Let $F:\reals^{d\cdot n}\rightarrow\reals$ and assume that it achieves a global minimum at $\tilde{\bw}_1^n = (\tilde{\bw}_1,\dots,\tilde{\bw}_n)\in\reals^{d\cdot n}$. Assume that there is an $\epsilon\in (0,1] $ and $\lambda >0$ such that in an $\epsilon$-neighborhood of $\tilde{\bw}_1^n$ the function $F$ is twice differentiable, has an $L$-Lipschitz gradient, and we have that 
\begin{align}\label{eq:optimization assumption}
            (\bw_1^n - \tilde{\bw}_1^n)^\top H(\bw_1^n) (\bw_1^n - \tilde{\bw}_1^n) \geq \lambda\norm{\bg}^2
\end{align}
where $ H(\bw_1^n)$ is the Hessian of $F$ at $\bw_1^n$, $\bg_i = \bw_i - \tilde{\bw}_i$ and $\bg = \sum_{i=1}^n \bg_i$. Let $\delta > 0$. Then, initializing $\bw_1^n(0)$ in an $\epsilon$-neighborhood of $\tilde{\bw}_1^n$ and using perturbed gradient descent (\algref{alg:pgd}) with learning rate $\eta < \frac{\lambda\delta^2}{64L^2}$ and noise $\alpha = \frac{\delta}{4n}$, after $T> \frac{\log\left(\delta\right)}{\log\left(1 - \frac{\eta\lambda\delta^2}{64}\right)}$ iterations w.p $> 1-Te^{-\Omega(d)}$ (over the random perturbations) we have that $\norm{\bw_1^n(T) - \tilde{\bw}_1^n}^2 \leq \delta$.  
\end{theorem}

Note that the OPSC condition in this theorem is almost the same as in \thmref{thm:almost OPSC} for the case of having the property in a standard $\epsilon$-neighborhood of the minimum. In this case, the $\norm{\bg_i}^2$ terms can be absorbed in the $\norm{\bg}^2$ terms (by increasing the constant $\lambda)$.
% , and we can multiply both sides of the inequality by $\norm{\bw_1^n - \tilde{\bw}_1^n}^2$ to get the simplified condition in \eqref{eq:optimization assumption}.

The full proof can be found in \appref{appen: proof of optimization thm}. The idea is to split the analysis into two cases: %proof intuition is to analyze the optimization process when there are two cases:
(1) $\norm{\bg}^2$ is not too small, hence a single gradient step will get $\bw_1^n$ closer to $\tilde{\bw}_1^n$; (2) $\norm{\bg}^2$ is very small, but the perturbation from the algorithm will help escape from those bad points.

\thmref{thm: OPSC in most directions optimization} shows that even when the function is non-convex, if it has the OPSC in most directions property, gradient descent with small perturbations converges to a global minimum.  
% In \thmref{thm:almost OPSC} we showed that this property holds only in an $\epsilon$-orthogonal neighborhood (rather than a standard neighborhood), it is an interesting question whether this is enough to ensure convergence, and we leave it for future research.

% \note{add in the discussion that in order to extend this theorem it would require to extend \thmref{thm:almost OPSC} to larger neighborhoods of the global minimum.}

\vspace{-10pt}

\section{Effects of Over-parameterization on Non-global Minima}\label{sec:non global minima}
\vspace{-5pt}

Having considered the effects of over-parameterization on the global minima of \eqref{eq:objective normal dist}, in this section we turn to study the effects of over-parameterization on the non-global minima. In what follows, we define $H_{i,i}(\bw_1^n)' \coloneqq H_{i,i}(\bw_1^n) - \frac{1}{2}I$
% \[
%     H_{i,i}(\bw_1^n)' \coloneqq H_{i,i}(\bw_1^n) - \frac{1}{2}I~=~\sum_{j\in [n]\setminus \{i\}} h_1(\bw_i,\bw_j) - \sum_{j\in [k]} h_1(\bw_i,\bv_j)
% \]
, the component of the $i$-th diagonal block of the Hessian at $\bw_1^n$, without the $\frac{1}{2}I$ term (see \eqref{eq:hess H_ii}). When the point $\bw_1^n$ is clear from context, we let $H_{i,i}'$ be shorthand for $H_{i,i}(\bw_1^n)'$. Given a point $\bw_1^n\in\reals^{nd}$, we let $\bw_1^{n}(\alpha,i) = (\bw_1,\ldots,\bw_{i-1},\alpha\bw_i,(1-\alpha)\bw_i,\bw_{i+1},\ldots,\bw_n)\in\reals^{(n+1)d}$ denote the point obtained from 
% over-parameterizing the objective and 
splitting the $i$-th neuron $\bw_1^n$ into two neurons, one with a factor of $\alpha$ and the other with a factor of $1-\alpha$. All proofs of theorems appearing in this section can be found in \appref{sec:saddle_proofs}.

\subsection{Over-parameterization Turns Non-global Minima into Saddle Points}

As was empirically shown in \citep{safran2017spurious}, very mild over-parameterization (adding one or two neurons) suffices for significantly improving the probability of gradient descent to recover global minima of \eqref{eq:objective normal dist}. Thus, it is interesting to understand how such minimal over-parameterization changes the optimization landscape, in a way that helps local search methods avoid non-global minima. One major obstacle for pursuing this direction is that only certain non-global minima of \eqref{eq:objective normal dist} are known to have an explicit characterization \citep{arjevani2020analytic}. However, if we are already given a local minimum $\bw_1^n$, a simple way to generate additional critical points is to split the $i$-th neuron to obtain a point $\bw_1^n(\alpha,i)$, for any $\alpha\in(0,1)$ (see \lemref{lem:over_hess} for a formal statement). Our main result in this section is to demonstrate that if $n\le k$ and $\sum_{i=1}^n\norm{\bw_i}\le k$, then there exists a neuron that when split, the critical point obtained is a saddle point:

\begin{theorem}\label{thm:norm_k}
    Suppose $n\le k$, $\bw_1^n = (\bw_1,\ldots,\bw_n)$ is a non-global minimum of the objective in \eqref{eq:objective normal dist} such that $\sum_{i=1}^n\norm{\bw_i}\le k$. Then $F$ is twice continuously differentiable at $\bw_1^n$ and there exists a neuron $\bw_i$ such that $\bw_1^n(\alpha,i)$ is a saddle point for all $\alpha\in(0,1)$. Moreover, for $\alpha\in\{0,1\}$ we have that $\bw_1^n(\alpha,i)$ is not a local minimum of $F$.
\end{theorem}

Although we do not have a proof that the assumption $\sum_{i=1}^n\norm{\bw_i}\le k$ holds for all minima of the objective, in \subsecref{subsec:experiment} we demonstrate empirically that this appears to be the case, at least for the minima found by gradient descent.
% We note that currently we do not have a proof that the assumption $\sum_{i=1}^n\norm{\bw_i}\le k$ holds for all minima of the objective. However, in \subsecref{subsec:experiment} we demonstrate empirically that this appears to be the case, at least for the minima found by gradient descent.
Moreover, this assumption provably holds for the global minima (see \thmref{thm:global_minima_form}). Finally, we can prove the following weaker bound for any minimum:
\begin{proposition}\label{prop:norm sum nk}
     Suppose $n\geq 1$, $\bw_1^n = (\bw_1,\ldots,\bw_n)$ is a local minimum of the objective in \eqref{eq:objective normal dist}. Then $\sum_{i=1}^n\norm{\bw_i}\le kn$.
\end{proposition}

We also remark that the theorem applies to the critical points obtained when splitting local minima where $n\le k$, and it is possible that there are new local minima formed when $n>k$ that did not exist when $n\le k$, which our analysis does not touch upon. However, current empirical evidence (see \citep{safran2017spurious}) suggests that these minima are less common and pose a much less significant obstacle to optimization.

Combining \thmref{thm:norm_k} with \thmref{thm:global_minima_form}, we see that global minima can be split arbitrarily and remain global minima, whereas non-global minima can only be split in restricted ways before turning into saddle points. This provides an indication for why over-parameterization makes the landscape more favorable to optimization, and possibly explains why recovering the global minimum becomes easier when over-parameterizing.

The key in proving \thmref{thm:norm_k} is the observation that when we split the $i$-th neuron in $\bw_1^n$, we obtain a critical point of $F$, and the Hessian of this new point cannot be PSD if the $H_{i,i}'$ is not PSD.
% and that the $i$-th diagonal block of the Hessian at $\bw_1^n$ (having dimensions $d\times d$ and given by $0.5I+H_{i,i}'$) is turned into a $2d\times2d$ block of the following form:
% \begin{equation}\label{eq:hess block}
%         \p{\begin{matrix}
%             \frac{1}{2}I + \frac{1}{\alpha}H_{i,i}^{\prime} & \frac{1}{2}I \\
%             \frac{1}{2}I & \frac{1}{2}I + \frac{1}{1-\alpha}H_{i,i}^{\prime}
%         \end{matrix}}.
% \end{equation}
% Next, we show that $H_{i,i}'$ is not PSD, hence the block matrix above is not PSD, and consequentially the Hessian at $\bw_1^n(\alpha,i)$ is not PSD, implying the theorem.
Indeed, the role of the norm sum bound assumption in \thmref{thm:norm_k} is to show that there must exist at least one neuron having a component $H_{i,i}'$ which is not PSD. However, if we make the stronger assumption that \emph{for several} $i$'s, $H_{i,i}'$ is not PSD (which based on the proof of \thmref{thm:norm_k}, we can expect to happen when for each such $i$, $\bw_i$ has roughly unit norm and $\min_{j\in[k]}\theta_{\bw_i,\bv_j}$ is not too small) then this implies a stronger result, that when we split any such neuron $i$ with non-PSD $H_{i,i}'$, this would necessarily turn the local minimum into a saddle point. More formally, we have the following theorem:
\begin{theorem}\label{thm:all_blocks_neg}
    Suppose $n\ge1$, $\bw_1^n = (\bw_1,\ldots,\bw_n)$ is a differentiable, non-global minimum of the objective in \eqref{eq:objective normal dist}. Then for all $i\in[n]$ such that $H_{i,i}^{\prime}$ is not PSD, $\bw_1^n(\alpha,i)$ is a saddle point for all $\alpha\in(0,1)$. Moreover, for $\alpha\in\{0,1\}$ and any such $i\in[n]$, $\bw_1^n(\alpha,i)$ is not a local minimum of $F$.
\end{theorem}
In particular, if $H_{i,i}'$ is not PSD for \emph{all} $i\in[n]$, then splitting $\bw_1^n$ would necessarily turn it into a saddle point, regardless of which neuron is being split. In the next subsection, we show empirically that this indeed appears to be the case in general.

\subsection{An Experiment}\label{subsec:experiment}

In this subsection,\footnote{The code can be found at \url{https://github.com/ItaySafran/Overparameterization}} we wish to substantiate empirically the assumption $\sum_{i=1}^{n}\norm{\bw_i}\leq k$ made in \thmref{thm:norm_k}, as well as the claim that $H_{i,i}'$ tends to be a non-PSD matrix. To that end, for each $n=k$ between $6$ and $100$, we ran $500$ instantiations of gradient descent on the objective in \eqref{eq:objective normal dist}, each using an independent and standard Xavier random initialization and a fixed step size of $5/k$,\footnote{Empirically, this step size resulted in satisfactory convergence rates for all values of $k$ we tested.} till the norm of the gradient was at most $10^{-12}$. \revision{Moreover, we ran 100 additional instantiations where we initialized at a point having a large norm-sum of roughly $2k^2$ (note that \propref{prop:norm sum nk} guarantees that there are no minima with norm-sum more than $k^2$)}. We identified points that were equivalent up to permutations of the neurons and their coordinates (up to Frobenius norm of at most $5\cdot10^{-9}$). For each group of equivalent points, we computed the spectrum of the Hessian to ensure that its minimal eigenvalue is positive (using floating point computations), which was always the case.

Once the local minima we converged to were processed, we first validated the norm sum assumption of $\sum_{i=1}^k\norm{\bw_i}\le k$ which we made in \thmref{thm:norm_k}. \emph{All} local minima found in our experiment indeed satisfy this bound. Moreover, histogram plots of a few selected values for $k$ are presented in \figref{fig:norm sum}, suggesting that the norm sum tends to be tightly concentrated at a value slightly below $k$.

\begin{figure}[t]
\centering
{\includegraphics[width=5.4in]{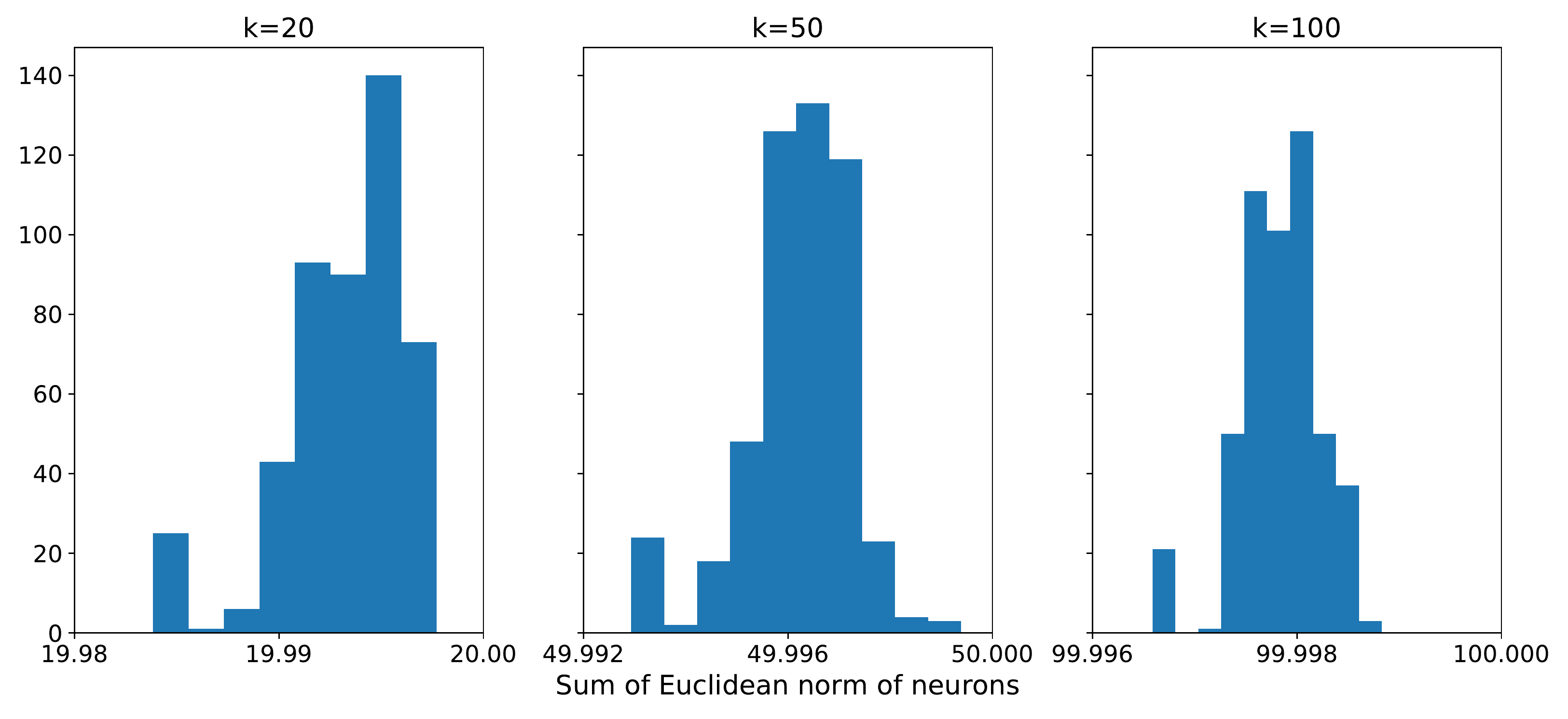}}
\caption{Histograms of the distributions of the sum of Euclidean norms of the neurons in the points converged to in the experiment, for $k=20,50,100$.} %The sum of norms seems concentrated around two centers, which correspond to the two types of points found (the first which corresponds to the large clusters are points having at most six distinct values -- e.g.\ \citep[Example 1]{safran2017spurious}, and the second which corresponds to the small clusters are points having eleven distinct values).}
\label{fig:norm sum}
\end{figure}

Next, we computed the eigenvalues of the $H_{i,i}'$ in the Hessians of the local minima found, for all $i$. As it turns out, \emph{all} block components for \emph{all} minima found have a negative eigenvalue, which by virtue of \thmref{thm:all_blocks_neg} implies that for any minimum point $\bw_1^k$ found, $\bw_1^k(\alpha,i)$ is a saddle point for all $i\in[k]$ and any $\alpha\in(0,1)$ (and not a minimum for $\alpha\in\{0,1\}$).

% \revision{
% \section{Discussion}
% }

\subsection*{Acknowledgements}
This research is supported in part by European Research Council (ERC) Grant 754705.

\setcitestyle{numbers}
\bibliographystyle{abbrvnat}
\bibliography{my_bib}

\begin{thebibliography}{35}
\providecommand{\natexlab}[1]{#1}
\providecommand{\url}[1]{\texttt{#1}}
\expandafter\ifx\csname urlstyle\endcsname\relax
  \providecommand{\doi}[1]{doi: #1}\else
  \providecommand{\doi}{doi: \begingroup \urlstyle{rm}\Url}\fi

\bibitem[Allen-Zhu and Li(2019{\natexlab{a}})]{allen2019bcan}
Z.~Allen-Zhu and Y.~Li.
\newblock What can resnet learn efficiently, going beyond kernels?
\newblock In \emph{Advances in Neural Information Processing Systems}, pages
  9017--9028, 2019{\natexlab{a}}.

\bibitem[Allen-Zhu and Li(2019{\natexlab{b}})]{allen2019can}
Z.~Allen-Zhu and Y.~Li.
\newblock Can {SGD} learn recurrent neural networks with provable
  generalization?
\newblock \emph{arXiv preprint arXiv:1902.01028}, 2019{\natexlab{b}}.

\bibitem[Allen-Zhu et~al.(2018)Allen-Zhu, Li, and Liang]{allen2018learning}
Z.~Allen-Zhu, Y.~Li, and Y.~Liang.
\newblock Learning and generalization in overparameterized neural networks,
  going beyond two layers.
\newblock \emph{arXiv preprint arXiv:1811.04918}, 2018.

\bibitem[Andoni et~al.(2014)Andoni, Panigrahy, Valiant, and
  Zhang]{andoni2014learning}
A.~Andoni, R.~Panigrahy, G.~Valiant, and L.~Zhang.
\newblock Learning polynomials with neural networks.
\newblock In \emph{International Conference on Machine Learning}, pages
  1908--1916, 2014.

\bibitem[Arjevani and Field(2019)]{arjevani2019spurious}
Y.~Arjevani and M.~Field.
\newblock Spurious local minima of shallow relu networks conform with the
  symmetry of the target model.
\newblock \emph{arXiv preprint arXiv:1912.11939}, 2019.

\bibitem[Arjevani and Field(2020{\natexlab{a}})]{arjevani2020analytic}
Y.~Arjevani and M.~Field.
\newblock Analytic characterization of the hessian in shallow relu models: A
  tale of symmetry.
\newblock \emph{arXiv preprint arXiv:2008.01805}, 2020{\natexlab{a}}.

\bibitem[Arjevani and Field(2020{\natexlab{b}})]{arjevani2020symmetry}
Y.~Arjevani and M.~Field.
\newblock Symmetry \& critical points for a model shallow neural network.
\newblock \emph{arXiv preprint arXiv:2003.10576}, 2020{\natexlab{b}}.

\bibitem[Brutzkus and Globerson(2017)]{brutzkus2017globally}
A.~Brutzkus and A.~Globerson.
\newblock Globally optimal gradient descent for a convnet with gaussian inputs.
\newblock In \emph{Proceedings of the 34th International Conference on Machine
  Learning-Volume 70}. JMLR. org, 2017.

\bibitem[Cao and Gu(2019)]{cao2019generalization}
Y.~Cao and Q.~Gu.
\newblock A generalization theory of gradient descent for learning
  over-parameterized deep {ReLU} networks.
\newblock \emph{arXiv preprint arXiv:1902.01384}, 2019.

\bibitem[Cho and Saul(2009)]{cho2009kernel}
Y.~Cho and L.~K. Saul.
\newblock Kernel methods for deep learning.
\newblock In \emph{Advances in neural information processing systems}, pages
  342--350, 2009.

\bibitem[Daniely(2017)]{daniely2017sgd}
A.~Daniely.
\newblock {SGD} learns the conjugate kernel class of the network.
\newblock In \emph{Advances in Neural Information Processing Systems}, pages
  2422--2430, 2017.

\bibitem[Daniely and Malach(2020)]{daniely2020learning}
A.~Daniely and E.~Malach.
\newblock Learning parities with neural networks.
\newblock \emph{arXiv preprint arXiv:2002.07400}, 2020.

\bibitem[Du et~al.(2018)Du, Zhai, Poczos, and Singh]{du2018gradient2}
S.~S. Du, X.~Zhai, B.~Poczos, and A.~Singh.
\newblock Gradient descent provably optimizes over-parameterized neural
  networks.
\newblock \emph{arXiv preprint arXiv:1810.02054}, 2018.

\bibitem[Ghorbani et~al.(2019{\natexlab{a}})Ghorbani, Mei, Misiakiewicz, and
  Montanari]{ghorbani2019limitations}
B.~Ghorbani, S.~Mei, T.~Misiakiewicz, and A.~Montanari.
\newblock Limitations of lazy training of two-layers neural network.
\newblock In \emph{Advances in Neural Information Processing Systems}, pages
  9108--9118, 2019{\natexlab{a}}.

\bibitem[Ghorbani et~al.(2019{\natexlab{b}})Ghorbani, Mei, Misiakiewicz, and
  Montanari]{ghorbani2019linearized}
B.~Ghorbani, S.~Mei, T.~Misiakiewicz, and A.~Montanari.
\newblock Linearized two-layers neural networks in high dimension.
\newblock \emph{arXiv preprint arXiv:1904.12191}, 2019{\natexlab{b}}.

\bibitem[Jacot et~al.(2018)Jacot, Gabriel, and Hongler]{jacot2018neural}
A.~Jacot, F.~Gabriel, and C.~Hongler.
\newblock Neural tangent kernel: Convergence and generalization in neural
  networks.
\newblock In \emph{Advances in neural information processing systems}, pages
  8571--8580, 2018.

\bibitem[Jin et~al.(2017)Jin, Ge, Netrapalli, Kakade, and
  Jordan]{jin2017escape}
C.~Jin, R.~Ge, P.~Netrapalli, S.~M. Kakade, and M.~I. Jordan.
\newblock How to escape saddle points efficiently.
\newblock In \emph{International Conference on Machine Learning}, pages
  1724--1732. PMLR, 2017.

\bibitem[Kamath et~al.(2020)Kamath, Montasser, and
  Srebro]{kamath2020approximate}
P.~Kamath, O.~Montasser, and N.~Srebro.
\newblock Approximate is good enough: Probabilistic variants of dimensional and
  margin complexity.
\newblock \emph{arXiv preprint arXiv:2003.04180}, 2020.

\bibitem[Karimi et~al.(2016)Karimi, Nutini, and Schmidt]{karimi2016linear}
H.~Karimi, J.~Nutini, and M.~Schmidt.
\newblock Linear convergence of gradient and proximal-gradient methods under
  the polyak-{\l}ojasiewicz condition.
\newblock In \emph{Joint European Conference on Machine Learning and Knowledge
  Discovery in Databases}, pages 795--811. Springer, 2016.

\bibitem[Kleinberg et~al.(2018)Kleinberg, Li, and
  Yuan]{kleinberg2018alternative}
B.~Kleinberg, Y.~Li, and Y.~Yuan.
\newblock An alternative view: When does sgd escape local minima?
\newblock In \emph{International Conference on Machine Learning}, pages
  2698--2707, 2018.

\bibitem[Lee and Valiant(2016)]{lee2016optimizing}
J.~C. Lee and P.~Valiant.
\newblock Optimizing star-convex functions.
\newblock In \emph{2016 IEEE 57th Annual Symposium on Foundations of Computer
  Science (FOCS)}, pages 603--614. IEEE, 2016.

\bibitem[Li and Liang(2018)]{li2018learning}
Y.~Li and Y.~Liang.
\newblock Learning overparameterized neural networks via stochastic gradient
  descent on structured data.
\newblock In \emph{Advances in Neural Information Processing Systems}, pages
  8168--8177, 2018.

\bibitem[Li and Yuan(2017)]{li2017convergence}
Y.~Li and Y.~Yuan.
\newblock Convergence analysis of two-layer neural networks with relu
  activation.
\newblock In \emph{Advances in neural information processing systems}, pages
  597--607, 2017.

\bibitem[Li et~al.(2020)Li, Ma, and Zhang]{li2020learning}
Y.~Li, T.~Ma, and H.~R. Zhang.
\newblock Learning over-parametrized two-layer neural networks beyond ntk.
\newblock In \emph{Conference on Learning Theory}, pages 2613--2682. PMLR,
  2020.

\bibitem[Livni et~al.(2014)Livni, Shalev-Shwartz, and
  Shamir]{livni2014computational}
R.~Livni, S.~Shalev-Shwartz, and O.~Shamir.
\newblock On the computational efficiency of training neural networks.
\newblock In \emph{Advances in Neural Information Processing Systems}, pages
  855--863, 2014.

\bibitem[Lojasiewicz(1963)]{lojasiewicz1963topological}
S.~Lojasiewicz.
\newblock A topological property of real analytic subsets.
\newblock \emph{Coll. du CNRS, Les {\'e}quations aux d{\'e}riv{\'e}es
  partielles}, 117:\penalty0 87--89, 1963.

\bibitem[Polyak(1963)]{polyak1963gradient}
B.~T. Polyak.
\newblock Gradient methods for minimizing functionals.
\newblock \emph{Zhurnal Vychislitel'noi Matematiki i Matematicheskoi Fiziki},
  3\penalty0 (4):\penalty0 643--653, 1963.

\bibitem[Safran and Shamir(2016)]{safran2016quality}
I.~Safran and O.~Shamir.
\newblock On the quality of the initial basin in overspecified neural networks.
\newblock In \emph{International Conference on Machine Learning}, pages
  774--782, 2016.

\bibitem[Safran and Shamir(2017)]{safran2017spurious}
I.~Safran and O.~Shamir.
\newblock Spurious local minima are common in two-layer relu neural networks.
\newblock \emph{arXiv preprint arXiv:1712.08968}, 2017.

\bibitem[Soltanolkotabi(2017)]{soltanolkotabi2017learning}
M.~Soltanolkotabi.
\newblock Learning relus via gradient descent.
\newblock In \emph{Advances in Neural Information Processing Systems}, pages
  2007--2017, 2017.

\bibitem[Soltanolkotabi et~al.(2019)Soltanolkotabi, Javanmard, and
  Lee]{soltanolkotabi2019theoretical}
M.~Soltanolkotabi, A.~Javanmard, and J.~D. Lee.
\newblock Theoretical insights into the optimization landscape of
  over-parameterized shallow neural networks.
\newblock \emph{IEEE Transactions on Information Theory}, 65\penalty0
  (2):\penalty0 742--769, 2019.

\bibitem[Tian(2017)]{tian2017analytical}
Y.~Tian.
\newblock An analytical formula of population gradient for two-layered relu
  network and its applications in convergence and critical point analysis.
\newblock In \emph{Proceedings of the 34th International Conference on Machine
  Learning-Volume 70}, pages 3404--3413. JMLR. org, 2017.

\bibitem[Yehudai and Shamir(2019)]{yehudai2019power}
G.~Yehudai and O.~Shamir.
\newblock On the power and limitations of random features for understanding
  neural networks.
\newblock In \emph{Advances in Neural Information Processing Systems}, pages
  6594--6604, 2019.

\bibitem[Yehudai and Shamir(2020)]{yehudai2020learning}
G.~Yehudai and O.~Shamir.
\newblock Learning a single neuron with gradient methods.
\newblock \emph{arXiv preprint arXiv:2001.05205}, 2020.

\bibitem[Zhong et~al.(2017)Zhong, Song, Jain, Bartlett, and
  Dhillon]{zhong2017recovery}
K.~Zhong, Z.~Song, P.~Jain, P.~L. Bartlett, and I.~S. Dhillon.
\newblock Recovery guarantees for one-hidden-layer neural networks.
\newblock In \emph{Proceedings of the 34th International Conference on Machine
  Learning-Volume 70}, pages 4140--4149. JMLR. org, 2017.

\end{thebibliography}

\appendix
\addcontentsline{toc}{section}{Appendices}
\section*{Appendices}

\section{Proof Of \lemref{lem:hessian is differentiable}}\label{appen:prof of hessian differentiable}

\begin{theorem}\label{thm:hessian of objective}
Let $n\geq k$, and let $\bw_1^n=(\bw_1,\dots,\bw_n)$ such that $\bw_i\neq 0$ for every $i\in[n]$. Denote by $H(\bw_1^n)$ the Hessian of $F(\bw_1^n)$ (the objective in \eqref{eq:objective normal dist}). It is an $(n\cdot d) \times (n\cdot d)$ matrix, where for ease of notations we view $H(\bw_1^n)$ as a $n\times n$ block matrix where each entry is a block of size $d\times d$. For every $i\in[n]$ the diagonal block entry of the Hessian is:
\begin{equation}\label{eq:hess H_ii}
    H_{i,i}(\bw_1^n) = \frac{1}{2}I + \sum_{j\neq i} h_1(\bw_i,\bw_j) - \sum_{j\in [k]} h_1(\bw_i,\bv_j)
\end{equation}
where 
\begin{equation}\label{eq:h1}
    h_1(\bw,\bv) = \frac{\sin(\theta_{\bw,\bv})\|\bv\|}{2\pi\|\bw\|}\left(I - \bar{\bw}\bar{\bw}^\top + \bar{\bn}_{\bv,\bw}\bar{\bn}_{\bv,\bw}^\top\right)
\end{equation}
and ${\bn}_{\bv,\bw} = \bar{\bv} - \cos(\theta_{\bw,\bv})\bar{\bw}$. For every $i,j\in[n]$ with $i\neq j$ the off-diagonal entry of the Hessian is
$
H_{i,j}(\bw_1^n) = h_2(\bw_i,\bw_j)  
$
where 
\begin{equation}\label{eq:h2}
    h_2(\bw,\bv)= \frac{1}{2\pi}\left((\pi- \theta_{\bw,\bv})I + \bar{\bn}_{\bw,\bv}\bar{\bv}^\top + \bar{\bn}_{\bv,\bw}\bar{\bw}^\top\right) ~. 
\end{equation} 
\end{theorem}

We will need the following auxiliary lemma which calculates $\lim_{\bw,\bv\to\bu}h_2(\bw,\bv)$ for $\bu\neq\mathbf{0}$
\begin{lemma}\label{lem:h2 parallel vectors}
Suppose $\bu\neq\mathbf{0}\in\mathbb{R}^d$. Then $\lim_{\bw,\bv\to\bu}h_2(\bw,\bv) = \frac{1}{2}I$.
\end{lemma}

\begin{proof}
By \thmref{thm:hessian of objective} we have that:
\begin{align}
    h_2(\bw,\bv) = \frac{1}{2\pi}\left((\pi-\theta_{\bw,\bv})I + \bar{\bn}_{\bw,\bv}\bar{\bv}^\top + \bar{\bn}_{\bv,\bw}\bar{\bw}^\top\right) = \frac{1}{2}I +\bar{\bn}_{\bw,\bv}\bar{\bv}^\top + \bar{\bn}_{\bv,\bw}\bar{\bw}^\top.
\end{align}
We will show that the second and third terms approach zero if $\bw,\bv\to\bu$. Define the shorthand $\theta\coloneqq \theta_{\bw,\bv}$, then we have:
\begin{align}
    \bar{\bn}_{\bw,\bv}\bar{\bv}^\top + \bar{\bn}_{\bv,\bw}\bar{\bw}^\top &= \frac{\bar{\bw}\bar{\bv}^\top- \cos(\theta)\bar{\bv}\bar{\bv}^\top}{\sin(\theta)} + \frac{\bar{\bv}\bar{\bw}^\top- \cos(\theta)\bar{\bw}\bar{\bw}^\top}{\sin(\theta)} \nonumber\\
    & = \frac{\bar{\bw}\bar{\bv}^\top - \cos(\theta)\bar{\bv}\bar{\bv}^\top + \bar{\bw}\bar{\bv}^\top\cos(\theta) - \bar{\bw}\bar{\bv}^\top\cos(\theta)}{\sin(\theta)} + \nonumber\\
    & +\frac{\bar{\bv}\bar{\bw}^\top - \cos(\theta)\bar{\bw}\bar{\bw}^\top + \bar{\bv}\bar{\bw}^\top\cos(\theta) - \bar{\bv}\bar{\bw}^\top\cos(\theta)}{\sin(\theta)} \nonumber\\
    & = \frac{\bar{\bw}\bar{\bv}^\top(1-\cos(\theta))}{\sin(\theta)} + \frac{\bar{\bv}\bar{\bw}^\top(1-\cos(\theta))}{\sin(\theta)} + \frac{(\bar{\bw}-\bar{\bv})\cos(\theta)\bar{\bv}^\top}{\sin(\theta)} + \frac{(\bar{\bv}-\bar{\bw})\cos(\theta)\bar{\bw}^\top}{\sin(\theta)} \nonumber\\
    & = \frac{(\bar{\bw}-\bar{\bv})(\bar{\bv}^\top-\bar{\bw}^\top)\cos(\theta)}{\sin(\theta)} + \frac{\bar{\bw}\bar{\bv}^\top(1-\cos(\theta))}{\sin(\theta)} + \frac{\bar{\bv}\bar{\bw}^\top(1-\cos(\theta))}{\sin(\theta)}. \label{eq:n_w,v v + n_v,w w}
\end{align}
If $\bw,\bv\to\bu$ the last two terms of \eqref{eq:n_w,v v + n_v,w w} go to zero, since the outer product results in a matrix of bounded norm that is multiplied by $(1-\cos(\theta))/\sin(\theta)$ which tends to zero (can be seen using  L'H\^opital's rule). For the first term, we will prove it is the zero matrix by showing that multiplying the term by any unit vector from the right yields the zero vector. Letting $\bz\in\mathbb{R}^d$ with $\norm{\bz}=1$, we have:
\begin{align*}
    \left\|\frac{(\bar{\bw}-\bar{\bv})(\bar{\bv}^\top-\bar{\bw}^\top)\bz\cos(\theta)}{\sin(\theta)}\right\| &= \frac{\norm{\bar{\bw}-\bar{\bv}}\cdot |\inner{\bar{\bv}-\bar{\bw},\bz}|\cos(\theta)}{\sin(\theta)} \\
    &\leq \frac{\norm{\bar{\bw}-\bar{\bv}}^{2}\norm{\bz}\cos(\theta)}{\sin(\theta)} = \frac{(2-2\cos(\theta))\cos(\theta)}{\sin(\theta)}\underset{\theta \rightarrow 0}{\rightarrow} 0~,
\end{align*}
where the inequality is from Cauchy-Schwarz. This is true for every unit vector $\bz$, hence this is the zero matrix. Combining the above shows that $\bar{\bn}_{\bw,\bv}\bar{\bv}^\top + \bar{\bn}_{\bv,\bw}\bar{\bw}^\top = \pmb{0}_{d\times d}$
\end{proof}

\begin{proof}[Proof of \lemref{lem:hessian is differentiable}]
    First, recall the gradient of \eqref{eq:objective normal dist} at $\bw_1^n$ which is defined and continuous as long as $\bw_i\neq\mathbf{0}$ for all $i\in[n]$, as computed in \citep{brutzkus2017globally,safran2017spurious}, where the coordinates with indices $(i-1)d+1$ to $i\cdot d$ are given by
        \begin{equation}\label{eq:G}
            \frac{1}{2}\bw_i + \sum_{j\neq i} g\p{\bw_i,\bw_j} - \sum_{j=1}^k g\p{\bw_i,\bv_j},
        \end{equation}
        where
        \begin{equation}\label{eq:g}
            g\p{\bw,\bv} = \frac{1}{2\pi} \p{\norm{\bv} \sin\p{\theta_{\bw,\bv}}\bar{\bw} + \p{\pi-\theta_{\bw,\bv}}\bv}.
        \end{equation}
    Clearly, by \thmref{thm:hessian of objective} the gradient is continuously differentiable for any $\bw_1^n$ where the angle between any two vectors $\theta_{\bw_i,\bw_j}\neq0,\pi$ for $i\neq j$. We will show that the partial derivatives of $h_1(\bw,\bv)$ and $h_2(\bw,\bv)$ are continuous for all $\bw,\bv\neq\mathbf{0}$, by showing that they coincide with the derivative of $g$ whenever $\theta_{\bw_i,\bw_j}$ tends to $0$ or $\pi$.
    
    We begin with computing the limits of $h_1(\bw,\bv)$ and $h_2(\bw,\bv)$ when $\theta_{\bw,\bv}\to0$ and $\theta_{\bw,\bv}\to\pi$. First, we have 
    \[
        \lim_{\bw\to\bu}h_1(\bw,\bu) = \pmb{0}_{d\times d}
    \]
    and
    \[
        \lim_{\bw\to-\bu}h_1(\bw,\bu) = \pmb{0}_{d\times d}.
    \]
    This holds since in both cases $\sin(\theta_{\bw,\bu})\to0$ and since that for any unit vector $\bx$, $\norm{\bx\bx^{\top}}$ is uniformly bounded. Next, we have from \lemref{lem:h2 parallel vectors} that
    \[
        \lim_{\bv\to\bu}h_2(\bu,\bv) = \frac{1}{2}I,
    \]
    and from a straightforward calculation that
    \[
        \lim_{\bv\to-\bu}h_2(\bu,\bv) = \pmb{0}_{d\times d}.
    \]
    Assume $\bw\to\bu$ and $\bv=t\bu$ for some $t>0$ and non-zero vector $\bu$ and let $\be_i$ denote the unit vector with all-zero coordinates except for the $i$-th coordinate, we compute the partial derivative of $g(\bw,\bv)$ with respect to coordinate $i$ of $\bw$:
    \begin{align}
        \frac{\partial g}{\partial w_i}(\bw,\bv) &= \lim_{\epsilon\to0}\frac{\norm{t\bu}\sin\p{\theta_{\bu+\epsilon \be_i,\bu}}\frac{\bu+\epsilon \be_i}{\norm{\bu+\epsilon \be_i}} + \p{\pi-\theta_{\bu+\epsilon \be_i,\bu}}t\bu - \p{\norm{t\bu}\sin\p{\theta_{\bu,\bu}}\frac{\bu}{\norm{\bu}} + \p{\pi-\theta_{\bu,\bu}}t\bu}}{2\pi\epsilon}\nonumber\\
        &=\lim_{\epsilon\to0}\frac{\norm{t\bu}\sin\p{\theta_{\bu+\epsilon \be_i,\bu}}\frac{\bu+\epsilon \be_i}{\norm{\bu+\epsilon \be_i}} + \p{\pi-\theta_{\bu+\epsilon \be_i,\bu}}t\bu - \pi t\bu}{2\pi\epsilon}\label{eq:angle_zero}\\
        &=\lim_{\epsilon\to0} \frac{t}{2\pi}\frac{\norm{\bu}\sin\p{\theta_{\bu+\epsilon \be_i,\bu}}\frac{\bu+\epsilon \be_i}{\norm{\bu+\epsilon \be_i}}-\theta_{\bu+\epsilon \be_i,\bu}\bu}{\epsilon}\nonumber\\
        &=\lim_{\epsilon\to0} \frac{t}{2\pi}\frac{\frac{\norm{\bu}}{\norm{\bu+\epsilon \be_i}}\sin\p{\theta_{\bu+\epsilon \be_i,\bu}}\bu-\theta_{\bu+\epsilon \be_i,\bu}\bu}{\epsilon},\label{eq:lim_zero}
    \end{align}
    where equality (\ref{eq:angle_zero}) is due to $\sin\p{\theta_{\bu,\bu}}=0$ and equality (\ref{eq:lim_zero}) is due to $\frac{\sin\p{\theta_{\bu+\epsilon \be_i,\bu}}\epsilon \be_i}{\epsilon}\to0$. Assume w.l.o.g.\ that $\epsilon\to0^+$ (the following arguments are reversed in order if $\epsilon\to0^-$), we have by using the inequality $\sin(x)\le x$ which holds for all $x\ge0$ that \eqref{eq:lim_zero} is upper bounded by
    \begin{equation}\label{eq:squeeze_upper}
        \lim_{\epsilon\to0} \frac{t}{2\pi}\frac{\theta_{\bu+\epsilon \be_i,\bu}}{\epsilon}\p{\frac{\norm{\bu}}{\norm{\bu+\epsilon \be_i}}-1}\bu.
    \end{equation}
    Next, we have by the law of sines that
    \[
        \frac{\epsilon}{\sin(\theta_{\bu+\epsilon \be_i,\bu})} = \frac{\norm{\bu}}{\sin(\theta_{\bu+\epsilon \be_i,\epsilon \be_i})} \ge \norm{\bu},
    \]
    which entails
    \[
        \theta_{\bu+\epsilon \be_i,\bu} \le \arccos\p{\sqrt{1-\frac{\epsilon^2}{\norm{\bu}^2}}},
    \]
    therefore by L'H\^opital's rule
    \begin{equation}\label{eq:theta_bound}
        \lim_{\epsilon\to0} \abs{\frac{\theta_{\bu+\epsilon \be_i,\bu}}{\epsilon}} \le \frac{1}{\norm{\bu}},
    \end{equation}
    and since $\frac{\norm{\bu}}{\norm{\bu+\epsilon \be_i}}\to1$, this implies that \eqref{eq:squeeze_upper} converges to $\mathbf{0}$. Using the inequality $\sin(x)\ge x-\frac{x^3}{6}$ which holds for all $x\ge0$ we lower bound \eqref{eq:lim_zero} by
    \begin{equation}\label{eq:squeeze_lower}
        \lim_{\epsilon\to0} \frac{t}{2\pi}\frac{\theta_{\bu+\epsilon \be_i,\bu}}{\epsilon}\p{1-\frac{\theta^2_{\bu+\epsilon \be_i,\bu}}{6}}\p{\frac{\norm{\bu}}{\norm{\bu+\epsilon \be_i}}-1}\bu.
    \end{equation}
    We have $1-\theta^2_{\bu+\epsilon \be_i,\bu}/6\to1$ and $\frac{\norm{\bu}}{\norm{\bu+\epsilon \be_i}}\to1$, and from \eqref{eq:theta_bound} we have that the above converges to $\mathbf{0}$. Combining \eqref{eq:squeeze_upper} and \eqref{eq:squeeze_lower} and using the squeeze theorem, we have that $\frac{\partial g}{\partial w_i}(\bw,\bv) =\mathbf{0}$, from which it follows that the Hessian at $(\bw,\bv)$ is the zero matrix $\pmb{0}_{d\times d}$.
    
    Now, assume $\bw\to\bu$ and $\bv=-t\bu$ for some $t>0$, and compute
    \begin{align}
        \frac{\partial g}{\partial w_i}(\bw,\bv) &= \lim_{\epsilon\to0}\frac{\norm{t\bu}\sin\p{\theta_{\bu+\epsilon \be_i,-\bu}}\frac{\bu+\epsilon \be_i}{\norm{\bu+\epsilon \be_i}} - \p{\pi-\theta_{\bu+\epsilon \be_i,-\bu}}t\bu - \p{\norm{t\bu}\sin\p{\theta_{\bu,-\bu}}\frac{\bu}{\norm{\bu}} - \p{\pi-\theta_{\bu,-\bu}}t\bu}}{2\pi\epsilon}\nonumber\\
        &=\lim_{\epsilon\to0}\frac{\norm{t\bu}\sin\p{\theta_{\bu+\epsilon \be_i,-\bu}}\frac{\bu+\epsilon \be_i}{\norm{\bu+\epsilon \be_i}} - \p{\pi-\theta_{\bu+\epsilon \be_i,-\bu}}t\bu}{2\pi\epsilon}\label{eq:angle_pi}\\
        &=\lim_{\epsilon\to0} \frac{t}{2\pi}\frac{\norm{\bu}\sin\p{\theta_{\bu+\epsilon \be_i,\bu}}\frac{\bu+\epsilon \be_i}{\norm{\bu+\epsilon \be_i}}-\theta_{\bu+\epsilon \be_i,\bu}\bu}{\epsilon}=\mathbf{0},\label{eq:lim_zero2}
    \end{align}
    where equality (\ref{eq:angle_pi}) is due to $\theta_{\bu,-\bu}=\pi$ and equality (\ref{eq:lim_zero2}) is due to $\theta_{\bu+\epsilon \be_i,-\bu}=\pi-\theta_{\bu+\epsilon \be_i,\bu}$, and since we get the same limit as we did in the previous case. This implies $\lim_{\bw\to-\bu}h_1(\bw,\bu)=\pmb{0}_{d\times d}$, and concludes the derivation for $h_1$.
    
    Moving on to $h_2$, assume $\bv\to\bu$ and $\bw=t\bu$ for some $t>0$, and compute
    \begin{align*}
        \frac{\partial g}{\partial v_i}(\bw,\bv) &= \lim_{\epsilon\to0}\frac{\norm{\bu+\epsilon \be_i}\sin\p{\theta_{\bu+\epsilon \be_i,\bu}}\frac{t\bu}{\norm{t\bu}} + \p{\pi-\theta_{\bu+\epsilon \be_i,\bu}}(\bu+\epsilon \be_i) - \p{\norm{\bu}\sin\p{\theta_{\bu,\bu}}\frac{t\bu}{\norm{t\bu}} + \p{\pi-\theta_{\bu,\bu}}\bu}}{2\pi\epsilon}\\
        &=\lim_{\epsilon\to0}\frac{\norm{\bu+\epsilon \be_i}\sin\p{\theta_{\bu+\epsilon \be_i,\bu}}\frac{\bu}{\norm{\bu}} + \p{\pi-\theta_{\bu+\epsilon \be_i,\bu}}(\bu+\epsilon \be_i) - \pi\bu}{2\pi\epsilon}\\
        &=\lim_{\epsilon\to0} \frac{1}{2\pi}\frac{\frac{\norm{\bu+\epsilon \be_i}}{\norm{\bu}}\sin\p{\theta_{\bu+\epsilon \be_i,\bu}}\bu -\theta_{\bu+\epsilon \be_i,\bu}\bu}{\epsilon} + \frac{\be_i}{2} - \lim_{\epsilon\to0} \frac{\theta_{\bu+\epsilon \be_i,\bu}\be_i}{2\pi}\\
        &=\lim_{\epsilon\to0} \frac{1}{2\pi}\frac{\frac{\norm{\bu+\epsilon \be_i}}{\norm{\bu}}\sin\p{\theta_{\bu+\epsilon \be_i,\bu}}\bu -\theta_{\bu+\epsilon \be_i,\bu}\bu}{\epsilon} + \frac{\be_i}{2},
    \end{align*}
    and following the same reasoning as in the proof for $h_1$ we have that the above limit is $\mathbf{0}$, which implies that
    \[
        \frac{\partial g}{\partial v_i}(\bu,\bu) = \frac{1}{2}I.
    \]
    Now assume $\bv\to\bu$ and $\bw=-t\bu$ for some $t>0$, and compute
    \begin{align*}
        &\frac{\partial g}{\partial v_i}(\bw,\bv)\\ =& \lim_{\epsilon\to0}\frac{\norm{\bu+\epsilon \be_i}\sin\p{\theta_{\bu+\epsilon \be_i,-\bu}}\frac{-t\bu}{\norm{-t\bu}} + \p{\pi-\theta_{\bu+\epsilon \be_i,-\bu}}(\bu+\epsilon \be_i) - \p{\norm{\bu}\sin\p{\theta_{\bu,-\bu}}\frac{-t\bu}{\norm{-t\bu}} + \p{\pi-\theta_{\bu,-\bu}}\bu}}{2\pi\epsilon}\\
        =&\lim_{\epsilon\to0}\frac{-\norm{\bu+\epsilon \be_i}\sin\p{\theta_{\bu+\epsilon \be_i,-\bu}}\frac{\bu}{\norm{\bu}} + \p{\pi-\theta_{\bu+\epsilon \be_i,-\bu}}(\bu+\epsilon \be_i)}{2\pi\epsilon}\\
        =&\lim_{\epsilon\to0} \frac{1}{2\pi}\frac{-\frac{\norm{\bu+\epsilon \be_i}}{\norm{\bu}}\sin\p{\theta_{\bu+\epsilon \be_i,\bu}}\bu + \theta_{\bu+\epsilon \be_i,\bu}\bu}{\epsilon} - \lim_{\epsilon\to0} \frac{\theta_{\bu+\epsilon \be_i,\bu}\be_i}{2\pi}\\
        =&-\frac{1}{2\pi}\lim_{\epsilon\to0} \frac{\frac{\norm{\bu+\epsilon \be_i}}{\norm{\bu}}\sin\p{\theta_{\bu+\epsilon \be_i,\bu}}\bu - \theta_{\bu+\epsilon \be_i,\bu}\bu}{\epsilon}.
    \end{align*}
    From the previous case we have that the above limit is $\mathbf{0}$, implying that
    \[
        \frac{\partial g}{\partial v_i}(\bu,-\bu) = \pmb{0}_{d\times d},
    \]
    and concluding the proof of the lemma.
\end{proof}

\section{Proofs from \secref{sec:global minima}}\label{apen:global minima proofs}
\subsection{Proof of \thmref{thm:global_minima_form}}
To prove the theorem we will need the following lemma, which essentially asserts that misclassifying a single instance will result in a strictly positive loss in expectation.

\begin{lemma}\label{lem:point_disagreement}
    Let $f,g:\reals^d\to\reals$ be continuous functions, and suppose exists $\bx_0\in\reals^d$ s.t.\ $f(\bx_0)\neq g(\bx_0)$. Then
    \[
        \mathbb{E}_{\bx\sim\mathcal{N}(\mathbf{0},I)}\left[\frac{1}{2}(f(\bx)-g(\bx))^2\right] > 0.
    \]
\end{lemma}

\begin{proof}
Assume w.l.o.g.\ $f(\bx_0) - g(\bx_0) = c >0$. Since $f$ and $g$ are continuous, there exists an open neighborhood $U\ni \bx_0$ s.t.
\begin{equation}\label{eq:strict_diff}
\left|f(\bz)-g(\bz)\right| > c,~~\forall \bz\in U.
\end{equation}
Let $A$ denote the event where a point $\bz$ sampled from a multivariate normal random variable belongs to $U$, then by the law of total expectation
\begin{align*}
    \mathbb{E}\left[\frac{1}{2}(f(\bx)-g(\bx))^2\right] &= \mathbb{E}\left[\frac{1}{2}(f(\bx)-g(\bx))^2|A\right]\Pr\left[A\right] + \mathbb{E}\left[\frac{1}{2}(f(\bx)-g(\bx))^2|\bar{A}\right]\Pr\left[\bar{A}\right]\\ 
    &\ge \mathbb{E}\left[\frac{1}{2}(f(\bx)-g(\bx))^2|A\right]\Pr\left[A\right] > 0,
\end{align*}
where the strict inequality is due to $\Pr\left[A\right]>0$ since $U$ is open and a multivariate normal random variable has a measure which is strictly positive on all of $\reals^d$, and due to $\mathbb{E}\left[\frac{1}{2}(f(\bx)-g(\bx))^2|A\right]>0$ by virtue of \eqref{eq:strict_diff} holding whenever $A$ occurs.
\end{proof}

\begin{proof}[Proof of \thmref{thm:global_minima_form}]
First, assume w.l.o.g.\ $v_j=e_j$ for all $j\in[k]$. This is justified since an orthonormal change of bases does not change the geometry of our objective. By virtue of \lemref{lem:point_disagreement} and the continuity of ReLU networks, it suffices to find a single point $\bx$ s.t.\ any network with a different structure than in the theorem statement disagrees on $\bx$ with $f(\bx)=\sum_{i=1}^k \relu{x_i}$. To this end, we shall divide the proof into several different cases, based on the set of weights $\bw_1,\ldots,\bw_n$ of the approximating network $N$.
\begin{itemize}
    \item
    If $w_{i,j}<0$ for some $i,j$, then w.l.o.g.\ $i=j=1$ and
    \[
        f(-e_1) = 0 < \relu{w_{1,1}\cdot-1} \le N(-e_1).
    \]
    \item
    Otherwise, for $\bx=e_1$ we have
    \[
        f(e_1) = 1 = N(e_1) = \sum_{i=1}^n\relu{w_{i,1}},
    \]
    and thus
    \[
        \sum_{i=1}^n w_{i,1}=1.
    \]
    \item
    Suppose that exist two coordinates in the same neuron that are not $0$, w.l.o.g\ $w_{1,1},w_{1,2}>0$. Then for $\bx=(1,-\frac{w_{i,1}}{w_{i,2}},0,\ldots,0)$, we have
    \[
        f(\bx)=1=\sum_{i=1}^n w_{i,1}=w_{i,1}+\sum_{i=2}^n\relu{w_{i,1}} > \sum_{i=2}^n\relu{w_{i,1}} \geq
        \sum_{i=1}^n\relu{w_{i,1}-\frac{w_{i,1}}{w_{i,2}} \cdot w_{i,2}} = N(\bx).
    \]
\end{itemize}
Overall, if $W^*$ does not have the structure as in the theorem statement then this results in a misclassified point which due to \lemref{lem:point_disagreement} implies the result.
\end{proof}

\subsection{Proof of \thmref{thm:n=k strongly convex}}
First we calculate the Hessian of the objective at a global minimum. Since we assume that the vectors $\bv_1,\dots,\bv_k$ are orthogonal the Hessian has a simple form:

\begin{lemma}\label{lem:n=k hessian at global minimum}
Assume that $n=k$ and let $\bw_1^n=(\bw_1,\dots,\bw_n)$ be a global minima. Then the Hessian $H(\bw_1^n)$ of the objective \eqref{eq:objective normal dist} has the following block form:
\begin{itemize}
    \item For $i\in[n]$:
        \[
        H(\bw_1^n)_{i,i} = \frac{1}{2}I
        \]
    \item For $i,j\in[n]$ with $i\neq j$:
        \[
        H_{i,j}(\bw_1^n) = \frac{1}{4}I + \frac{1}{2\pi}\left(\bar{\bw}_i\bar{\bw}_j^\top + \bar{\bw}_j\bar{\bw}_i^\top\right)
        \]
\end{itemize}
where we look at the Hessian as a $k\times k$ block matrix, each block of size $d\times d$.
\end{lemma}
\begin{proof}
Assume w.l.o.g that at this global minimum $\bw_i = \bv_i$ for every $i$. For the first item let $i\in[n]$, we have that:
\begin{align}
    H(\bw_1^n)_{i,i} = \frac{1}{2}I + \sum_{j\neq i}h_1(\bw_i,\bw_j) - \sum_{l=1}^k h_1(\bw_i,\bv_l) = \frac{1}{2}I - h_1(\bw_i,\bv_i)~.
\end{align}
We will show that if $\bw,\bv$ are parallel then $h_1(\bw,\bv) = \pmb{0}_{d\times d}$. Let $\bw,\bv$ be two parallel non-zero vectors and $\bu$ be some vector not parallel to them. We have by definition of $h_1(\bu,\bv)$ that:
\begin{align}\label{eq:proof that h1 parallel is zero}
    \lim_{\bu\rightarrow\bw}h_1(\bu,\bv) =& \lim_{\bu\rightarrow\bw} \frac{\sin(\theta_{\bu,\bv})\norm{\bv}}{2\pi\norm{\bu}}\left(I - \bar{\bu}\bar{\bu}^\top + \bar{\bn}_{\bv,\bu}\bar{\bn}_{\bv,\bu}^\top\right) \nonumber\\
    =& \lim_{\bu\rightarrow\bw} \frac{\norm{\bv}}{2\pi\norm{\bu}}\sin(\theta_{\bu,\bv})\bar{\bn}_{\bv,\bu}\bar{\bn}_{\bv,\bu}^\top~.
\end{align}
We will show that the second term above is the zero matrix. Note that:
\begin{align*}
\norm{\bn_{\bw,\bv}} &= \sqrt{\inner{\bar{\bv} - \cos(\theta_{\bw,\bv})\bar{\bw}, \bar{\bv} - \cos(\theta_{\bw,\bv})\bar{\bw}}} \\
& = \sqrt{1 - \cos^2(\theta_{\bw,\bv})} = \sin(\theta_{\bw,\bv})~,
\end{align*}
hence we have that $\sin(\theta_{\bu,\bv})\bar{\bn}_{\bv,\bu}\bar{\bn}_{\bv,\bu}^\top = \bn_{\bv,\bu}\bar{\bn}_{\bv,\bu}^\top$. Letting $\bx$ be some vectors with norm $1$, we have:
\begin{align*}
    \lim_{\bu\rightarrow\bw}\norm{\bn_{\bv,\bu}\bar{\bn}_{\bv,\bu}^\top\bx} \leq \lim_{\bu\rightarrow\bw}\norm{\bn_{\bv,\bu}}\cdot |\inner{\bar{\bn}_{\bv,\bu},\bx}|\leq \lim_{\bu\rightarrow\bw}\norm{\bn_{\bv,\bu}}  = 0~.
\end{align*}
This is true for every vector $\bx$, hence $\lim_{\bu\rightarrow\bw}\bn_{\bv,\bu}\bar{\bn}_{\bv,\bu} = \pmb{0}_{d\times d}$. Combining this with \eqref{eq:proof that h1 parallel is zero} and that $\bw\neq \pmb{0}$ we have that $\lim_{\bu\rightarrow\bw}h_1(\bu,\bv) = \pmb{0}_{d\times d}$. This proves the first item  of the lemma.

For the second item, recall that by our assumption the target vectors are orthogonal. Hence we have for $i\neq j$:
\begin{align*}
    h_2(\bw_i\bw_j) &= \frac{1}{2\pi}\left((\pi-\theta_{\bw_i,\bw_j})I + \bar{\bn}_{\bw_i,\bw_j} \bar{\bw}_j^\top + \bar{\bn}_{\bw_j,\bw_i}\bar{\bw}_i^\top\right) \\
    &= \frac{1}{2\pi}\left(\left(\pi-\frac{\pi}{2}\right)I + \frac{\bar{\bw}_i\bar{\bw}_j^\top - \cos(\theta_{\bw_i,\bw_j})\bar{\bw}_j\bar{\bw}_j^\top}{\sin(\theta_{\bw_i,\bw_j})} + \frac{\bar{\bw}_j\bar{\bw}_i^\top - \cos(\theta_{\bw_i,\bw_j})\bar{\bw}_i\bar{\bw}_i^\top}{\sin(\theta_{\bw_i,\bw_j})}\right) \\
    &= \frac{1}{4}I + \frac{1}{2\pi}\left(\bar{\bw}_i\bar{\bw}_j^\top + \bar{\bw}_j\bar{\bw}_i^\top\right)~.
\end{align*}
\end{proof}

We are now ready to prove the theorem:

\begin{proof}[Proof of \thmref{thm:n=k strongly convex}]
Let $\bw_1^n=(\bw_1,\dots,\bw_k)$ be some global minimum, by \lemref{lem:n=k hessian at global minimum} the Hessian at $\bw$ is equal to the sum of the following matrices:
\begin{equation}\label{eq:n=k hessian decomposition}
    H(\bw) = \begin{pmatrix}
    \frac{1}{2}I_d \ \dots \ \frac{1}{4}I_d \\
    \vdots \ \ddots \ \vdots \\
    \frac{1}{4}I_d\  \dots \ \frac{1}{2}I_d
    \end{pmatrix} + 
    \frac{1}{2\pi}\begin{pmatrix}
    0_d  &E_{1,2}  &\dots  &E_{1,n} \\
    E_{2,1} &\ddots &\   &\vdots \\
    \vdots  &\ &\ddots &E_{1,n-1} \\
    E_{n,1} &\dots &E_{n-1,n} &0_d
    \end{pmatrix},
\end{equation}
where $E_{i,j} = \bar{\bw}_i\bar{\bw}_j^\top + \bar{\bw}_j\bar{\bw}_i^\top$ is a $d\times d $ matrix. Recall that the Hessian can be viewed as a $k\times k$ block matrix with blocks of size $d\times d$. We will calculate the smallest eigenvalue of the two matrices in \eqref{eq:n=k hessian decomposition}, thus bounding the smallest eigenvalue of $H(\bw)$.

For the first matrix, the vectors $\begin{pmatrix} \be_i \\ \vdots \\ \be_i\end{pmatrix}$ is an eigenvector for every $i\in[d]$ with eigenvalue $\frac{k+1}{4}$. Also, the vectors $\begin{pmatrix}\pmb{0}\\ \vdots\\ \pmb{0} \\ \be_i \\ -\be_i \\ \pmb{0}\\ \vdots \\ \pmb{0} \end{pmatrix}$ where the $\be_i$ can be at any two consecutive coordinates, are eigenvectors with eigenvalue $\frac{1}{4}$. There are $d$ eigenvectors of the first kind, and $(k-1)\cdot d$ of the second kind. All of these vectors are linearly independent, thus we found $k\cdot d$ independent eigenvectors. This proves that the smallest eigenvalue of the first matrix is $\frac{1}{4}$.

For the second matrix we define a block vector $\tilde{\pmb{\alpha}}$ of size $k\cdot d$ as a vectors with $k$ coordinates, each coordinate is a vector of size $d$. Let $i,j\in[k]$ with $i\neq j$ and define the following block vectors:
\begin{itemize}
    \item $\tilde{\pmb{\alpha}}_i = \bar{\bw}_i$, $\tilde{\pmb{\alpha}}_j = \bar{\bw}_j$, and the rest of the coordinates of $\tilde{\pmb{\alpha}}$ are the zero vector.
    \item $\tilde{\pmb{\beta}}_i = \bar{\bw}_j$, $\tilde{\pmb{\beta}}_j = -\bar{\bw}_i$, and the rest of the coordinates of $\tilde{\pmb{\beta}}$ are the zero vector.
    \item $\tilde{\pmb{\gamma}}_i = \bar{\bw}_i$, $\tilde{\pmb{\gamma}}_j = \bar{\bw}_j$, and the rest of the coordinates of $\tilde{\pmb{\gamma}}$ are the zero vector.
    \item $\tilde{\pmb{\delta}}_i = \bar{\bw}_i$ for every $i\in[k]$
\end{itemize}
Note that $E_{i,j}\bar{\bw}_i = \bar{\bw}_i\bar{\bw}_j^\top\bar{\bw}_i + \bar{\bw}_j\bar{\bw}_i^\top\bar{\bw}_i = \bar{\bw}_j$, in the same manner $E_{i,j}\bar{\bw}_j = \bar{\bw}_i$ and for $l\in [k]$ with $l\neq i,~l\neq j$ we have $E_{i,j}\bar{\bw}_l = \pmb{0}$. Denoting the second matrix in \eqref{eq:n=k hessian decomposition} as A, we have that $A\tilde{\pmb{\alpha}} = -\tilde{\alpha},~ A\tilde{\pmb{\beta}}= -\tilde{\pmb{\beta}},~ A\tilde{\pmb{\gamma}} = \tilde{\pmb{\gamma}},~ A\tilde{\pmb{\delta}} = k\tilde{\pmb{\delta}}$. Hence the vectors $\tilde{\pmb{\alpha}}, \tilde{\pmb{\beta}}$ are eigenvectors for every $i\neq j$ with eigenvalue $-1$, the vectors $\tilde{\pmb{\gamma}}$ are eigenvectors for every $i\neq j$ with eigenvalue $1$, and $\tilde{\pmb{\delta}}$ is an eigenvector with eigenvalue $k$. If $d=k$ then these eigenvectors span the entire space, hence the smaller eigenvalue is $-1$. If $d>k$ we complete $\bw_{k+1},\dots,\bw_d$ to an orthogonal basis of the entire space, and add the eigenvectors which corresponds to the eigenvalue $0$. In both cases the smallest eigenvalue of A is $-1$.

Combining the above with \eqref{eq:n=k hessian decomposition} and letting $\bv$ be any vector with norm $1$, we have:
\begin{equation*}
    \bv^\top H(\bw) \bv = \bv^\top \begin{pmatrix}
    \frac{1}{2}I_d \ \dots \ \frac{1}{4}I_d \\
    \vdots \ \ddots \ \vdots \\
    \frac{1}{4}I_d\  \dots \ \frac{1}{2}I_d
    \end{pmatrix}\bv + \frac{1}{2\pi}\bv^\top\begin{pmatrix}
    &0_d  &E_{1,2}  &\dots  &E_{1,n} \\
    &E_{2,1} &\ddots &\   &\vdots \\
    &\vdots  &\ &\ddots &E_{n-1,n} \\
    &E_{n,1} &\dots &E_{n,n-1} &0_d
    \end{pmatrix}\bv \geq \frac{1}{4} - \frac{1}{2\pi} ~.
\end{equation*}
This proves that the Hessian is positive definite with minimal eigenvalue strictly larger than $\frac{1}{4} - \frac{1}{2\pi}$. Since the objective is twice differentiable, and the eigenvalue of a matrix is a continuous function, we have that the Hessian is positive definite in an open neighborhood of the global minimum. In particular, for any $0<\lambda < \frac{1}{4} - \frac{1}{2\pi}$ there is an open neighborhood of the global minimum for which the objective is $\lambda$-strongly convex.
\end{proof}

\subsection{Proofs from \subsecref{sec:over-parameterization} and \subsecref{subsec:OPSC}}\label{apen:proofs from OPSC}

\begin{figure}[H]
 	\centering
     {\includegraphics[trim={10cm 2cm 7.62cm 2cm},clip,scale=0.22]{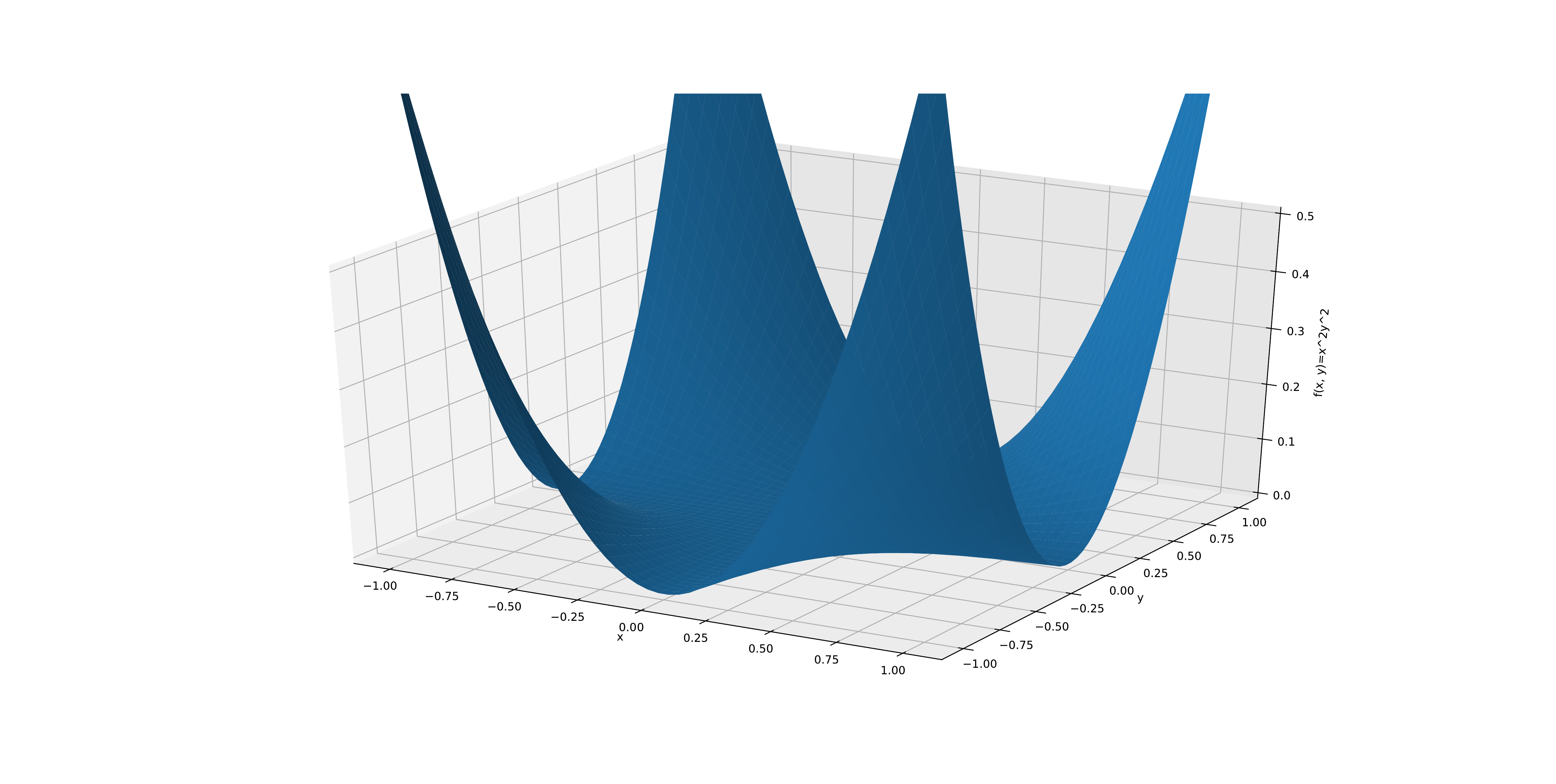}}
     \hskip 1cm
     \includegraphics[width = 2in, height = 2in]{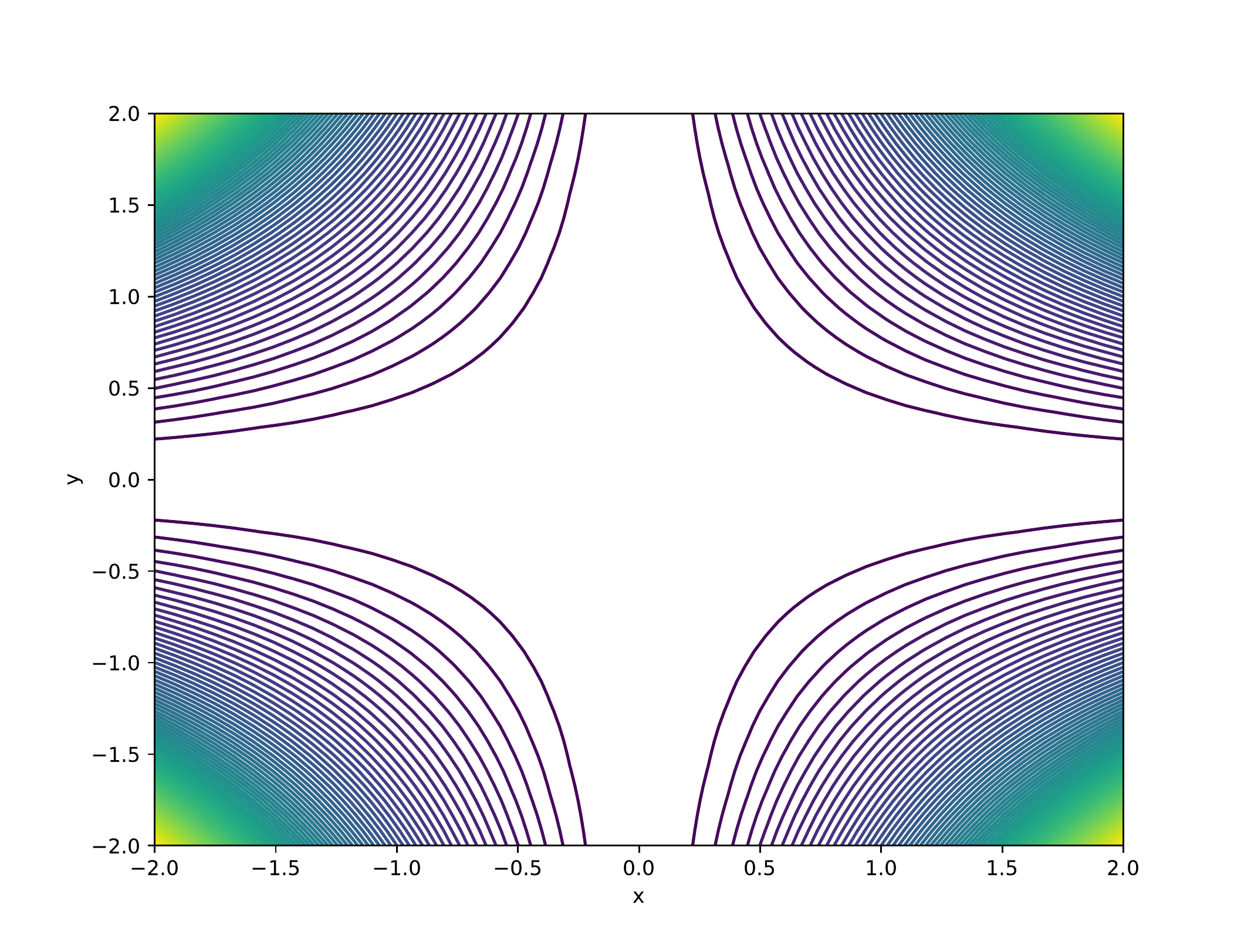}
     \caption{Plot and contour of the function $f(x,y)=x^2y^2$. This function is continuously differentiable and has a manifold of global minima $\{(x,y):x=0~\text{or}~y=0\}$.
     $f(x,y)$ is not locally convex almost anywhere (as the determinant of its Hessian is negative almost everywhere). It is also not one-point strongly convex w.r.t.\ any global minimum.}
     \label{fig:x2y2}
\end{figure}

See \figref{fig:x2y2} for a plot of the function $f(x,y)=x^2y^2$. Note that although this function is neither locally convex nor OPSC, it is OPSC in an $\epsilon$-orthogonal neighborhood of every global minima. Indeed, take some point $\tilde{\bw}=(0,y)$ with any $y\neq 0$, then $\tilde{\bw}$ is a global minimum of $f(x,y)$ for which $U_\epsilon^\perp(\tilde{\bw})$ is not empty. It can be easily seen that for any $\epsilon>0$ the function $f(x,y)$ is OPSC in $U_\epsilon^\perp(\tilde{\bw})$ with respect to $\tilde{\bw}$, with a strong convexity parameter of $\lambda=2$. This is actually true for any global minimum $(x,y)\neq (0,0)$ of $f(x,y)$. Thus, although this function is not convex and also not OPSC, it is OPSC in an $\epsilon$-orthogonal neighborhood of any global minimum except $(0,0)$. This function also does not satisfy the PL condition. Indeed, its global minimal value is $f^*=0$, and we have 
\[
    \norm{\nabla f(x,y)}^2 = 4\left(f(x,y)-f^*\right)\cdot \left\|\begin{pmatrix} x \\ y \end{pmatrix}\right\|^2~,
\]
where it is easy to see that there is no global constant $\lambda >0$ that satisfies the PL condition.

See \figref{fig:OPSC-orthogonal} for an intuition on an $\epsilon$-orthogonal neighborhood.

%First, it is easy to see that the function $f(x,y)=x^2y^2$ depicted in \figref{fig:x2y2} does not satisfy the PL condition. Indeed its global minimal value is $f^*=0$, and we have 
%
%

%Note that if $n>k$ then this neighborhood is not empty for any differentiable global minimum. Returning to the example of $f(x,y)=x^2y^2$ (\figref{fig:x2y2}), 

\begin{figure}[H]
\centering

  \begin{minipage}[c]{0.5\textwidth}
    \includegraphics[width=\textwidth]{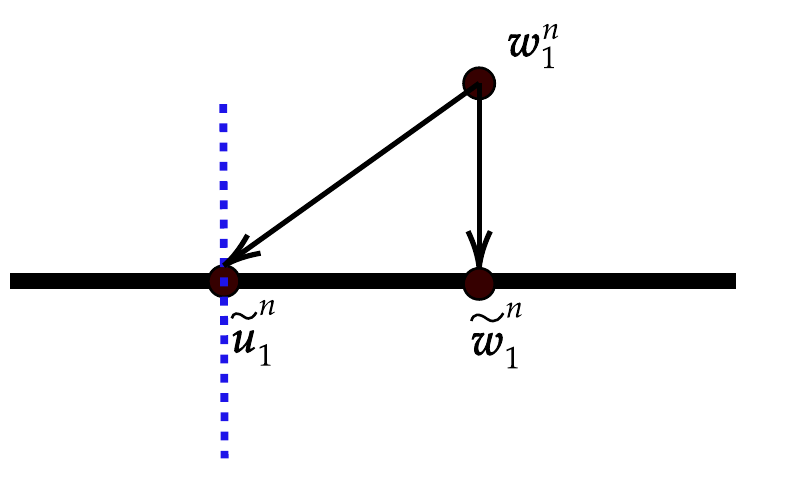}
  \end{minipage}\hfill
  \begin{minipage}[c]{0.5\textwidth}
    \caption{
 Suppose that the bold black line represents the manifold of global minima,  $\tilde{\bw}_1^n$ and $\tilde{\bu}_1^n$ are two different global minima and $\bw_1^n$ is some point. The blue dashed line represents the $\epsilon$-orthogonal neighborhood of $\tilde{\bu}_1^n$. In order to verify that the objective is OPSC we would have to decide with respect to which point it is OPSC. For example, from $\bw_1^n$ it would make more sense to check its OPSC parameter with respect to $\tilde{\bw}_1^n$ since it is its closest global minimum,
 rather than with respect to $\tilde{\bu}_1^n$. 
 }
  \end{minipage}

% {\includegraphics[width=2.5in]{orthogonal.pdf}}
% \caption{Histograms of the distributions of the sum of Euclidean norms of the neurons in the points converged to in the experiment, for $k=20,50,100$. The sum of norms seems concentrated around two centers, which correspond to the two types of points found (the first which corresponds to the large clusters are points having at most six distinct values -- e.g.\ \citet[Example 1]{safran2017spurious}, and the second which corresponds to the small clusters are points having eleven distinct values).}

\label{fig:OPSC-orthogonal}
\end{figure}

\subsubsection{Proof of \thmref{thm:over-parameterization not convex}}
\begin{proof}
Suppose $\bw_1^n=(\bw_1,\dots,\bw_n)$ is a global minimum. Assume w.l.o.g that $\bw_1 = \alpha_1  \bv_1,~ \bw_2 = \alpha_2 \bv_1$ for some $\alpha_1,\alpha_2 > 0$, that is there are at least two neurons that correspond to $\bv_1$. Let $\epsilon > 0$, take $\bu$ to be some unit vector orthogonal to $\bv_1$ and define $\tilde{\bw}_1 = \alpha_1\epsilon \bu + \alpha_1 \bv_1,\ \tilde{\bw}_2 = \alpha_2\epsilon \bu + \alpha_2 \bv_1, \ \tilde{\bw}_3 = \bw_3,\dots \tilde{\bw}_n = \bw_n$. Note that $\tilde{\bw}_1$ and $\tilde{\bw}_2$ are in the same direction. 

We will calculate the Hessian at $\tilde{\bw}_1^n$. Recall that we view the Hessian as composed of $n\times n$ blocks, where each block is of size $d\times d$. By \lemref{lem:h2 parallel vectors} we have that the blocks w.r.t. neurons $\bw_1,\bw_2$ are $H(\tilde{\bw}_1^n)_{12} = H(\tilde{\bw}_1^n)_{21} = \frac{1}{2}I$. For the diagonal components of the Hessian, note that if $\bz_1$ and $\bz_2$ are parallel then for every non-zero vector $\bu$ we have that $h_1(\bu,\bz_1) + h_1(\bu,\bz_2) = (\norm{\bz_1}+\norm{\bz_2})h_1(\bu,\bar{\bz}_1)$. By \thmref{thm:hessian of objective} we have:
\begin{align}
    H(\tilde{\bw}_1^n)_{11} &= \frac{1}{2}I + \sum_{j\neq 1}h_1(\tilde{\bw}_1,\tilde{\bw}_j) - \sum_{l=1}^k h_1(\tilde{\bw}_1,\bv_l) \nonumber \\
    &=\frac{1}{2}I - (\norm{\tilde{\bw}_1} + \norm{\tilde{\bw}_2})h_1(\tilde{\bw}_1,\bv_1)\nonumber\\
    &= \frac{1}{2}I - (1+\epsilon)(\alpha_1+\alpha_2)h_1(\tilde{\bw}_1,\bv_1)~,
\end{align}
and in the same manner $H(\tilde{\bw})_{22}=\frac{1}{2}I - (1+\epsilon)(\alpha_1+\alpha_2)h_1(\tilde{\bw}_2,\bv_1)$. Also since $\tilde{\bw}_1$ and $\tilde{\bw}_2$ are parallel we have that $h_1(\tilde{\bw}_2,\bv_1) = \frac{\|\tilde{\bw}_1\|}{\|\tilde{\bw}_2\|}h_1(\tilde{\bw}_1,\bv_1)$. The matrix $h_1(\tilde{\bw}_2,\bv_1)$ has an eigenvalue equal to $\frac{\sin(\theta_{\tilde{\bw}_1,\bv_1})\|\bv_1\|}{\pi\|\tilde{\bw}_1\|}$. Note that this eigenvalue is positive since we define the angle to be $\theta_{\tilde{\bw}_1,\bv_1}\in [0,\pi]$. Taking $\bz$ to be a unit eigenvector corresponding to this eigenvalue, we have:

\begin{align*} 
    &\begin{pmatrix}
    \bz^\top \ -\bz^\top \ 0 \ \dots \ 0
    \end{pmatrix}H(\tilde{\bw}_1^n) \begin{pmatrix}
    \bz \\ -\bz \\ 0 \\ \vdots \\ 0
    \end{pmatrix}  = \bz^\top H(\tilde{\bw}_1^n)_{11} \bz^\top + \bz^\top H(\tilde{\bw}_1^n)_{22} \bz - \bz^\top H(\tilde{\bw}_1^n)_{12} \bz - \bz^\top H(\tilde{\bw}_1^n)_{21} \bz  \\
    = & \frac{1}{2} - (1+\epsilon)(\alpha_1+\alpha_2)\frac{\sin(\theta_{\tilde{\bw}_1,\bv_1})\|\bv_1\|}{\pi\|\tilde{\bw}_1\|} + \frac{1}{2} - \frac{\alpha_1}{\alpha_2}\cdot (1+\epsilon)(\alpha_1+\alpha_2)\frac{\sin(\theta_{\tilde{\bw}_1,\bv_1})\|\bv_1\|}{\pi\|\tilde{\bw}_1\|} - \frac{1}{2} - \frac{1}{2} \\
    = & -(1+\epsilon)\frac{\sin(\theta_{\tilde{\bw}_1,\bv_1})}{\pi}\cdot \frac{(\alpha_1 + \alpha_2)^2}{\alpha_1\alpha_2} <0~.
\end{align*}

This is true for every $\epsilon >0$, hence in every neighborhood of the global minimum we found a point where the Hessian is not PSD, meaning that the loss is not locally convex.
\end{proof}

\subsubsection{Proof of \thmref{thm:over-parameterization not OPSC}}
\begin{proof}
Let $\epsilon >0$. The idea of the proof is to show that there is ${\bw}_1^n\in U_\epsilon^\perp (\tilde{\bw}_1^n)$ such that the Hessian of the objective, projected in the direction ${\bw}_1^n - \tilde{\bw}_1^n$ is of magnitude $O(\epsilon)$. This means that there is no $\lambda >0$ such that the objective is $\lambda$-OPSC in an $\epsilon$-orthogonal neighborhood of the global minimum.

From the assumption that $n>k$, and by \thmref{thm:global_minima_form} we know that there are at least two vectors $\tilde{\bw}_i,\tilde{\bw}_j$ which are parallel. In particular, assume w.l.o.g that $\tilde{\bw}_1=\alpha_1 \bv_1,~\tilde{\bw}_2=\alpha_2\bv_1$ where $\alpha_1,\alpha_2 >0$. We look at the following point:
\[
{\bw}_1 = \alpha_1 \bv_1 + \epsilon \bv_2 ,\ {\bw}_2 = \alpha_2 \bv_1 - \epsilon \bv_2,\ \bw_3 = {\tilde{\bw}}_3,\dots, \bw_n = \tilde{\bw}_n~. 
\]
Recall that the target vectors are orthogonal, hence ${\bw}_1^n\in U_\epsilon^\perp(\tilde{\bw}_1^n)$. Using \thmref{thm:hessian of objective} we can calculate the Hessian at the above point in the direction of the global minimum:
\begin{align}
    &\frac{1}{\norm{{\bw}_1^n-\tilde{\bw}_1^n}^2}({\bw}_1^n-\tilde{\bw}_1^n)^\top H({\bw}_1^n)({\bw}_1^n-\tilde{\bw}_1^n) = \frac{1}{2} 
    \begin{pmatrix}
    - \bv_2 \\ \bv_2 \\ 0 \\  \vdots \\ 0
    \end{pmatrix}^\top 
    H({\bw}_1^n) \begin{pmatrix}
    - \bv_2 \\ \bv_2 \\ 0 \\ \vdots \\ 0
    \end{pmatrix} \nonumber\\
    & = \frac{1}{2}\Bigg(1 + \bv_2^\top h_1(\bw_1,\bw_2)\bv_2 +  \bv_2^\top h_1(\bw_2,\bw_1)\bv_2 
     - \bv_2^\top h_1(\bw_1,\tilde{\bv}_1)\bv_2 - \bv_2^\top h_1(\bw_2,\tilde{\bv}_1)\bv_2 \nonumber\\
    & - 2\bv_2h_2(\bw_1,\bw_2)\bv_2 \Bigg)~.\label{eq:n > k not one point strongly convex main equation}
\end{align}
The largest eigenvalue of $h_1({\bw_1,\bw_2})$ (see Lemma 9 in \cite{safran2017spurious}) is: 
\begin{equation*}
    \frac{\sin(\theta_{\bw_1,\bw_2})\|\tilde{w}_2\|}{\pi\|\bw_1\|} = \frac{\epsilon(\alpha_1 + \alpha_2)}{\pi \|\bw_1\|^2} = \frac{\epsilon(\alpha_1 + \alpha_2)}{\pi\alpha_1^2 + \pi\epsilon^2} =  O(\epsilon)
\end{equation*}
Hence $\bv_2^\top h_1({\bw_1,\bw_2})\bv_2 = O(\epsilon)$, and for the same reasoning we get that 
\[
\bv_2^\top h_1({\bw_2,\bw_1})\bv_2 = O(\epsilon),~ \bv_2^\top h_1(\bw_1,\bv_1)\bv_2=O(\epsilon), ~\bv_2^\top h_1(\bw_2,\bv_1)\bv_2 = O(\epsilon)~.
\]

For the last term of the Hessian we will need the following:
\begin{align*}
    & \cos(\theta_{\bw_1,\bw_2}) = \frac{\inner{\alpha_1 \bv_1 + \epsilon \bv_2, \alpha_2 \bv_1 - \epsilon \bv_2}}{\sqrt{(\alpha_1^2 + \epsilon^2)(\alpha_2^2 + \epsilon^2)}} = \frac{\alpha_1 \alpha_2 - \epsilon^2}{\sqrt{(\alpha_1^2 + \epsilon^2)(\alpha_2^2 + \epsilon^2)}}  \\
    &\sin(\theta_{\bw_1,\bw_2}) = \sqrt{1 -\frac{(\alpha_1 \alpha_2 - \epsilon^2)^2}{(\alpha_1^2 + \epsilon^2)(\alpha_2^2 + \epsilon^2)}} = \frac{\epsilon(\alpha_1 + \alpha_2)}{\sqrt{(\alpha_1^2 + \epsilon^2)(\alpha_2^2 + \epsilon^2)}}  \\
    & \theta_{\bw_1,\bw_2} = \arccos(\cos(\theta_{\bw_1,\bw_2})) = O(\epsilon)
\end{align*}
Using the above we can calculate last term in \eqref{eq:n > k not one point strongly convex main equation}:
\begin{align*}
    &\bv_2^\top h_2(\bw_1,\bw_2)\bv_2 = \frac{1}{2\pi}\Bigg( (\pi - \theta_{\bw_1,\bw_2})\inner{\bv_2,\bv_2} + \frac{\inner{\bar{\bw}_2,\bv_2}(\inner{\bar{\bw}_1,\bv_2} - \cos(\theta_{\bw_1,\bw_2})\inner{\bar{\bw}_2,\bv_2})}{\sin(\theta_{\bw_1,\bw_2})} + \\
    & + \frac{\inner{\bar{\bw}_1,\bv_2}(\inner{\bar{\bw}_2,\bv_2} - \cos(\theta_{\bw_1,\bw_2})\inner{\bar{\bw}_1,\bv_2})}{\sin(\theta_{\bw_1,\bw_2})}  \Bigg) = \frac{1}{2} + O(\epsilon).
\end{align*}
 
Hence in total we have:
\begin{equation*}
\frac{1}{\norm{\bw_1^n-\tilde{\bw}_1^n}^2}({\bw}_1^n-\tilde{\bw}_1^n)^\top H(\bw_1^n)(\bw_1^n-\tilde{\bw}_1^n) = \frac{1}{2}\left(1 + O(\epsilon) -1 -O(\epsilon)\right) = O(\epsilon)~.
\end{equation*}
\end{proof}

\subsubsection{Proof of \thmref{thm:no PL}}
\begin{proof}
    The proof method is similar to that of the proof of \thmref{thm:over-parameterization not OPSC}. We use the same point $\bw_1^n$ as in \thmref{thm:over-parameterization not OPSC} which is in an $\epsilon$-orthogonal neighborhood of the relevant global minima. For ease of notation let $\theta:=\theta_{\bw_1,\bw_2}$ and $\gamma_1:=\theta_{\bw_1,\bv},\ \gamma_2=\theta_{\bw_2,\bv}$.
    
    We first calculate the objective of \eqref{eq:objective normal dist} using the closed form in \cite{safran2017spurious} Section 4.1.1. Set $\alpha = \alpha_1+\alpha_2$, and note that all the terms cancel out, except for those which include $\bw_1$ and $\bw_2$:
    \begin{align}\label{eq:loss for PL with f(w,v) terms}
        &F(\bw_1^n) = f(\bw_1,\bw_1) + f(\bw_2,\bw_2) + f(\bw_1,\bw_2) + \nonumber\\
        &+ f(\bw_2,\bw_1) - 2f(\bw_1,\alpha\bv_1)-2f(\bw_2,\alpha\bv_2)+ f(\alpha\bv_1,\alpha\bv_1) 
    \end{align}
    where 
    \[
    f(\bw,\bv) = \mathbb{E}_{\bx\sim\mathcal{N}(0,I)}[\left[\left[\bw^\top\bx\right]_+\left[\bv^\top\bx\right]_+\right] = \frac{1}{2\pi}\norm{\bw}\norm{\bv}\left(\sin(\theta_{\bw,\bv}) + (\pi - \theta_{\bw,\bv})\cos(\theta_{\bw,\bv})\right)~.
    \]
    To calculate this term we will need to the following expressions (calculated the same way as in \thmref{thm:over-parameterization not OPSC}):
    \begin{align*}
        &\cos(\theta) = \frac{\alpha_1\alpha_2 - \epsilon^2}{\sqrt{(\alpha_1^2 + \epsilon^2)(\alpha_2^2 + \epsilon^2)}},~ &\sin(\theta) = \frac{\epsilon(\alpha_1 + \alpha_2)}{\sqrt{(\alpha_1^2 + \epsilon^2)(\alpha_2^2 + \epsilon^2)}} \\
        &\cos(\gamma_1) = \frac{\alpha_1}{\sqrt{\alpha_1^2 + \epsilon^2}} ,~ &\sin(\gamma_1) =  \frac{\epsilon}{\sqrt{\alpha_1+\epsilon^2}} \\
        &\cos(\gamma_2) = \frac{\alpha_2}{\sqrt{\alpha_2^2 + \epsilon^2}} ,~ &\sin(\gamma_2) =  \frac{\epsilon}{\sqrt{\alpha_1+\epsilon^2}} \\
        & \norm{\bw_1}^2 = \alpha_1^2 + \epsilon^2,~ &\norm{\bw_2}^2 = \alpha_2^2 + \epsilon^2
    \end{align*}
    
    Also note, that using the taylor series of $\arccos$ in the same manner as the proof of \thmref{thm:over-parameterization not OPSC} we get that: $\theta = O(\epsilon),~\gamma_1 = O(\epsilon),~ \gamma_2 = O(\epsilon)$. The expression $f(\bw,\bv)$ depends only on the norms of $\bw$ and $\bv$ and on the angle between them, and also $f(\bw,\bw) = \frac{1}{2}\norm{\bw}^2$. Thus, returning to \eqref{eq:loss for PL with f(w,v) terms} we get:
    \begin{align}
        F(\bw_1^n) &=\frac{\norm{\bw_1}^2}{2} + \frac{\norm{\bw_2}^2}{2} + \frac{\norm{\bw_1}\norm{\bw_2}}{\pi}\left(\sin(\theta) + (\pi-\theta)\cos(\theta)\right) -\nonumber\\ &-\frac{\alpha\norm{\bw_1}}{\pi}\left(\sin(\gamma_1) + (\pi-\gamma_1)\cos(\gamma_1)\right) -\frac{\alpha\norm{\bw_2}}{\pi}\left(\sin(\gamma_2) + (\pi-\gamma_2)\cos(\gamma_2)\right)+ \frac{\alpha^2}{2} \nonumber\\
        & = \frac{1}{\pi}\left(\epsilon\alpha_1 + \epsilon\alpha_2 + \theta\epsilon^2 - \theta\alpha_1\alpha_2 - 2\alpha\epsilon + \alpha\alpha_1\gamma_1 + \alpha\alpha_2\gamma_2 \right) = \Omega(\epsilon)~.
    \end{align}

Next, we calculate the gradient of the objective using the closed form in in \cite{safran2017spurious} Section 4.1.1. 
\begin{align*}
    \left(\nabla F(\bw_1^n))\right)_1 = \frac{1}{2}\bw_1 + g(\bw_1,\bw_2) - g(\bw_1,\alpha\bv_1)
\end{align*}
where :
\[
g(\bw,\bv) = \frac{1}{2\pi}\left(\norm{\bv}\sin(\theta_{\bw,\bv}) + (\pi - \theta_{\bw,\bv})\bv \right)~.
\]
Hence the norm of the gradient of $\bw_1$ is:
\begin{align}\label{eq:PL proof bound on norm of gradient}
    \norm{ \left(\nabla F(\bw_1^n))\right)_1}^2 &= \Bigg\|\frac{1}{2}\bw_1 + \frac{1}{2\pi}\left(\frac{\norm{\bw_2}}{\norm{\bw_1}}\sin(\theta)\bw_1 + (\pi - \theta)\bw_2  \right) -\nonumber\\
    -& \frac{1}{2\pi}\left( \frac{\alpha}{\norm{\bw_1}}\sin(\gamma_1)\bw_1 + (\pi-\gamma_1)\alpha\bv_1  \right) \Bigg\|^2\nonumber \\
    &=\left\| \frac{\norm{\bw_2}\sin(\theta)}{2\pi\norm{\bw_1}}\bw_1 - \frac{\theta}{2\pi}\bw_2 - \frac{\alpha\sin(\gamma_1)}{2\pi\norm{\bw_1}}\bw_1 + \frac{\alpha\gamma_1}{2\pi}\bv_1  \right\|^2 = O(\epsilon^2)~.
\end{align}
In the same manner as in \eqref{eq:PL proof bound on norm of gradient} we can show that also the norm of every other coordinate of the gradient of the objective is $O(\epsilon^2)$, hence we also have that $\norm{ \nabla F(\bw_1^n)}^2 = O(\epsilon^2)$, where here the $O$ notation hides a linear term in $n$ (note that $F(\bw_1^n)$ does not depend on $n$). In particular for every $\lambda >0$  we can find $\epsilon >0$ such that  $\norm{ \nabla F(\bw_1^n))}^2  < \lambda\cdot(F(\bw_1^n) - f^*)$ (Recall that $f^*$ is the value at the global minimum which is $0$). This shows that the PL condition does not hold, even in an $\epsilon$-orthogonal neighborhood of the global minimum.
\end{proof}

\section{Proofs from \subsecref{sec:OPSC in most directions}}

The Hessian at $\bw_1^n = (\bw_{i,j})_{i,j=1}^{k,m}$ in the direction of global minimum $\tilde{\bw}_1^n = (\tilde{\bw}_{i,j})_{i,j=1}^{k,m}$ is (recall that $\bw_1^n - \tilde{\bw}_1^n = \bg_1^n$):
\begin{align}\label{eq:full hessian directed to the global minimum}
    \bg_1^{n\top} H(\bw_1^n)\bg_1^n &=\sum_{i,j=1}^{k,m}\Bigg( \frac{1}{2}\norm{\bg_{i,j}}^2 + \sum_{\substack{a,b=1 \\ (a,b)\neq (i,j)}}^{k,m} \bg_{i,j}^\top h_1(\bw_{i,j},\bw_{a,b})\bg_{i,j} - \sum_{l=1}^k \bg_{i,j}^\top h_1(\bw_{i,j},\bv_l)\bg_{i,j} +\nonumber\\
    & +\sum_{\substack{a,b=1 \\ (a,b)\neq (i,j)}}^{k,m}\bg_{i,j}^\top h_2(\bw_{i,j},\bw_{a,b})\bg_{a,b}\Bigg)
\end{align}

The proof idea of \thmref{thm:almost OPSC} is to bound each term in \eqref{eq:full hessian directed to the global minimum} separately. Since we look at a point close to the global minimum, each $\bw_{i,j}$ should be close to its target vector $\bv_i$, hence most of the expressions will \emph{almost} cancel out, up to an $O(\sqrt{\epsilon})$ factor. 

For the proof we denote the following angles for ease of notations: 
\begin{enumerate}
    \item $\theta_{i,j}^{a,b}$: the angle between $\bw_{i,j}$ and $\bw_{a,b}$ for $i,a\in[k],~j,b\in[m]$.
    \item $\gamma_{i,j}^l$: the angle between $\bw_{i,j}$ and $v_l$ for $i,l\in[k],~j\in[m]$.
\end{enumerate}
For every $i,j$ we can write $\bw_{i,j} = \frac{1}{m}\bv_i + \bg_{i,j}$. Assume in the following that for some $\epsilon >0$, we have that $\bw_1^n\in U_\epsilon^\perp(\tilde{\bw}_1^n)$, hence we have that $\norm{\bg_{i,j}}\leq \epsilon$ and $\bg_{i,j}\perp \bv_i$. We will need the following terms for $i\in[k],~j,l\in[m]$:

\begin{align}
    & \|\bw_{i,j}\| = \left\|\frac{1}{m}\bv_i + \bg_{i,j}\right\| = \sqrt{\frac{1}{m^2} + \inner{\bv_i,\bg_{i,j}} + \norm{\bg_{i,j}}^2} = \sqrt{\frac{1}{m^2} + \norm{\bg_{i,j}}^2} =  \frac{1}{m} + O(\epsilon)\label{eq:norm of w i j}\\
    &\cos(\theta_{i,j}^{i,l}) = \frac{\inner{\frac{1}{m}\bv_i + \bg_{i,j}, \frac{1}{m}\bv_i + \bg_{i,l}}}{\norm{\bw_{i,j}}\norm{\bw_{i,l}}} = \frac{\frac{1}{m^2}  + \inner{\bg_{i,j},\bg_{i,l}}}{\frac{1}{m^2} + O(\epsilon)} = 1 + O(\epsilon) \label{eq:cos w same target}\\
    &\sin(\theta_{i,j}^{i,l}) = \sqrt{1 - \cos^2(\theta_{i,j}^{i,l})} = O(\sqrt{\epsilon})\label{eq:sin w same target}\\
    &\cos(\gamma_{i,j}^i) = \frac{\inner{\frac{1}{m}\bv_i + \bg_{i,j},\bv_i}}{\norm{\bw_{i,j}}} = \frac{\frac{1}{m}}{\frac{1}{m} + O(\epsilon)} = 1 + O(\epsilon)\\
    & \sin(\gamma_{i,j}^i) = \sqrt{1 - \cos^2(\gamma_{i,j}^i)} = O(\sqrt{\epsilon})\label{eq:sin w v same target}
\end{align}

For $i,a\in[k]$ and $j,b\in[m]$ with $i\neq a$ we have that:
\begin{align}
    &\cos(\theta_{i,j}^{a,b}) = \frac{\inner{\frac{1}{m}\bv_i + \bg_{i,j}, \frac{1}{m}\bv_a + \bg_{a,b}}}{\norm{\bw_{i,j}}\norm{\bw_{a,b}}} = \frac{\inner{\bg_{i,j},\bv_a} + \inner{\bg_{a,b},\bv_i} + \inner{\bg_{i,j},\bg_{a,b}}}{\frac{1}{m^2} + O(\epsilon)} = O(\epsilon)\label{eq:cos w different target}\\
    &\sin(\theta_{i,j}^{a,b}) = \sqrt{1 - \cos^2(\theta_{i,j}^{a,b})} = 1 + O(\sqrt{\epsilon})\label{eq:sin w different target}\\
    &\cos(\gamma_{i,j}^a) = \frac{\inner{\frac{1}{m}\bv_i + \bg_{i,j},\bv_a}}{\norm{\bw_{i,j}}} = \frac{\inner{\bg_{i,j},\bv_a}}{\frac{1}{m} + O(\epsilon)} = O(\epsilon)\label{eq:cos w v different target}\\
    &\sin(\gamma_{i,j}^a) = \sqrt{1 - \cos^2(\gamma_{i,j}^a)} = 1 + O(\sqrt{\epsilon})\label{eq:sin w v different target}
\end{align}

We will first bound the terms in \eqref{eq:full hessian directed to the global minimum} which are related to $h_1$.
\begin{lemma}\label{lem:h 1 same target}
For every $i\in[k]$ and $j,l\in[m]$ we have that:
\begin{enumerate}
    \item $\bg_{i,j}^\top h_1(\bw_{i,j},\bw_{i,l})\bg_{i,j} = O\left(\epsilon^{2.5}\right)$
    \item $\bg_{i,j}^\top h_1(\bw_{i,j},\bv_{i})\bg_{i,j} = O\left(\epsilon^{2.5}\right)$
\end{enumerate} 
\end{lemma}
\begin{proof}
By Lemma 9 of \cite{safran2017spurious} we know that the largest eigenvalue of $h_1(\bw,\bv)$ is $\frac{\sin(\theta_{\bw,\bv})\norm{\bv}}{\bw}$. Hence we have that:
\begin{align*}
    \bg_{i,j}^\top h_1(\bw_{i,j},\bw_{i,l})\bg_{i,j} \leq \norm{\bg_{i,j}}^2 \frac{\sin(\theta_{i,j}^{i,l})\norm{\bw_{i,l}}}{\norm{\bw_{i,j}}} = O\left(\epsilon^{2.5}\right)
\end{align*}
where we used \eqref{eq:norm of w i j}, \eqref{eq:sin w same target} and that $\norm{\bg_{i,j}}= O(\epsilon)$.

The second part is the same as the first, where we use \eqref{eq:sin w v same target}.
\end{proof}

Now we can bound all the terms in \eqref{eq:full hessian directed to the global minimum} related to $h_1$, leaving only on $O(\sqrt{\epsilon})$ term. Note that in the main theorem we will divide the full expression by the sum of the norms of $\bg_{i,j}$, which is of magnitude $O(\epsilon^2)$.

\begin{lemma}\label{lem:h1 sum}
For every $i\in [k],j\in[m]$ we have that:
\begin{align}
    \sum_{\substack{a,b=1 \\ (a,b)\neq (i,j)}}^{k,m} \bg_{i,j}^\top h_1(\bw_{i,j},\bw_{a,b})\bg_{i,j} - \sum_{l=1}^k \bg_{i,j}^\top h_1(\bw_{i,j},\bv_l)\bg_{i,j} = O\left(\epsilon^{2.5}\right)
\end{align}
\end{lemma}

\begin{proof}
Let $i,a\in[k]$ and $j,b\in[m]$ where $i\neq a$. First we have:
\begin{align}
     &\frac{\inner{\bg_{i,j}, \bw_{i,j}}}{\norm{\bw_{i,j}}} \stackrel{(\ref{eq:norm of w i j})}{=} \frac{\inner{\bg_{i,j},\frac{1}{m}\bv_i + \bg_{i,j}}}{\frac{1}{m} + O(\epsilon)} = \frac{\norm{\bg_{i,j}}^2}{\frac{1}{m} + O(\epsilon)} = O(\epsilon^2)\label{eq:g i j times w i j} \\
     &\frac{\inner{\bg_{i,j}, \bar{\bw}_{a,b} - \cos(\theta_{i,j}^{a,b})\bar{\bw}_{i,j}}}{\sin(\theta_{i,j}^{a,b})} = \frac{\inner{\bg_{i,j}, \frac{1}{m}\bv_a + \bg_{a,b}}}{\norm{\bw_{a,b}}\sin(\theta_{i,j}^{a,b})} - \frac{\cos(\theta_{i,j}^{a,b})\inner{\bg_{i,j}, \frac{1}{m}\bv_i + \bg_{i,j}}}{\norm{\bw_{i,j}}\sin(\theta_{i,j}^{a,b})} \nonumber \\
     &\stackrel{(\ref{eq:cos w different target}),(\ref{eq:sin w different target})}{=} \frac{\frac{1}{m}\inner{\bg_{i,j},\bv_a} + O(\epsilon^2)}{\left(\frac{1}{m} + O(\epsilon)\right)\cdot (1+O(\sqrt{\epsilon}))} - \frac{\norm{\bg_{i,j}}^2\cdot O(\epsilon)}{\left(\frac{1}{m} + O(\epsilon)\right)\cdot (1+O(\sqrt{\epsilon}))} = \inner{\bg_{i,j},\bv_a} + O(\epsilon^{1.5})\label{eq: g i j times n w  a b w i j}
\end{align}

Using the function $h_1$ we have:
\begin{align}\label{eq:h1 of wij and wab}
    \bg_{i,j}^\top h_1(\bw_{i,j},\bw_{a,b})\bg_{i,j} &= \frac{\sin(\theta_{i,j}^{a,b})\norm{\bw_{a,b}}}{2\pi \norm{\bw_{i,j}}}\left(\norm{\bg_{i,j}}^2 - \frac{\inner{\bg_{i,j}, \bw_{i,j}}^2}{\norm{\bw_{i,j}}^2} + \frac{\inner{\bg_{i,j}, \bar{\bw}_{a,b} - \cos(\theta_{i,j}^{a,b})\bar{\bw}_{i,j}}^2}{\sin^2(\theta_{i,j}^{a,b})}\right) \nonumber \nonumber\\
    & \stackrel{(\ref{eq:g i j times w i j}), (\ref{eq: g i j times n w  a b w i j})}{=} \frac{\sin(\theta_{i,j}^{a,b})\norm{\bw_{a,b}}}{2\pi \norm{\bw_{i,j}}}\left(\norm{\bg_{i,j}}^2 + O(\epsilon^4) + \inner{\bg_{i,j},\bv_a}^2 + O(\epsilon^{2.5}) \right) \nonumber\\
    & \stackrel{(\ref{eq:norm of w i j}),(\ref{eq:sin w different target})}{=} \frac{\left(1 + O(\sqrt{\epsilon})\right)\cdot \left(\frac{1}{m} + O(\epsilon)\right)}{2\pi \left(\frac{1}{m} + O(\epsilon)\right)}\left(\norm{\bg_{i,j}}^2 + \inner{\bg_{i,j},\bv_a}^2 + O(\epsilon^{2.5}) \right) \nonumber \\
    &=\frac{1}{2\pi}\left(\norm{\bg_{i,j}}^2 + \inner{\bg_{i,j},\bv_a}^2\right) + O(\epsilon^{2.5})
\end{align}
In the same manner for $\bv_a$ we have (recall that $\norm{\bv_a}=1$):

\begin{align}
    &\frac{\inner{\bg_{i,j}, \bv_a - \cos(\gamma_{i,j}^{a})\bar{\bw}_{i,j}}^2}{\sin^2(\gamma_{i,j}^{a})} = \frac{\inner{\bg_{i,j},\bv_a}}{\sin(\gamma_{i,j}^a)} - \frac{\cos(\gamma_{i,j}^a)\inner{\bg_{i,j},\frac{1}{m}\bv_i + \bg_{i,j}}}{\norm{\bw_{i,j}}\sin(\gamma_{i,j}^a)}\nonumber \\
    & \stackrel{(\ref{eq:cos w v different target}),(\ref{eq:sin w v different target})}{=}\frac{\inner{\bg_{i,j},\bv_a}}{1+O(\epsilon)} - \frac{\norm{\bg_{i,j}}^2\cdot O(\epsilon)}{\left(\frac{1}{m} + O(\epsilon)\right)\cdot (1+O(\sqrt{\epsilon}))} = \inner{\bg_{i,j},\bv_a} + O(\epsilon^{1.5})\label{eq:g i j times n v a w i j}
\end{align}
Hence we get:
\begin{align}\label{eq:h1 of wij and va}
    \bg_{i,j}^\top h_1(\bw_{i,j},\bv_a)\bg_{i,j} &= \frac{\sin(\gamma_{i,j}^{a})}{2\pi \norm{\bw_{i,j}}}\left(\norm{\bg_{i,j}}^2 - \frac{\inner{\bg_{i,j}, \bw_{i,j}}^2}{\norm{\bw_{i,j}}^2} + \frac{\inner{\bg_{i,j}, \bv_a - \cos(\gamma_{i,j}^{a})\bar{\bw}_{i,j}}^2}{\sin^2(\gamma_{i,j}^{a})}\right) \nonumber\\
    &\stackrel{(\ref{eq:g i j times w i j}),(\ref{eq:g i j times n v a w i j})}{=}\frac{\sin(\gamma_{i,j}^{a})}{2\pi \norm{\bw_{i,j}}}\left(\norm{\bg_{i,j}}^2 + O(\epsilon^4) + \inner{\bg_{i,j},\bv_a}^2 + O(\epsilon^{2.5}) \right) \nonumber\\
    &\stackrel{(\ref{eq:norm of w i j}),(\ref{eq:sin w v different target})}{=}\frac{1 + O(\sqrt{\epsilon})}{2\pi \left(\frac{1}{m} + O(\epsilon)\right)}\left(\norm{\bg_{i,j}}^2 + \inner{\bg_{i,j},\bv_a}^2 + O(\epsilon^{2.5})\right) \nonumber\\
    &= \frac{m}{2\pi}\left(\norm{\bg_{i,j}}^2 + \inner{\bg_{i,j},\bv_a}^2\right) + O(\epsilon^{2.5})
\end{align}

For any $a\in[k]$ with $a\neq i$, combining \eqref{eq:h1 of wij and wab} and \eqref{eq:h1 of wij and va} and summing over all $b\in[m]$ we get:
\begin{align}
    &\sum_{b=1}^m\left(\bg_{i,j}^\top h_1(\bw_{i,j},\bw_{a,b})\bg_{i,j}\right) - \bg_{i,j}^\top h_1(\bw_{i,j},\bv_a)\bg_{i,j} \nonumber\\
    = &\sum_{b=1}^m\left( \frac{1}{2\pi}\left(\norm{\bg_{i,j}}^2 + \inner{\bg_{i,j},\bv_a}^2\right) + O(\epsilon^{2.5})\right) - \frac{m}{2\pi}\left(\norm{\bg_{i,j}}^2 + \inner{\bg_{i,j},\bv_a}^2\right) + O(\epsilon^{2.5}) = O(\epsilon^{2.5})\nonumber
\end{align}
Also, using \lemref{lem:h 1 same target} we get that $\sum_{b\neq j}\bg_{i,j}^\top h_1(\bw_{i,j},\bw_{i,b})\bg_{i,j} = O(\epsilon^{2.5})$ and that $\bg_{i,j}^\top h_1(\bw_{i,j},\bv_a)\bg_{i,j} = O(\epsilon^{2.5})$. This finishes the proof.
\end{proof}

Now we will bound terms related to $h_2$:
\begin{lemma}\label{lem:h2 different target }
Letting $i,a\in[k]$ with $i\neq a$ and $b,j\in[m]$, we have:
\[
\bg_{i,j}^\top h_2(\bw_{i,j},\bw_{a,b})\bg_{a,b} = \frac{1}{4}\inner{\bg_{i,j},\bg_{a,b}} + \frac{1}{2\pi}\inner{\bg_{i,j},\bv_a}\cdot \inner{\bg_{a,b},\bv_i} + O(\epsilon^{2.5})
\]
\end{lemma}

\begin{proof}
We use \eqref{eq:h2} to get:
\begin{align}\label{eq:h2 different target}
    \bg_{i,j}^\top h_2(\bw_{i,j},\bw_{a,b})\bg_{a,b} = \frac{1}{2\pi}\left( (\pi - \theta_{i,j}^{a,b})\inner{\bg_{i,j},\bg_{a,b}} + \bg_{i,j}^\top\bar{n}_{\bw_{i,j},\bw_{a,b}}\bar{\bw}_{a,b}^\top\bg_{a,b} + \bg_{i,j}^\top\bar{n}_{\bw_{a,b},\bw_{i,j}}\bar{\bw}_{i,j}^\top\bg_{a,b} \right)
\end{align}

We will now bound each expression in \eqref{eq:h2 different target}. For the first term we will bound the angle $\theta_{i,j}^{a,b}$ using the Taylor series of $\arccos$. The Taylor series of $\arccos$ is $\arccos(x) = \frac{\pi}{2} - \sum_{n=0}^\infty \frac{(2n)!}{2^{2^n}(n!)^2}\frac{x^{2n+1}}{2n+1} = \frac{\pi}{2} -  \sum_{n=0}^\infty c_n x^{2n+1}$ where $c_n \leq \frac{1}{2}$ for all $n \geq 0$. Hence we have that:
\begin{align*}
    \theta_{i,j}^{a,b} = \arccos\left(\cos(\theta_{i,j}^{a,b})\right) = \frac{\pi}{2} - \sum_{n=0}^\infty c_n \left(\cos(\theta_{i,j}^{a,b})\right)^{2n+1} \stackrel{(\ref{eq:cos w different target})}{=} \frac{\pi}{2} + O(\epsilon)
\end{align*}

Hence we can bound the first term:
\begin{align}\label{eq:h2 different targets first expression}
    (\pi - \theta_{i,j}^{a,b})\inner{\bg_{i,j},\bg_{a,b}} = \left(\pi - \frac{\pi}{2} + O(\epsilon)\right)\inner{\bg_{i,j},\bg_{a,b}} = \frac{\pi}{2}\inner{\bg_{i,j},\bg_{a,b}} + O(\epsilon^3)
\end{align}

For the second term we have:
\begin{align}\label{eq:h2 different targets second expression}
    &\bg_{i,j}^\top\bar{n}_{\bw_{i,j},\bw_{a,b}}\bar{\bw}_{a,b}^\top\bg_{a,b} = \nonumber\\
    &= \frac{\inner{\bar{\bw}_{a,b},\bg_{a,b}}\cdot\left(\inner{\bar{\bw}_{i,j},\bg_{i,j}}-\cos(\theta_{i,j}^{a,b})\inner{\bar{\bw}_{a,b},\bg_{i,j}} \right)}{\sin(\theta_{i,j}^{a,b})} \nonumber \\
    &=\frac{\inner{\frac{1}{m}\bv_a + \bg_{a,b},\bg_{a,b}}\cdot\left(\frac{1}{\norm{\bw_{i,j}}}\inner{\frac{1}{m}\bv_i + \bg_{i,j},\bg_{i,j}}-\frac{\cos(\theta_{i,j}^{a,b})}{\norm{\bw_{a,b}}}\inner{\frac{1}{m}\bv_a + \bg_{a,b},\bg_{i,j}} \right)}{\sin(\theta_{i,j}^{a,b})\norm{\bw_{a,b}}} \nonumber \\
    &= \frac{\norm{\bg_{a,b}^2}\left(\frac{\norm{\bg_{i,j}}^2}{\norm{\bw_{i,j}}} - \frac{\cos(\theta_{i,j}^{a,b})}{\norm{\bw_{a,b}}}\left(\frac{1}{m}\inner{\bv_a ,\bg_{i,j}} + \inner{\bg_{a,b},\bg_{i,j}}\right)\right)}{\sin(\theta_{i,j}^{a,b})\norm{\bw_{a,b}}} \nonumber \\
    &\stackrel{(\ref{eq:norm of w i j}),(\ref{eq:sin w different target}),(\ref{eq:cos w different target})}{=} \frac{\norm{\bg_{a,b}}^2\left(\frac{\norm{\bg_{i,j}}^2}{\frac{1}{m} + O(\epsilon)} - \frac{O(\epsilon)}{\frac{1}{m} + O(\epsilon)}\left(\frac{1}{m}\inner{\bv_a ,\bg_{i,j}} + \inner{\bg_{a,b},\bg_{i,j}}\right) \right)}{(1 + O(\sqrt{\epsilon}))\cdot \left(\frac{1}{m} + O(\epsilon)\right)} = O(\epsilon^3)
\end{align}

For the third expression we have:
\begin{align*}
    &\bg_{i,j}^\top\bar{n}_{\bw_{a,b},\bw_{i,j}}\bar{\bw}_{i,j}^\top\bg_{a,b} = \nonumber\\
     &= \frac{\inner{\bar{\bw}_{i,j},\bg_{a,b}}\cdot\left(\inner{\bar{\bw}_{a,b},\bg_{i,j}}-\cos(\theta_{i,j}^{a,b})\inner{\bar{\bw}_{i,j},\bg_{i,j}} \right)}{\sin(\theta_{i,j}^{a,b})} \nonumber \\
     &=\frac{\inner{\frac{1}{m}\bv_i + \bg_{i,j},\bg_{a,b}}\cdot\left(\frac{1}{\norm{\bw_{a,b}}}\inner{\frac{1}{m}\bv_a + \bg_{a,b},\bg_{i,j}}-\frac{\cos(\theta_{i,j}^{a,b})}{\norm{\bw_{i,j}}}\inner{\frac{1}{m}\bv_i + \bg_{i,j},\bg_{i,j}} \right)}{\sin(\theta_{i,j}^{a,b})\norm{\bw_{i,j}}} \nonumber \\
     &=\frac{\left(\frac{1}{m}\inner{\bv_i,\bg_{a,b}} + \inner{\bg_{i,j},\bg_{a,b}}\right)\cdot \left(\frac{1}{\norm{\bw_{a,b}}}\left(\frac{1}{m}\inner{\bv_a,\bg_{i,j}} + \inner{\bg_{a,b},\bg_{i,j}}\right) - \frac{\cos(\theta_{i,j}^{a,b})\norm{\bg_{i,j}}^2}{\norm{\bw_{i,j}}}\right)}{\sin(\theta_{i,j}^{a,b})\norm{\bw_{i,j}}} \nonumber \\
     &=\frac{\inner{\bv_i,\bg_{a,b}}\cdot \inner{\bv_a,\bg_{i,j}}}{m^2\sin(\theta_{i,j}^{a,b})\norm{\bw_{a,b}}\norm{\bw_{i,j}}} + \frac{\inner{\bv_i,\bg_{a,b}}\cdot \left(\frac{\inner{\bg_{a,b},\bg_{i,j}}}{\norm{\bw_{a,b}}} - \frac{\cos(\theta_{i,j}^{a,b})\norm{\bg_{i,j}}^2}{\norm{\bw_{i,j}}}\right)}{m\sin(\theta_{i,j}^{a,b})\norm{\bw_{i,j}}}\nonumber\\
     &+ \frac{\inner{\bg_{i,j},\bg_{a,b}}\cdot \left(\frac{1}{\norm{\bw_{a,b}}}\left(\frac{1}{m}\inner{\bv_a,\bg_{i,j}} + \inner{\bg_{a,b},\bg_{i,j}}\right) - \frac{\cos(\theta_{i,j}^{a,b})\norm{\bg_{i,j}}^2}{\norm{\bw_{i,j}}}\right)}{\sin(\theta_{i,j}^{a,b})\norm{\bw_{i,j}}}
\end{align*}
As in the previous expression, since $\inner{\bg_{i,j},\bg_{a,b}} , \norm{\bg_{i,j}}^2 = O(\epsilon^2)$, we have that:
\begin{align*}
    &\frac{\inner{\bv_i,\bg_{a,b}}\cdot \left(\frac{\inner{\bg_{a,b},\bg_{i,j}}}{\norm{\bw_{a,b}}} - \frac{\cos(\theta_{i,j}^{a,b})\norm{\bg_{i,j}}^2}{\norm{\bw_{i,j}}}\right)}{m\sin(\theta_{i,j}^{a,b})\norm{\bw_{i,j}}} = O(\epsilon^3) \\
    & \frac{\inner{\bg_{i,j},\bg_{a,b}}\cdot \left(\frac{1}{\norm{\bw_{a,b}}}\left(\frac{1}{m}\inner{\bv_a,\bg_{i,j}} + \inner{\bg_{a,b},\bg_{i,j}}\right) - \frac{\cos(\theta_{i,j}^{a,b})\norm{\bg_{i,j}}^2}{\norm{\bw_{i,j}}}\right)}{\sin(\theta_{i,j}^{a,b})\norm{\bw_{i,j}}} = O(\epsilon^3)
\end{align*}
In total we get that:
\begin{align}\label{eq:h2 different targets third expression}
    &\bg_{i,j}^\top\bar{n}_{\bw_{a,b},\bw_{i,j}}\bar{\bw}_{i,j}^\top\bg_{a,b} = \frac{\inner{\bv_i,\bg_{a,b}}\cdot \inner{\bv_a,\bg_{i,j}}}{m^2\sin(\theta_{i,j}^{a,b})\norm{\bw_{a,b}}\norm{\bw_{i,j}}} + O(\epsilon^3)\nonumber\\
    &\stackrel{(\ref{eq:norm of w i j}),(\ref{eq:sin w different target})}{=} \frac{\inner{\bv_i,\bg_{a,b}}\cdot \inner{\bv_a,\bg_{i,j}}}{m^2 (1+O(\sqrt{\epsilon}))\left(\frac{1}{m} + O(\epsilon)\right)^2} + O(\epsilon^3)= \inner{\bv_i,\bg_{a,b}}\cdot \inner{\bv_a,\bg_{i,j}} + O(\epsilon^{2.5})
\end{align}

Overall, using \eqref{eq:h2 different targets first expression}, \eqref{eq:h2 different targets second expression}, \eqref{eq:h2 different targets third expression} we have:
\begin{align*}
     &\bg_{i,j}^\top h_2(\bw_{i,j},\bw_{a,b})\bg_{a,b} = \frac{1}{2\pi}\left(\frac{\pi}{2}\inner{\bg_{i,j},\bg_{a,b}} + O(\epsilon^3) + O(\epsilon^3) + \inner{\bv_i,\bg_{a,b}}\cdot\inner{\bv_a,\bg_{i,j}^{a,b}} + O(\epsilon^{2.5})\right)\\
     &= \frac{1}{4}\inner{\bg_{i,j},\bg_{a,b}} + \frac{1}{2\pi}\inner{\bg_{i,j},\bv_a}\cdot \inner{\bg_{a,b},\bv_i} + O(\epsilon^{2.5})
\end{align*}
\end{proof}

\begin{lemma}\label{lem:h2 same target }
Letting $i\in [k]$ and $j,l\in[m]$, we have:
\[
\bg_{i,j}^\top h_2(\bw_{i,j},\bw_{i,l})\bg_{i,l} = \frac{1}{2}\inner{\bg_{i,j},\bg_{i,l}} + O(\epsilon^{3})
\]
\end{lemma}

\begin{proof}
We use \thmref{thm:hessian of objective} to get:
\begin{align}\label{eq:h2 same target}
    \bg_{i,j}^\top h_2(\bw_{i,j},\bw_{i,l})\bg_{i,l} = \frac{1}{2\pi}\left( (\pi - \theta_{i,j}^{i,l})\inner{\bg_{i,j},\bg_{i,l}} + \bg_{i,j}^\top\bar{n}_{\bw_{i,j},\bw_{i,l}}\bar{\bw}_{i,l}^\top\bg_{i,l} + \bg_{i,j}^\top\bar{n}_{\bw_{i,l},\bw_{i,j}}\bar{\bw}_{i,j}^\top\bg_{i,l} \right)
\end{align}

We will bound each expression of \eqref{eq:h2 same target}. For the first expression, as in the proof of \lemref{lem:h2 different target } we use the Taylor series of $\arccos$ to get:
\begin{align*}
    \theta_{i,j}^{a,b} = \arccos\left(\cos(\theta_{i,j}^{i,l})\right) = \frac{\pi}{2} - \sum_{n=0}^\infty c_n \left(\cos(\theta_{i,j}^{i,l})\right)^{2n+1} \stackrel{(\ref{eq:cos w same target})}{=} O(\epsilon)~.
\end{align*}
Hence, we have that:
\begin{align}\label{eq:h2 same target first expression}
     (\pi - \theta_{i,j}^{i,l})\inner{\bg_{i,j},\bg_{i,l}} = \pi \inner{\bg_{i,j},\bg_{i,l}} + O(\epsilon^3)~.
\end{align}
For the second expression we get:
\begin{align}
    &\bg_{i,j}^\top\bar{n}_{\bw_{i,j},\bw_{i,l}}\bar{\bw}_{i,l}^\top\bg_{i,l} = \frac{\inner{\bar{\bw}_{i,l},\bg_{i,l}}\cdot\left(\inner{\bar{\bw}_{i,j},\bg_{i,j}}-\cos(\theta_{i,j}^{i,l})\inner{\bar{\bw}_{i,l},\bg_{i,j}} \right)}{\sin(\theta_{i,j}^{a,b})} \nonumber \\
    &= \frac{\frac{\norm{\bg_{i,l}}^2}{\norm{\bw_{i,l}}}\left( \frac{\norm{\bg_{i,j}}^2}{\norm{\bw_{i,j}}} - \cos(\theta_{i,j}^{i,l})\frac{\inner{\bg_{i,j},\bg_{i,l}}}{\norm{\bw_{i,l}}}\right)}{\sin(\theta_{i,j}^{i,l})} \stackrel{(\ref{eq:norm of w i j}),(\ref{eq:cos w same target}),(\ref{eq:sin w same target})}{=}\frac{\frac{\norm{\bg_{i,l}}^2}{\frac{1}{m} + O(\epsilon)}\left( \frac{\norm{\bg_{i,j}}^2}{\frac{1}{m} + O(\epsilon)} - (1+O(\epsilon))\frac{\inner{\bg_{i,j},\bg_{i,l}}}{\frac{1}{m} + O(\epsilon)}\right)}{O(\sqrt{\epsilon})}\nonumber \\
    &= O(\epsilon^{3.5})~. \label{eq:h2 same target second expression}
\end{align}
For the third expression we have:
\begin{align}
    &\bg_{i,j}^\top\bar{n}_{\bw_{i,l},\bw_{i,j}}\bar{\bw}_{i,j}^\top\bg_{i,l} = \frac{\inner{\bar{\bw}_{i,j},\bg_{i,l}}\cdot\left(\inner{\bar{\bw}_{i,l},\bg_{i,j}}-\cos(\theta_{i,j}^{i,l})\inner{\bar{\bw}_{i,j},\bg_{i,j}} \right)}{\sin(\theta_{i,j}^{i,l})} \nonumber \\
    &= \frac{\frac{\inner{\bg_{i,j},\bg_{i,l}}}{\norm{\bw_{i,j}}}\cdot \left(\frac{\inner{\bg_{i,j},\bg_{i,l}}}{\norm{\bw_{i,l}}} - \cos(\theta_{i,j}^{i,l})\frac{\norm{\bg_{i,j}}^2}{\norm{\bw_{i,j}}} \right)}{\sin(\theta_{i,j}^{i,l})} \stackrel{(\ref{eq:norm of w i j}),(\ref{eq:cos w same target}),(\ref{eq:sin w same target})}{=}\frac{\frac{\inner{\bg_{i,j},\bg_{i,l}}}{\frac{1}{m} + O(\epsilon)}\cdot \left(\frac{\inner{\bg_{i,j},\bg_{i,l}}}{\frac{1}{m} + O(\epsilon)} - (1+ O(\epsilon))\frac{\norm{\bg_{i,j}}^2}{\frac{1}{m} + O(\epsilon)} \right)}{O(\sqrt{\epsilon)}} \nonumber \\
    &= O(\epsilon^{3.5})~.\label{eq:h2 same target third expression}
\end{align}
Combining \eqref{eq:h2 same target first expression}, \eqref{eq:h2 same target second expression} and \eqref{eq:h2 same target third expression} we get:
\begin{align*}
     \bg_{i,j}^\top h_2(\bw_{i,j},\bw_{i,l})\bg_{i,l} = \frac{1}{2\pi}\left( \pi\inner{\bg_{i,j},\bg_{i,l}} + O(\epsilon^3) + O(\epsilon^{3.5})+ O(\epsilon^{3.5})\right) = \frac{1}{2}\inner{\bg_{i,j},\bg_{i,l}} + O(\epsilon^3)
\end{align*}
\end{proof}

Before combining all the parts we will also need the following technical lemma:
\begin{lemma}\label{lemma:useful bounds on sums of u i j}
Let $\bu_1,\dots, \bu_n\in \mathbb{R}^n$, and denote by $u_{ij}$ the $j-th$ coordinate of $\bu_i$. Then:
\begin{equation*}
    \sum_{i=1}^n\sum_{j\neq i} u_{ij}u_{ji} \geq \sum_{i=1}^n u_{ii}^2 - \sum_{i=1}^n \norm{\bu_i}^2
\end{equation*}
\end{lemma}

\begin{proof}
Let $U$ be a matrix with columns equal to $\bu_i$.
Note that:
\begin{align*}
    & tr(U^\top U) = \sum_{i=1}^n\sum_{j=1}^n u_{ij}^2 \\
    & tr(U^2) = \sum_{i=1}^n \sum_{j=1}^n u_{ij}u_{ji} \\
    & tr(U U^\top) = \sum_{i=1}^n\sum_{j=1}^n u_{ji}^2 
\end{align*}
Now we have that:
\begin{align}
    &tr(U^\top U) + tr(U^2) = \frac{1}{2}tr(U^\top U) + tr(U^2) + \frac{1}{2}tr(U U^\top) = \sum_{i=1}^n\sum_{j=1}^n \left( \frac{1}{2}u_{ij}^2 + u_{ij}u_{ji} + \frac{1}{2}u_{ji}^2  \right) \label{eq:tr U top U plus tr U squared}\\
    & = \frac{1}{2}\sum_{i=1}^n \sum_{j=1}^n (u_{ij} + u_{ji})^2 \geq \frac{1}{2}\sum_{i=1}^n\sum_{j\neq i}(u_{ij} + u_{ji})^2 + 2\sum_{i=1}^n u_{ii}^2 \geq 2\sum_{i=1}^n u_{ii}^2 \nonumber
\end{align}
The two terms from the first part of \eqref{eq:tr U top U plus tr U squared} also have the following useful forms:
\begin{align*}
    & tr(U^\top U) = \sum_{i=1}^n \|\bu_i\|^2 \\
    & tr(U^2) = \sum_{i=1}^n\sum_{j\neq i} u_{ij} u_{ji} + \sum_{i=1}^n u_{ii}^2
\end{align*}
In total we have:
\begin{align*}
    \sum_{i=1}^n\sum_{j\neq i} u_{ij} u_{ji} + \sum_{i=1}^n u_{ii}^2 + \sum_{i=1}^n \|\bu_i\|^2 \geq 2\sum_{i=1}^n u_{ii}^2,
\end{align*}
hence:
\begin{align*}
    \sum_{i=1}^n\sum_{j\neq i} u_{ij} u_{ji} \geq \sum_{i=1}^n u_{ii}^2 -  \sum_{i=1}^n \|\bu_i\|^2 .
\end{align*}

\end{proof}

We are now ready to prove the main theorem of this section:
\begin{proof}[Proof of \thmref{thm:almost OPSC}]
By \lemref{lem:h1 sum} we have that
\begin{align}
    \sum_{i,j=1}^{k,m}\left(\sum_{\substack{a,b=1 \\ (a,b)\neq (i,j)}}^{k,m} \bg_{i,j}^\top h_1(\bw_{i,j},\bw_{a,b})\bg_{i,j} - \sum_{l=1}^k \bg_{i,j}^\top h_1(\bw_{i,j},\bv_l)\bg_{i,j}\right) = \sum_{i,j=1}^{k,m} O(\epsilon^{2.5}) = O(\epsilon^{2.5})~.
\end{align}
Applying the above to \eqref{eq:full hessian directed to the global minimum} we get:
\begin{align}\label{eq:full hessian without h1}
    \bg_1^{n\top} H(\bw_1^n)\bg_1^n = \sum_{i,j=1}^{k,m}\left(\frac{1}{2}\norm{\bg_{i,j}}^2 + \sum_{\substack{a,b=1 \\ (a,b)\neq (i,j)}}^{k,m}\bg_{i,j}^\top h_2(\bw_{i,j},\bw_{a,b})\bg_{a,b}\right) + O(\epsilon^{2.5})~.
\end{align}
First we separate the expression inside the sum of \eqref{eq:full hessian without h1} for the different values of $i,a$ where either $i=a$ or $i\neq a$:
\begin{align}
    &\sum_{i,j=1}^{k,m}\sum_{\substack{a,b=1 \\ (a,b)\neq (i,j)}}^{k,m}\bg_{i,j}^\top h_2(\bw_{i,j},\bw_{a,b})\bg_{a,b} \nonumber\\
    =& \sum_{i,j=1}^{k,m}\sum_{l\neq j} \bg_{i,j}^\top h_2(\bw_{i,j},\bw_{i,l})\bg_{i,l} + \sum_{i,j=1}^{k,m}\sum_{\substack{a,b=1 \\ a\neq i}}^{k,m}\bg_{i,j}^\top h_2(\bw_{i,j},\bw_{a,b})\bg_{a,b}~.\label{eq:h2 sum separeted}
\end{align}

Recall that $\bg_i= \sum_{j=1}^m\bg_{i,j}$ and $\bg = \sum_{i=1}^k\bg_i$. For the first sum in \eqref{eq:h2 sum separeted} we fix $i\in[k]$ and use \lemref{lem:h2 same target } to get:

\begin{align}\label{eq:sum of inner prod trick}
    &\sum_{j=1}^m\sum_{l\neq j} \bg_{i,j}^\top h_2(\bw_{i,j},\bw_{i,l})\bg_{i,l} = \sum_{j=1}\sum_{l\neq j} \frac{1}{2}\inner{\bg_{i,j},\bg_{i,l}} + O(\epsilon^3) = \frac{1}{2}\sum_{j=1}^m\inner{\bg_{i,j},\sum_{l\neq j}\bg_{i,l}} + O(\epsilon^3)\nonumber\\
    &= \frac{1}{2}\sum_{j=1}^m\inner{\bg_{i,j},\bg_i - \bg_{i,j}} + O(\epsilon^3) = \frac{1}{2}\left( \sum_{j=1}^m \inner{\bg_{i,j},\bg_i} - \sum_{j=1}^m\inner{\bg_{i,j},\bg_{i,j}}\right) + O(\epsilon^3)\nonumber \\
    &= \frac{1}{2}\norm{\bg_i}^2 - \frac{1}{2}\sum_{j=1}^m\norm{\bg_{i,j}}^2 + O(\epsilon^3)
\end{align}
Summing for all $i\in[k]$ we get:
\begin{align}\label{eq:h2 sum separeted first expression}
    \sum_{i,j=1}^{k,m}\sum_{l\neq j} \bg_{i,j}^\top h_2(\bw_{i,j},\bw_{i,l})\bg_{i,l} = \frac{1}{2}\sum_{i=1}^k\norm{\bg_{i}}^2 - \frac{1}{2}\sum_{i,j=1}^{k,m}\norm{\bg_{i,j}}^2 + O(\epsilon^3)
\end{align}
For the second sum in \eqref{eq:h2 sum separeted} we use \lemref{lem:h2 different target } to get:
\begin{align}\label{eq:h2 sum separeted second expression split}
    \sum_{i,j=1}^{k,m}\sum_{\substack{a,b=1 \\ a\neq i}}^{k,m}\bg_{i,j}^\top h_2(\bw_{i,j},\bw_{a,b})\bg_{a,b} = \sum_{i,j=1}^{k,m}\sum_{\substack{a,b=1 \\ a\neq i}}^{k,m}\left( \frac{1}{4}\inner{\bg_{i,j},\bg_{a,b}} + \frac{1}{2\pi}\inner{\bg_{i,j},\bv_a}\cdot \inner{\bg_{a,b},\bv_i}\right) + O(\epsilon^{2.5})
\end{align}

We will bound the two expressions in \eqref{eq:h2 sum separeted second expression split}. For the first expression we use the same calculation as \eqref{eq:sum of inner prod trick} to get:
\begin{align}\label{eq:inner prod g i j g a b sum}
    \sum_{i,j=1}^{k,m}\sum_{\substack{a,b=1 \\ a\neq i}}^{k,m}\frac{1}{4}\inner{\bg_{i,j},\bg_{a,b}} = \sum_{i=1}^k\sum_{a\neq i}\frac{1}{4}\inner{\bg_i,\bg_a} = \frac{1}{4}\norm{\bg}^2 - \frac{1}{4}\sum_{i=1}^k\norm{\bg_i}^2~.
\end{align}

For the second expression in \eqref{eq:h2 sum separeted second expression split} we first simplify:
\begin{align*}
    \frac{1}{2\pi}\sum_{i,j=1}^{k,m}\sum_{\substack{a,b=1 \\ a\neq i}}^{k,m} \inner{\bg_{i,j},\bv_i}\cdot \inner{\bg_{a,b},\bv_i} = \frac{1}{2\pi}\sum_{i=1}^k\sum_{a\neq i}\inner{\bg_i,\bv_a}\cdot \inner{\bg_a,\bv_i}
\end{align*}

Denote by $\bu_1,\dots,\bu_k\in\mathbb{R}^k$ the vectors  where the $a-th$ coordinate of $\bu_i$ is equal to $u_{ia} = \inner{\bg_i,\bv_a}$ for $a=1,\dots,k$. Since $\bg\in\mathbb{R}^d$ with $d \geq k$ and the vectors $\bv_1,\dots,\bv_k$ are orthonormal we have that $\norm{\bu_i}^2 = \sum_{a=1}^k \inner{\bg_i,\bv_a}^2 \leq \norm{\bg_i}^2$. Also, by the assumptions of the theorem we have $u_{ii} = \inner{\bg_i,\bv_i} = 0$. Using \lemref{lemma:useful bounds on sums of u i j} on the vectors $\bu_1,\dots,\bu_k$ we have:
\begin{align}\label{eq:mult inner prod gi va ga vi}
    \frac{1}{2\pi}\sum_{i=1}^k\sum_{a\neq i}\inner{\bg_i,\bv_a}\cdot \inner{\bg_a,\bv_i} = \frac{1}{2\pi}\sum_{i=1}^k\sum_{a\neq i} u_{ia}u_{ai} \geq \frac{1}{2\pi}\left(\sum_{i=1}^ku_{ii}^2 - \sum_{i=1}^k\norm{\bu_i}^2\right) \geq -\frac{1}{2\pi}\sum_{i=1}^k\norm{\bg_i}^2 + O(\epsilon^{2.5})
\end{align}

Returning to \eqref{eq:h2 sum separeted second expression split}, we combine \eqref{eq:inner prod g i j g a b sum} and \eqref{eq:mult inner prod gi va ga vi} to get:
\begin{align}
    \sum_{i,j=1}^{k,m}\sum_{\substack{a,b=1 \\ a\neq i}}^{k,m}\bg_{i,j}^\top h_2(\bw_{i,j},\bw_{a,b})\bg_{a,b} \geq \frac{1}{4}\norm{\bg}^2 - \frac{1}{4}\sum_{i=1}^k\norm{\bg_i}^2 - \frac{1}{2\pi}\sum_{i=1}^k\norm{\bg_i}^2 \nonumber
\end{align}
Combining the above with \eqref{eq:full hessian directed to the global minimum} and \eqref{eq:h2 sum separeted first expression} to get:
\begin{align*}
    &\bg_1^{n\top} H(\bw_1^n)\bg_1^n \\
    &\geq \frac{1}{2}\sum_{i,j=1}^{k,m}\norm{\bg_{i,j}}^2 + \frac{1}{2}\sum_{i=1}^k\norm{\bg_i}^2 - \frac{1}{2}\sum_{i,j=1}^{k,m}\norm{\bg_{i,j}}^2 + O(\epsilon^3) + \frac{1}{4}\norm{\bg}^2 - \left(\frac{1}{4} + \frac{1}{2\pi}\right)\cdot \sum_{i=1}^k\norm{\bg_i}^2 + O(\epsilon^{2.5}) \\
    &\geq \frac{1}{4}\norm{\bg}^2 + \left(\frac{1}{2} - \frac{1}{4} - \frac{1}{2\pi}\right)\sum_{i=1}^k\norm{\bg_i}^2 + O(\epsilon^{2.5}) \geq \frac{1}{4}\norm{\bg}^2 + \left(\frac{1}{4}-\frac{1}{2\pi}\right)\sum_{i=1}^k\norm{\bg_i}^2 + O(\epsilon^{2.5})
\end{align*}

\end{proof}

\section{Optimization Proofs}

\subsection{Empirical Investigation of Conjecture \ref{conj:epsilon}}\label{app:curv_experiment}

In this section, we wish to empirically verify the correctness of Conjecture \ref{conj:epsilon}. Our method is to focus on a small neighborhood of a global minimum, sample the vectors $\bg_{i,j}$ either randomly or adversarially, and verify that indeed the l.h.s is larger than the r.h.s, even without the $O(\epsilon^{2.5})$ term, and in a standard $\epsilon$-neighborhood of the global minimum. To this end, we sampled 1,000 random Gaussian vectors independently with zero mean and variance $10^{-5}$, for each pair of values $k\in\{5,10,20\}$ and $m\in\{2,5,10\}$, and computed the difference between the left-hand side and right-hand side of Conjecture \ref{conj:epsilon}. Additionally, we also tested the above difference in an adversarial sample, when the right-hand side cancels out, i.e.\ when $\sum_{j=1}^m\bg_{i,j}=\mathbf{0}$ for all $i\in[k]$, by replacing the noise vector $\bg_{i,1}$ with $-\sum_{j=2}^{m}\bg_{i,j}$. In \emph{all} the samples made, the difference computed was strictly positive. This means that empirically for many global minima (i.e. many $k$ and $m$) the conjecture holds. Table \ref{tab:curvs} summarizes the smallest curvature (i.e. the l.h.s of Conjecture \ref{conj:epsilon}) found among the 1,000 random samples performed on each pair.

\begin{table}[H]
    
    \centering

    \begin{tabular}{|c|c|c|c|}
    \hline
    $k$ & $m$ & Gaussian $\bg_{i,j}$   & $\bg_{i,1}=-\sum_{j=2}^{m}\bg_{i,j}$ \\ \hline
    5   & 2   & $0.84$  & $2.27\cdot10^{-4}$               \\ \hline
    5   & 5   & $2.43$ & $0.00754$                 \\ \hline
    5   & 10  & $5.43$  & $0.0693$              \\ \hline
    10  & 2   & $12.7$ & $0.0037$              \\ \hline
    10  & 5   & $43.1$  & $0.088$               \\ \hline
    10  & 10  & $62.3$ & $0.733$                \\ \hline
    20  & 2   & $88.6$  & $0.0318$                 \\ \hline
    20  & 5   & $231$ & $0.67$               \\ \hline
    20  & 10  & $412$  & $5.15$               \\ \hline
    
    \end{tabular}
    \caption{The minimal curvature measured among 1,000 random samples, for different values of $k$ and $m$. The third column is the minimal curvature measured for multivariate Normally-distributed noise $\bg_{i,j} $, and the rightmost column is the curvature measured for noise satisfying $\sum_{j=1}^m\bg_{i,j}=\mathbf{0}$. While the curvature is lower bounded by roughly $1$ for Gaussian noise, we see that the more problematic points in the neighborhood of the global minimum where the right-hand side of Conjecture \ref{conj:epsilon} cancels out have curvature several magnitudes smaller. In both cases, the curvature increases at least linearly with the over-parameterization factor $m$, which further demonstrates the benefits of over-parameterization for making the loss surface more benign.}
    \label{tab:curvs}
\end{table}

\subsection{Perturbed Gradient Descent}\label{appen: perturbed GD}
% In the previous section we have shown that although several standard properties which guarantee convergence with gradient descent (convexity, OPSC and PL condition) are not satisfied by our objective, it does satisfy another property - OPSC in most directions. In this section we would like to show that, at least in certain cases, this property is enough to ensure convergence. 

% First, we note that in \thmref{thm:almost OPSC} there is a negative $O(\sqrt{\epsilon})$ term. In the proof the sign of this term is not clearly determined, and further analysis can be done to do so, which we leave for future work. With that said, we conjecture that this term is actually non-negative, at least in a close enough neighborhood of the global minimum, we state this in the following:

% \begin{conjecture}\label{conj:epsilon}
% In the setting of \thmref{thm:almost OPSC} and under the same assumptions, we have that:
% \[
%     \frac{1}{\norm{\bw_1^n - \tilde{\bw}_1^n}^2}\cdot(\bw_1^n - \tilde{\bw}_1^n)^\top H(\bw_1^n) (\bw_1^n - \tilde{\bw}_1^n) \geq \frac{1}{4\sum_{i=1}^k\sum_{j=1}^m\norm{\bg_{i,j}}^2}\cdot\left(\norm{\bg}^2 + \left(1-\frac{2}{\pi}\right)\sum_{i=1}^k\norm{\bg_i}^2 \right)
% \]
% \end{conjecture}

% \note{We conduct thorough experiments to verify this conjecture empirically.}
To show an optimization guarantee, we use the following form of the perturbed gradient descent algorithm:

\begin{algorithm}[H]\label{alg:pgd}
% \SetAlgoLined
 Input: $\bw_1^n(0),~ \eta, \alpha > 0,~ T\in\mathbb{N}$\\
 \For{$t=1,2,\dots,T$}{
  Sample $\xi\sim \mathcal{N}\left(\pmb{0}_d,\frac{1}{\sqrt{d}}I\right)$\\
  Set $\hat{\bw}_1^n(t) := (\bw_1(t) + \alpha\xi,\dots,\bw_n(t) + \alpha\xi) $ \\
  Update $\bw_1^n(t+1) = \hat{\bw}_1^n(t) - \eta \nabla F(\hat{\bw}_1^n(t)) $
  }
  Return $\bw_1^n(T)$
 \caption{Perturbed gradient descent}
\end{algorithm}

\algref{alg:pgd} inputs an initialized weights $\bw_1^n(0)$, a learning rate $\eta$ and noise level $\alpha$. At each iteration the algorithm updates the weights w.r.t the loss function $F$ similarly to gradient descent, and adds a perturbation in a random direction with magnitude $\alpha$. Note that the same perturbation direction is the same for all the learned vectors $\bw_1,\dots,\bw_n$.

We believe it is possible to achieve better optimization guarantees if in \algref{alg:pgd}, a different noise direction would be added to each $\bw_i$. This would also require a more careful analysis, as there is a positive (small) probability that adding up all the noise vectors would produce a vector with small norm.

\subsection{Proof of \thmref{thm: OPSC in most directions optimization}}\label{appen: proof of optimization thm}

    We will first lower bound the norm of $\bg$:
    \begin{align}
    \norm{\bg}^2 &= \left\|\sum_{i=1}^n \bw_i + \alpha\xi\right\|^2 \nonumber\\
    & = \left\|\sum_{i=1}^n \bw_i \right\|^2 + \left\|\sum_{i=1}^n \alpha\xi \right\|^2 + 2\inner{\sum_{i=1}^n \alpha\xi, \sum_{i=1}^n \bw_i} \nonumber\\
    &\geq \left\|\sum_{i=1}^n \bw_i \right\|^2 + n^2\alpha^2\norm{\xi}^2 - 2n\alpha\left|\inner{\xi, \sum_{i=1}^n \bw_i}\right|\label{eq:bound on norm g}
    \end{align}
    Assume that $\left\|\sum_{i=1}^n \bw_i \right\| \geq \delta$, then we can use Cauchy-Schwartz to bound \eqref{eq:bound on norm g} by:
    \begin{align}
        \norm{\bg}^2 &\geq \left\|\sum_{i=1}^n \bw_i \right\|^2 - 2n\alpha\norm{\xi}\cdot \left\|\sum_{i=1}^n \bw_i \right\|\nonumber \\
        & \geq \left\|\sum_{i=1}^n \bw_i \right\|\cdot (\delta - 2n\alpha \norm{\xi}) \geq \delta \cdot (\delta - 2n\alpha \norm{\xi})~.\nonumber
    \end{align}
Note that $\norm{\xi}^2$ has a $\chi^2_d$ distribution. By standard concentration bound on the $\chi^2$ distribution, w.p $>1-e^{-\Omega(d)}$ we have that $\norm{\xi}\leq 1.5$, taking this event into account and substituting $\alpha$ we get that $\norm{\bg} \geq \frac{\delta^2}{4}$.

Now, assume that $\left\|\sum_{i=1}^n \bw_i \right\| < \delta$, and denote $\bu = \sum_{i=1}^n \bw_i$, we can bound \eqref{eq:bound on norm g} by:
\begin{align}
     \norm{\bg}^2 &\geq n^2\alpha^2\norm{\xi}^2 - 2n\alpha\left|\inner{\xi, \sum_{i=1}^n \bw_i}\right| \nonumber~ \\
     &\geq n^2\alpha^2\norm{\xi}^2 - 2n\alpha\norm{\bu}\left|\inner{\xi, \bar{\bu}}\right| \nonumber \\
     &\geq n^2\alpha^2\norm{\xi}^2 - 2n\alpha\delta\left|\inner{\xi, \bar{\bu}}\right| \nonumber\\
     &= \frac{\delta^2}{16}\norm{\xi}^2 - \frac{\delta^2}{2}\left|\inner{\xi, \bar{\bu}}\right|~.\nonumber
\end{align}
Since $\xi$ has a spherically symmetric distribution, independent of $\bu$, we can assume w.l.o.g that $\bar{\bu}$ is a standard unit vector, hence $\inner{\xi,\bar{\bu}}\sim \mathcal{N}\left(0,\frac{1}{\sqrt{d}}\right)$. Again, using standard concentration bounds on both the distribution of $\norm{\xi}^2$ and $\inner{\xi,\bar{\bu}}$, and applying union bound, we get that w.p $> 1-e^{-\Omega(d)}$ we have that $\norm{\xi}^2 > 0.5$ and $\inner{\xi,\bar{\bu}} \leq \frac{1}{32}$. Taking those bounds into account we get that $\norm{\bg}^2 \geq \frac{\delta^2}{32}$. In total, from both cases we get that w.p $> 1-e^{-\Omega(d)}$ we have that $\norm{\bg} \geq \frac{\delta^2}{32}$.

Applying the above bound to \eqref{eq:optimization assumption} we get that w.p $> 1-e^{-\Omega(d)}$:
\begin{align*}
    (\bw_1^n - \tilde{\bw}_1^n)^\top H(\bw_1^n) (\bw_1^n - \tilde{\bw}_1^n) &\geq \lambda \norm{\bg}^2  \geq \frac{\lambda \delta^2}{32}
\end{align*}
By the assumption that the function is twice differentiable, the above bound translates to the following bound on the gradient  w.p $> 1-e^{-\Omega(d)}$:
\begin{align}
    \inner{\nabla F(\bw_1^n), \bw_1^n-\tilde{\bw}_1^n} \geq \frac{\lambda\delta^2}{32}\cdot \norm{\bw_1^n - \tilde{\bw}_1^n}^2~.\label{eq:gradient innerprod bound}
\end{align}

Now we can bound the iterates of gradient descent, conditioning on the event of \eqref{eq:gradient innerprod bound}:
\begin{align}
    \norm{\bw_1^n(t+1) - \tilde{\bw}_1^n}^2 &= \norm{\bw_1^n(t) - \eta \nabla F(\bw_1^n(t)) + \tilde{\bw}_1^n}^2 \nonumber\\
    & = \norm{\bw_1^n(t) - \tilde{\bw}_1^n}^2 - \eta\inner{\nabla F(\bw_1^n(t)), \bw_1^n(t)-\tilde{\bw}_1^n} +  \eta^2\norm{\nabla F(\bw_1^n(t))}^2 \nonumber\\
    & \leq  \norm{\bw_1^n(t) - \tilde{\bw}_1^n}^2 - \frac{\eta\lambda\delta^2}{32}\cdot \norm{\bw_1^n(t) - \tilde{\bw}_1^n}^2 + \eta^2L^2\norm{\bw_1^n(t) - \tilde{\bw}_1^n}^2   \label{eq:lipschitz gradient}\\
    &= \norm{\bw_1^n(t) - \tilde{\bw}_1^n}^2\left(1 - \frac{\eta\lambda\delta^2}{32} + \eta^2 L^2\right) \nonumber \\
    &\leq \norm{\bw_1^n(t) - \tilde{\bw}_1^n}^2\left(1 - \frac{\eta\lambda\delta^2}{64}\right) \nonumber
\end{align}
where in \eqref{eq:lipschitz gradient} we used \eqref{eq:gradient innerprod bound}, the assumption that the gradient of $F$ is Lipschitz and that $\tilde{\bw}_1^n$ is a global minima (hence $\nabla F(\tilde{\bw}_1^n) = \pmb{0}$). By induction on $t$ we get that:
\begin{align*}
    \norm{\bw_1^n(t) - \tilde{\bw}_1^n}^2 \leq \norm{\bw_1^n(0) - \tilde{\bw}_1^n}^2\left(1 - \frac{\eta\lambda\delta^2}{64}\right)^t~.
\end{align*}
Recall that we initialized in an $\epsilon$ neighborhood of $\tilde{\bw}_1^n$ for $\epsilon < 1$, hence $\norm{\bw_1^n(0) - \tilde{\bw}_1^n}^2 < 1$. Using union bound, after $T> \frac{\log\left(\delta\right)}{\log\left(1 - \frac{\eta\lambda\delta^2}{64}\right)}$ iterations, w.p $> 1-Te^{-\Omega(d)}$ we get that $\norm{\bw_1^n(T) - \tilde{\bw}_1^n}^2 \leq \delta$.

\section{Proofs from \secref{sec:non global minima}}\label{sec:saddle_proofs}
    Before we prove \thmref{thm:norm_k}, we will first state and prove some auxiliary lemmas.
    
\begin{lemma}\label{lem:origin not min or max}
    For any $n\ge1$, the origin is neither a local minimum nor a local maximum of \eqref{eq:objective normal dist}.
\end{lemma}

\begin{proof}
    Assume $\bw_1^n=\mathbf{0}$ is the origin. Consider the point $\tilde{\bw}_1^n = (\tilde{\bw}_1,\ldots,\tilde{\bw}_n)$ where $\tilde{\bw}_1=\epsilon\bv_1$ for some real $\epsilon$, and $\tilde{\bw}_i=0$ for any $i\in\{2,\ldots,n\}$.
    Recall the closed-form of the objective in \eqref{eq:objective normal dist}, given in \citet[Section 4.1.1]{safran2017spurious} by
    \begin{equation}\label{eq:F}
        F\p{\bw_1^n} ~=~ \frac{1}{2}\sum_{i,j=1}^n f\p{\bw_i,\bw_j} - 
        \sum_{\substack{i\in[n]\\j\in[k]}} f\p{\bw_i,\bv_j} + 
        \frac{1}{2}\sum_{i,j=1}^k f\p{\bv_i,\bv_j},
    \end{equation}
    where
    \begin{equation}\label{eq:f}
	    f\p{\bw,\bv} = \frac{1}{2\pi} 
	    \norm{\bw}\norm{\bv}\p{\sin\p{\theta_{\bw,\bv}}+\p{\pi-\theta_{\bw,\bv}}
		\cos\p{\theta_{\bw,\bv}}}.
    \end{equation}
    Next, \eqref{eq:f} reveals that $f(\bw,\bv)\ge0$ for any two vectors $\bw,\bv$, and from \eqref{eq:F} we have
    \begin{align*}
        F(\bw_1^n) - F(\tilde{\bw}_1^n) &= \frac{1}{2}\sum_{i,j=1}^k f\p{\bv_i,\bv_j} - \p{\frac{1}{2}\sum_{i,j=1}^k f\p{\bv_i,\bv_j} + \frac{1}{2}f(\epsilon\bv_1,\bv_1) - \sum_{j\in[k]}f(\epsilon\bv_1,\bv_j)} \\
        &= \epsilon\sum_{j=2}^n f(\bv_1,\bv_j) + \frac{1}{2}\epsilon f(\bv_1,\bv_1)
        = \epsilon\p{\sum_{j=2}^n f(\bv_1,\bv_j) + \frac{1}{2} f(\bv_1,\bv_1)},
    \end{align*}
    and since 
    \[
        c\coloneqq\p{\sum_{j=2}^n f(\bv_1,\bv_j) + \frac{1}{2} f(\bv_1,\bv_1)} \ge \frac{1}{4}\norm{\bv_1}^2=\frac{1}{4},
    \]
    we have that $F(\bw_1^n) - F(\tilde{\bw}_1^n)=\epsilon c\to0_+$ if $\epsilon\to0_+$ and $F(\bw_1^n) - F(\tilde{\bw}_1^n)=\epsilon c\to0_-$ if $\epsilon\to0_-$, therefore we can approach $\bw_1^n$ from two different directions where in one the objective is strictly increasing and in the other it is strictly decreasing, hence $\bw_1^n$ is neither a local minimum nor a local maximum.
\end{proof}

\begin{lemma}\label{lem:no maxima}
    For any $n\ge1$, the objective in \eqref{eq:objective normal dist} has no local maxima.
\end{lemma}

\begin{proof}
    Let $\bw_1^n = (\bw_1,\ldots,\bw_n)$ and for $t\ge0$ define $\bw_1^n(t) = (t\bw_1,\ldots,t\bw_n)$. We have from Equations (\ref{eq:F}) and (\ref{eq:f}) that
    \begin{equation}\label{eq:quadratic obj}
        F(\bw_1^n(t)) = \frac{1}{2}t^2\sum_{i,j=1}^n f\p{\bw_i,\bw_j} - 
        t\sum_{\substack{i\in[n]\\j\in[k]}} f\p{\bw_i,\bv_j} + 
        \frac{1}{2}\sum_{i,j=1}^k f\p{\bv_i,\bv_j},
    \end{equation}
    hence the objective is quadratic as a function of $t$.
    
    Assuming $\bw_1^n$ is not the origin, we have from \eqref{eq:f} that 
    \[
        \sum_{i,j=1}^n f\p{\bw_i,\bw_j} > 0,
    \]
    thus from the above and \eqref{eq:quadratic obj} the objective is strongly convex in $t$, and therefore cannot attain a local maximum at $\bw_1^n=\bw_1^n(1)$. Otherwise, if $\bw_1^n$ is the origin, then from \lemref{lem:origin not min or max} it is not a local maximum.
\end{proof}

\begin{lemma}\label{lem:over_hess}
    Suppose $n\ge1$, $\bw_1^n = (\bw_1,\ldots,\bw_n)$ is a differentiable local minimum of the objective in \eqref{eq:objective normal dist}. Then for all $\alpha\in(0,1)$ and any $i\in[n]$, the point $\bw_1^{n}(\alpha,i)$ is a critical point of $F$, and the Hessian of $F$ at $\bw_1^{n}(\alpha,i)$ is given in terms of the blocks of $H(\bw_1^n)$ by
    \end{lemma}
    \[
        \p{\begin{matrix}
            H_{1,1} & \cdots & H_{1,i-1} & H_{1,i} & H_{1,i} & H_{1,i+1} & \cdots & H_{1,n} \\
            \vdots & \ddots & \vdots & \vdots & \vdots & \vdots & \ddots & \vdots \\
            H_{i-1,1} & \cdots & H_{i-1,i-1} & H_{i-1,i} & H_{i-1,i} & H_{i-1,i+1} & \cdots & H_{i-1,n} \\
            H_{i,1} & \cdots & H_{i,i-1} & \frac{1}{2}I + \frac{1}{\alpha}H_{i,i}^{\prime} & \frac{1}{2}I & H_{i,i+1} & \cdots & H_{i,n} \\
            H_{i,1} & \cdots & H_{i,i-1} & \frac{1}{2}I & \frac{1}{2}I + \frac{1}{1-\alpha}H_{i,i}^{\prime} & H_{i,i+1} & \cdots & H_{i,n} \\
            H_{i+1,1} & \cdots & H_{i+1,i-1} & H_{i+1,i} & H_{i+1,i} & H_{i+1,i+1} & \cdots & H_{i+1,n} \\
            \vdots & \ddots & \vdots & \vdots & \vdots & \vdots & \ddots & \vdots \\
             H_{n,1} & \cdots & H_{n,i-1} & H_{n,i} & H_{n,i} & H_{n,i+1} & \cdots & H_{n,n}
        \end{matrix}}.
    \]

    \begin{proof}
        From the gradient of the objective in \eqref{eq:G} and \eqref{eq:g} and since $\bw_1^n$ is a local minimum, we have for all $i\in[n]$ that
        \[
            \frac{1}{2}\bw_i+\sum_{j\neq i} g(\bw_i,\bw_j) -\sum_{j=1}^kg (\bw_i,\bv_j) = \mathbf{0}.
        \]
        We begin with asserting that for all $\alpha\in(0,1)$, $\bw_1^n(\alpha,i)$ is a critical point of $F$. First, from \citet[Lemma 3.2]{brutzkus2017globally}, $F$ is differentiable for all $\alpha\in(0,1)$, therefore the gradient is well-defined. For $m\in[n+1]\setminus\{i,i+1\}$ we have
        \begin{align*}
            \frac{\partial}{\partial\bw_m}F(\bw_1^n(\alpha,i)) &= \frac{1}{2}\bw_m + \sum_{j\in [n]\setminus \{i,m\}} g(\bw_m,\bw_j) + g(\bw_m,\alpha\bw_i) + g(\bw_m,(1-\alpha)\bw_{i}) - \sum_{j=1}^kg (\bw_m,\bv_j) \\
            &= \frac{1}{2}\bw_m + \sum_{j\in [n]\setminus \{i,m\}} g(\bw_m,\bw_j) + \alpha g(\bw_m,\bw_i) + (1-\alpha)g(\bw_m,\bw_{i}) - \sum_{j=1}^kg (\bw_m,\bv_j) \\
            &= \frac{1}{2}\bw_m + \sum_{j\in [n]\setminus \{i,m\}} g(\bw_m,\bw_j) + g(\bw_m,\bw_i) - \sum_{j=1}^kg (\bw_m,\bv_j) \\
            &= \frac{1}{2}\bw_m + \sum_{j\in [n]\setminus \{m\}} g(\bw_m,\bw_j) - \sum_{j=1}^kg (\bw_m,\bv_j) = \mathbf{0}.
        \end{align*}
        For $i$ we have
        \begin{align*}
            \frac{\partial}{\partial\bw_i}F(\bw_1^n(\alpha,i)) &= \frac{1}{2}\alpha\bw_i + \sum_{j\in[n]\setminus\{i\}}g(\alpha\bw_i,\bw_j) + g(\alpha\bw_i,(1-\alpha)\bw_i) - \sum_{j=1}^kg (\bw_i,\bv_j) \\
            &= \frac{1}{2}\alpha\bw_i + \sum_{j\in[n]\setminus\{i\}}g(\bw_i,\bw_j) + (1-\alpha)g(\bw_i,\bw_i) - \sum_{j=1}^kg (\bw_i,\bv_j) \\
            &= \frac{1}{2}\alpha\bw_i + \sum_{j\in[n]\setminus\{i\}}g(\bw_i,\bw_j) + \frac{1}{2}(1-\alpha)\bw_i - \sum_{j=1}^kg (\bw_i,\bv_j) \\
            &= \frac{1}{2}\bw_i + \sum_{j\in[n]\setminus\{i\}}g(\bw_i,\bw_j) - \sum_{j=1}^kg (\bw_i,\bv_j) = \mathbf{0},
        \end{align*}
        where likewise a similar computation for $\bw_{i+1}$ shows that $\frac{\partial}{\partial\bw_{i+1}}F(\bw_1^n(\alpha,i)) = \mathbf{0}$.
        Turning to the Hessian, we have from \lemref{lem:hessian is differentiable} that $F$ at $\bw_1^n(\alpha,i)$ is twice differentiable. We then have by \thmref{thm:hessian of objective} that all off-diagonal blocks other than $H_{i,i+1}$, and $H_{i+1,i}$ remain the same as in $H(\bw_1^n)$, since $h_2(\bu_1,\bu_2)$ isn't affected by linearly rescaling $\bu_1,\bu_2$. For the remaining two off-diagonal blocks we have from \lemref{lem:h2 parallel vectors} that each is $\frac12I$, and lastly we compute the diagonal blocks, starting with the $m$-th block $H_{m,m}$ where $m\in[n+1]\setminus\{i,i+1\}$. We have
        \begin{align*}
            H_{m,m}(\bw_1^n(\alpha,i)) &= \frac{1}{2}I + \sum_{j\in [n]\setminus \{i,m\}} h(\bw_m,\bw_j) + h(\bw_m,\alpha\bw_i) + h(\bw_m,(1-\alpha)\bw_{i}) - \sum_{j=1}^k h(\bw_m,\bv_j) \\
            &= \frac{1}{2}I + \sum_{j\in [n]\setminus \{i,m\}} h(\bw_m,\bw_j) + \alpha h(\bw_m,\bw_i) + (1-\alpha)h(\bw_m,\bw_{i}) - \sum_{j=1}^k h(\bw_m,\bv_j) \\
            &= \frac{1}{2}I + \sum_{j\in [n]\setminus \{i,m\}} h(\bw_m,\bw_j) + h(\bw_m,\bw_i) - \sum_{j=1}^k h(\bw_m,\bv_j) \\
            &= \frac{1}{2}I + \sum_{j\in [n]\setminus \{m\}} h(\bw_m,\bw_j) - \sum_{j=1}^k h(\bw_m,\bv_j).
        \end{align*}
        That is, $H_{m,m}(\bw_1^n(\alpha,i))$ equals $H_{m,m}(\bw_1^n)$ for $m\in[i-1]$ and equals $H_{m-1,m-1}(\bw_1^n)$ for $m\in\{i+2,\ldots,n+1\}$.
        For the $i$-th block $H_{i,i}$ we have
        \begin{align*}
            H_{i,i}(\bw_1^n(\alpha,i)) &= \frac{1}{2}I + \sum_{j\in[n]\setminus\{i\}}h(\alpha\bw_i,\bw_j) + h(\alpha\bw_i,(1-\alpha)\bw_i) - \sum_{j=1}^k h(\alpha\bw_i,\bv_j) \\
            &= \frac{1}{2}I + \frac{1}{\alpha}\sum_{j\in[n]\setminus\{i\}}h(\bw_i,\bw_j) - \frac{1}{\alpha}\sum_{j=1}^k h(\bw_i,\bv_j) \\
            &= \frac{1}{2}I + \frac{1}{\alpha}H_{i,i}'(\bw_1^n),
        \end{align*}
        where the second equality is due to \lemref{lem:hessian is differentiable}, and likewise, a similar computation reveals that
        \[
            H_{i+1,i+1}(\bw_1^n(\alpha,1)) = \frac{1}{1-\alpha}H_{i,i}'(\bw_1^n),
        \]
        concluding the proof of the lemma.
\end{proof}

\begin{lemma}\label{lem:negative block}
     Suppose $n\ge1$, $\bw_1^n = (\bw_1,\ldots,\bw_n)$ is a differentiable local minimum of the objective in \eqref{eq:objective normal dist} such that there exists $i\in[n]$ with component $H_{i,i}'$ satisfying $\bu^{\top}H_{i,i}'\bu=\lambda$ for some unit vector $\bu\in\reals^d$ and scalar $\lambda<0$. Then for any $\alpha\in(0,1)$, $\bw_1^n(\alpha,i)$ is a saddle point. Moreover, for $\alpha\in\{0,1\}$, $\bw_1^n(\alpha,i)$ is not a local minimum of $F$.
\end{lemma}

\begin{proof}
    The key in proving the lemma is that the $i$-th diagonal block of the Hessian at $\bw_1^n$ (having dimensions $d\times d$ and given by $0.5I+H_{i,i}'$) is turned into a $2d\times2d$ block of the following form:
    \begin{equation}\label{eq:hess block}
        \p{\begin{matrix}
            \frac{1}{2}I + \frac{1}{\alpha}H_{i,i}^{\prime} & \frac{1}{2}I \\
            \frac{1}{2}I & \frac{1}{2}I + \frac{1}{1-\alpha}H_{i,i}^{\prime}
        \end{matrix}}.
    \end{equation}
    Next, we show that $H_{i,i}'$ is not PSD, hence the block matrix above is not PSD, and consequentially the Hessian at $\bw_1^n(\alpha,i)$ is not PSD, implying the lemma.
    For the case where $\alpha\in(0,1)$, by using \lemref{lem:over_hess}, we have that this is a critical point, and from \lemref{lem:no maxima} it is not a local maximum. We multiply the $2d\times2d$ block of the Hessian at $\bw_1^n(\alpha,i)$ given in \eqref{eq:hess block} by $(\bu, -\bu)\in\reals^{2d}$ from both sides and obtain
    \[
        (\bu, -\bu)\p{\begin{matrix}
            \frac{1}{2}I + \frac{1}{\alpha}H_{i,i}^{\prime} & \frac{1}{2}I \\
            \frac{1}{2}I & \frac{1}{2}I + \frac{1}{1-\alpha}H_{i,i}^{\prime}
        \end{matrix}}
        \p{\begin{matrix}
            \bu \\
            -\bu
        \end{matrix}}
        = \frac{1}{\alpha}\bu^{\top}H_{i,i}'\bu + \frac{1}{1-\alpha}\bu^{\top}H_{i,i}'\bu
        = \lambda\p{\frac{1}{\alpha} + \frac{1}{1-\alpha}} < 0.
    \]
    Next, letting $\tilde{\bu}\in\reals^{(n+1)d}$ be the all-zero vector, except for entries $(i-1)d+1$ to $(i+1)d$ which equal $(\bu, -\bu)$, then 
    \[
        \tilde{\bu}^{\top}H(\bw_1^n(\alpha,i))\tilde{\bu} = \lambda\p{\frac{1}{\alpha} + \frac{1}{1-\alpha}} < 0,
    \]    
    thus $H(\bw_1^n(\alpha,i))$ is not a PSD matrix.
    
    For the case where $\alpha\in\{0,1\}$, since the objective is not differentiable in this case, we will show that the point cannot be a local minimum by showing that in any neighborhood containing it also contains a point with a strictly smaller objective.
    
    Assuming $\alpha=0$, we have from the above derivation that there exists $i\in[n]$ such that for all $\alpha'\in(0,1)$, $\bw_1^n(\alpha',i)$ is not a local minimum. In particular, given some $\delta>0$, choose $\alpha'>0$ small enough so that $\norm{\bw_1^{n}(\alpha',i)-\bw_1^{n}(0,i)}\le\delta/2$. Since $\bw_1^n(\alpha',i)$ is not a local minimum, there exists $\tilde{\bw}_1^{n+1}$ such that $\norm{\bw_1^n(\alpha',i) - \tilde{\bw}_1^{n+1}} \le \delta/2$ and $F(\tilde{\bw}_1^{n+1}) < F(\bw_1^n(\alpha',i))$. Since
    \[
        \relu{\inner{\alpha'\bw,\bx}} + \relu{\inner{(1-\alpha')\bw,\bx}} = \relu{\inner{\bw,\bx}}
    \]
    for all $\bw,\bx\in\reals^d$ and $\alpha'\in[0,1]$, this entails 
    \[
        F(\tilde{\bw}_1^{n+1}) < F(\bw_1^n(\alpha',i)) = F(\bw_1^n(0,i))
    \] 
    and $\norm{\tilde{\bw}_1^{n+1} - \bw_1^n(0,i)} \le \delta$, hence $\bw_1^n(0,i)$ is not a local minimum. 
    
    Finally, the case for $\alpha=1$ follows from the $\alpha=0$ case by permuting the neurons.
\end{proof}

We are now ready to prove \thmref{thm:norm_k}.
\begin{proof}[Proof of \thmref{thm:norm_k}]
    Throughout the proof we will assume that $F$ is differentiable at $\bw_1^n$, which by \lemref{lem:hessian is differentiable} also implies that it is twice continuously differentiable there, and as would be evident later in the proof we will see that this is necessarily the case.
    
    We will show that there exist $i\in[n]$ and some unit vector $\bu$ such that
    \begin{equation}\label{eq:neg lambda}
        \lambda\coloneqq \bu^{\top}H_{i,i}^{\prime}\bu = \sum_{j\in[n]\setminus\{i\}} \frac{\sin(\theta_{\bw_i,\bw_j})\norm{\bw_j}}{2\pi\norm{\bw_i}}-\sum_{j=1}^k\frac{\sin(\theta_{\bw_i,\bv_j})}{2\pi\norm{\bw_i}} < 0,
    \end{equation}
    hence $H_{i,i}^{\prime}$ is not a PSD matrix. We begin with letting $O^{\top}DO$ be the eigendecomposition of the symmetric matrix $\bar{\bw}\bar{\bw}^{\top} - \bar{\bn}_{\bv,\bw}\bar{\bn}_{\bv,\bw}^{\top}$, where $O$ is an orthonormal matrix and $D$ is diagonal. We have that $\text{diag}(D)=(1,-1,0,\ldots,0)$, as readily seen by taking the orthogonal eigenvectors $\bar{\bw},\bar{\bn}_{\bv,\bw}$ which correspond to the eigenvalues $1,-1$ respectively, where the remaining $d-2$ vectors orthogonal to $\bar{\bw},\bar{\bn}_{\bv,\bw}$ comprise the rest of the spectrum with all zero correponding eigenvalues. Taking expectation over a random vector $\hat{\bu}=(u_1,\ldots,u_d)$ uniformly on the unit hypersphere we have
    \begin{align}
        \E_{\hat{\bu}}\left[\hat{\bu}^{\top} h_1(\bw,\bv)\hat{\bu} \right] &= \E_{\hat{\bu}}\left[\hat{\bu}^{\top} \frac{\sin(\theta_{\bw,\bv})\norm{\bv}}{2\pi\norm{\bw}}\p{I-\bar{\bw}\bar{\bw}^{\top} + \bar{\bn}_{\bv,\bw}\bar{\bn}_{\bv,\bw}^{\top}} \hat{\bu} \right] \nonumber\\
        &= \frac{\sin(\theta_{\bw,\bv})\norm{\bv}}{2\pi\norm{\bw}}-\E_{\hat{\bu}}\left[\hat{\bu}^{\top}\p{\bar{\bw}\bar{\bw}^{\top} - \bar{\bn}_{\bv,\bw}\bar{\bn}_{\bv,\bw}^{\top}} \hat{\bu} \right] \nonumber\\
        &= \frac{\sin(\theta_{\bw,\bv})\norm{\bv}}{2\pi\norm{\bw}}-\E_{\hat{\bu}}\left[\hat{\bu}^{\top}O^{\top}DO \hat{\bu} \right] \nonumber\\
        &= \frac{\sin(\theta_{\bw,\bv})\norm{\bv}}{2\pi\norm{\bw}}-\E_{\hat{\bu}}\left[\hat{\bu}^{\top}D \hat{\bu} \right] \label{eq:orthogonal}\\
        &= \frac{\sin(\theta_{\bw,\bv})\norm{\bv}}{2\pi\norm{\bw}}-\E_{\hat{\bu}}\left[u_1^2-u_2^2 \right] = \nonumber\\ 
        &= \frac{\sin(\theta_{\bw,\bv})\norm{\bv}}{2\pi\norm{\bw}}, \label{eq:iid}
    \end{align}
    where equality (\ref{eq:orthogonal}) is due to a uniform distribution on the unit hypersphere being invariant to orthonormal transformations, and equality (\ref{eq:iid}) is due to all coordinates of $\hat{\bu}$ being i.i.d. From \eqref{eq:iid}, the definition of $H_{i,i}'$, the linearity of expectation and the fact that $\norm{\bv_i}=1$ for all $i\in[k]$, we have 
    \begin{equation}\label{eq:expectation_sphere}
        \E_{\hat{\bu}}\left[ \hat{\bu}^{\top}H_{i,i}^{\prime}\hat{\bu} \right] = \sum_{j\in[n]\setminus\{i\}} \E_{\hat{\bu}}\left[\hat{\bu}^{\top}h(\bw_i,\bw_j)\hat{\bu}\right] - \sum_{j=1}^k\E_{\hat{\bu}}\left[\hat{\bu}^{\top}h(\bw_i,\bv_j)\hat{\bu}\right] = \lambda
    \end{equation} 
    for all $i\in[k]$.
    We will show this implies the existence of a particular vector $\bu$ satisfying the above equality. Choose an arbitrary unit vector $\bu'$. If $\bu'$ satisfies the above equality we are done. Otherwise, assume w.l.o.g.\ that
    $\bu'^{\top}H_{i,i}^{\prime}\bu' < \lambda$, in which case there must exist another unit vector $\bu''$ such that $\bu''^{\top}H_{i,i}^{\prime}\bu'' > \lambda$ (since otherwise $\E_{\hat{\bu}}\left[ \hat{\bu}^{\top}H_{i,i}^{\prime}\hat{\bu} \right] < \lambda$, contradicting \eqref{eq:expectation_sphere}). Let $\gamma(t)= \frac{t\bu' + (1-t)\bu''}{\norm{t\bu' + (1-t)\bu''}}$, then by the continuity of $\gamma(t)^{\top}H_{i,i}^{\prime}\gamma(t)$ in $t$ and the intermediate value theorem we have some $t'$ satisfying $\gamma(t')^{\top}H_{i,i}^{\prime}\gamma(t') =\lambda$, and by taking $\bu=\gamma(t')$ we have $\bu^{\top}H_{i,i}^{\prime}\bu =\lambda$.
    
    Next, we show that $\lambda<0$ under the assumptions in the theorem statement. To this end, it suffices to show that for some $i\in[n]$
    \begin{equation}\label{eq:negative_eigenvalue}
        \sum_{j\neq i}^n\sin(\theta_{\bw_i,\bw_j})\norm{\bw_j} - \sum_{j=1}^k\sin(\theta_{\bw_i,\bv_j})<0.
    \end{equation}
    Beginning with the positive term, we have
    \[
        \sum_{j\neq i}^n\sin(\theta_{\bw_i,\bw_j})\norm{\bw_j} \le \sum_{j\neq i}^n\sin(\theta_{\bw_i,\bw_j})\frac{k\norm{\bw_j}}{\sum_{m=1}^n\norm{\bw_m}} \le \sum_{j\neq i}^n\frac{k\norm{\bw_j}}{\sum_{m=1}^n\norm{\bw_m}}.
    \]
    Since $\sum_{j=1}^n\frac{k\norm{\bw_j}}{\sum_{m=1}^n\norm{\bw_m}}=k$, there exists some $i\in[n]$ such that 
    \begin{equation*}%\label{}
        \frac{k\norm{\bw_i}}{\sum_{m=1}^n\norm{\bw_m}}\ge \frac kn \ge 1,
    \end{equation*}
    thus the above equals
    \begin{equation}\label{eq:positive_term}
        \sum_{j=1}^n\frac{k\norm{\bw_j}}{\sum_{m=1}^n\norm{\bw_m}} - \frac{k\norm{\bw_i}}{\sum_{m=1}^n\norm{\bw_m}} \le k-1.
    \end{equation}
    Turning to the negative term in \eqref{eq:negative_eigenvalue}, recall that $\bw_i=(w_{i,1},\ldots,w_{i,d})$ and $d\ge k$. Assume w.l.o.g.\ that $\bv_j$ is a standard unit vector for all $j\in[k]$ (otherwise apply a change of basis under which the following argument is invariant), and compute
    \[
        \sum_{j=1}^k\sin(\theta_{\bw_i,\bv_j}) = \sum_{j=1}^k\sqrt{1-\frac{\inner{\bw_i,\bv_j}^2}{\norm{\bw_i}^2}} = \sum_{j=1}^k\sqrt{1-\frac{w_{i,j}^2}{\norm{\bw_i}^2}} \ge \sum_{j=1}^k\p{1-\frac{w_{i,j}^2}{\norm{\bw_i}^2}} \ge k-\sum_{j=1}^d\frac{w_{i,j}^2}{\norm{\bw_i}^2} = k-1.
    \]
    Observe that if \eqref{eq:positive_term} is not a strict inequality, then it must hold that $n=k$ and $\norm{\bw_j}=1$ for all $j\in[n]$. In such case, since $\bw_1^n$ is not global, we have from \thmref{thm:global_minima_form} that it is not a permutation of the standard basis, therefore there must exist $\bw_i$ of unit norm which is non-zero in at least two coordinates. For this $\bw_i$, the above must be a strict inequality since $\sqrt{x}>x$ for any $x\in(0,1)$, which guarantees a strict inequality for at least two summands. Now, combining the above with \eqref{eq:positive_term} and plugging in \eqref{eq:negative_eigenvalue} establishes \eqref{eq:neg lambda}. Next, we invoke \lemref{lem:negative block} with what was shown in \eqref{eq:neg lambda}.
    
    To conclude the proof of the theorem, it only remains to show that $\bw_1^n$ cannot be non-differentiable (which also implies that $\bw_1^n(\alpha,i)$ is not a local minimum for $\alpha\in\{0,1\}$). Assume $\bw_1^n$ is non-differentiable, then by \lemref{lem:hessian is differentiable}, there exists some $i\in[n]$ such that $\bw_i=\mathbf{0}$.
    
    First assume that $\bw_1^n$ is not the origin. If we remove all zero vector neurons to obtain a differentiable point $\bw_1^{n'}\in\reals^{n'd}$ for $n'<n$, then we reduce to the previous case and there exists $j$ such that $\bw_1^{n'}(0,j)$ is a saddle point, and clearly adding more zero vector neurons (and permuting the neurons accordingly) till we recover $\bw_1^n$, we have that it cannot be a local minimum, contradicting the theorem assumption.
    
    Finally, if $\bw_1^n$ is the origin then from \lemref{lem:origin not min or max}, $\bw_1^n$ is not a local minimum.
    
\end{proof}

\begin{proof}[Proof of \propref{prop:norm sum nk}]
    Given a point $\bw_1^n$, let $\bw_1^n(t)=(t\bw_1,\ldots,t\bw_n)$. From \eqref{eq:quadratic obj}, we have that the objective is quadratic in $t$ and therefore in particular, for a point to be minimal in our original $nd$-dimensional space, it must be minimal over $t$. Optimizing over $t$ we have that the optimum $t^*$ is given by
    \[
        t^* = \frac{\sum_{i\in[n],j\in[k]} f\p{\bw_i,\bv_j}}{\sum_{i,j=1}^n f\p{\bw_i,\bw_j}},
    \]
    therefore any local minimum must be of the form $\bw_1^n(t^*)$, in which case its sum of Euclidean norms is given using \eqref{eq:f} by
    \[
        \sum_{i=1}^n\norm{t^*\bw_i} = t^*\sum_{i=1}^n\norm{\bw_i}
        = \frac{\sum_{i\in[n],j\in[k]} 
	    \norm{\bw_i}\norm{\bv_j}\p{\sin\p{\theta_{\bw_i,\bv_j}}+\p{\pi-\theta_{\bw_i,\bv_j}}
		\cos\p{\theta_{\bw_i,\bv_j}}}}
		{ \sum_{i,j=1}^n \norm{\bw_i}\norm{\bw_j}\p{\sin\p{\theta_{\bw_i,\bw_j}}+\p{\pi-\theta_{\bw_i,\bw_j}}
		\cos\p{\theta_{\bw_i,\bw_j}}}}\sum_{i=1}^n\norm{\bw_i}.
    \]
    Elementary calculus reveals that in $[0,\pi]$, the function $x\mapsto \sin(x) + (\pi-x)\cos(x)$ is monotonically decreasing, thus its image is bounded in the same interval, and the above displayed equation is upper bounded by
    \[
        \frac{\pi\sum_{i\in[n],j\in[k]} 
	    \norm{\bw_i}\norm{\bv_j}}
		{\sum_{i=1}^n\pi\norm{\bw_i}^2 + \sum_{i\neq j}^n \norm{\bw_i}\norm{\bw_j}\p{\sin\p{\theta_{\bw_i,\bw_j}}+\p{\pi-\theta_{\bw_i,\bw_j}}
		\cos\p{\theta_{\bw_i,\bw_j}}}}\sum_{i=1}^n\norm{\bw_i},
    \]
    which in turn is at most
    \begin{equation}\label{eq:CS}
		\frac{\pi\sum_{i\in[n],j\in[k]} 
	    \norm{\bw_i}\norm{\bv_j}}
		{\pi\sum_{i=1}^n\norm{\bw_i}^2} \sum_{i=1}^n\norm{\bw_i} = \frac{k\sum_{i\in[n]} 
	    \norm{\bw_i}}
		{\sum_{i=1}^n\norm{\bw_i}^2} \sum_{i=1}^n\norm{\bw_i} = k\frac{\p{\sum_{i=1}^n\norm{\bw_i}}^2}{\sum_{i=1}^n\norm{\bw_i}^2}.
    \end{equation}
    Letting $\bu = (\norm{\bw_1},\ldots,\norm{\bw_n})\in\reals^n$ and $\bu'=(1,\ldots,1)\in\reals^n$, we have from CS
    \[
        \p{\sum_{i=1}^n\norm{\bw_i}}^2 = \inner{\bu,\bu'}^2 \le \norm{\bu}^2\norm{\bu'}^2 = n\sum_{i=1}^n\norm{\bw_i}^2,
    \]
    thus by plugging the above we have that \eqref{eq:CS} is upper bounded by $kn$.
\end{proof}

\begin{proof}[Proof of \thmref{thm:all_blocks_neg}]
    Simply invoke \lemref{lem:negative block} with each $i\in[n]$ satisfying the assumption in the theorem statement to obtain the result.
\end{proof}

    %5   & 2   & $8.443288545751672e-11$  & $2.272899510113521e-14$               \\ \hline
    %5   & 5   & $2.4386576622092723e-10$ & $7.5408045450743e-13$                 \\ \hline
    %5   & 10  & $5.431512719356243e-10$  & $6.9316920257790635e-12$              \\ \hline
    %10  & 2   & $1.2745045549912855e-09$ & $3.7099274902368857e-13$              \\ \hline
    %10  & 5   & $4.319014139254053e-09$  & $8.806565589271407e-12$               \\ \hline
    %10  & 10  & $6.2332110034429355e-09$ & $7.33486345749715e-11$                \\ \hline
    %20  & 2   & $8.866242183239748e-09$  & $3.1855776645185e-12$                 \\ \hline
    %20  & 5   & $2.3120884672163665e-08$ & $6.701848202314294e-11$               \\ \hline
    %20  & 10  & $4.126226922979256e-08$  & $5.156018618107414e-10$               \\ \hline    

\end{document}